\documentclass[]{customh}

\OneAndAHalfSpacedXI 

\usepackage{natbib,mathabx}
 \def\bibfont{\small}%
\def\eventz{\mathcal{Z}}
 \usepackage[]{algorithm, algpseudocode}
\usepackage{bm}
\usepackage{booktabs} 
\usepackage{caption,xcolor}
\usepackage{comment,enumitem}
\usepackage{graphicx}
\usepackage{subcaption}

\usepackage{dsfont}

\usepackage{endnotes}
\let\footnote=\endnote

 \bibpunct[, ]{(}{)}{,}{a}{}{,}%
 \def\bibfont{\small}%
\def\etab{\boldsymbol{\eta}}
\def \reg{{\textsc{Reg}}}
\def \E {\mathbb{E}}
\def \P {\mathbb{P}}
\def \eps {\varepsilon}
\def\S{S}
\def\eventforc{\mathcal{G}}

\def\ind{\mathbb{I}}
\def\event{\mathcal A}
 \def\tjmax{\widebar{\mathcal{T}_j}}
  \def\tjmaxminus{\widebar{\mathcal{T}}_{j-1}}
  \def\tjmaxzero{\widebar{\mathcal{T}_0}}
 \def\tj{\mathcal{T}_j}
 \def\tdelta{\small{\mathcal U}_{j}}

\usepackage{calligra}
\usepackage{mathtools}
\DeclarePairedDelimiter{\ceil}{\lceil}{\rceil}

\newcommand{\p}[1]{#1}
\newcommand{\ki}{k}
\renewcommand{\vec}[1]{\bm{#1}}
\def \ZE {Y}
\newcommand{\G}[1]{\text{Gap}_i(#1)}
\def\level{L}

\TheoremsNumberedThrough     
\ECRepeatTheorems

\EquationsNumberedThrough    

\begin{document}

\RUNTITLE{Learning Product Rankings Robust to Fake Users }

\TITLE{Learning Product Rankings Robust to Fake Users }

\ARTICLEAUTHORS{
\AUTHOR{Negin Golrezaei}
\AFF{MIT Sloan School of Management, Massachusetts Institute of Technology, \texttt{golrezae@mit.edu}}
\AUTHOR{Vahideh Manshadi}
\AFF{Yale School of Management, Yale University,  \texttt{vahideh.manshadi@yale.edu}}
\AUTHOR{Jon Schneider}
\AFF{Google Research, \texttt{jschnei@google.com}}
\AUTHOR{Shreyas Sekar}
\AFF{Department of Management, University of Toronto Scarborough and Rotman School of Management, \texttt{shreyas.sekar@rotman.utoronto.ca}}
}
\RUNAUTHOR{Golrezaei, Manshadi, Schneider, and Sekar}

\ABSTRACT{%
In many online platforms, customers' decisions are substantially influenced by product rankings as most customers only examine a few top-ranked products. Concurrently, such platforms also use the same data corresponding to customers' actions to learn how these products must be ranked or ordered. These interactions in the underlying learning process, however, may incentivize sellers to artificially inflate their position by employing fake users, as exemplified by the emergence of click farms. 
Motivated by such fraudulent behavior, we study the ranking problem of a platform that faces a mixture of real and fake users who are indistinguishable from one another. 
We first show that existing learning algorithms---that are optimal in the absence of fake users---may converge to highly sub-optimal rankings under manipulation by fake users. To overcome this deficiency, we develop efficient learning algorithms under two informational environments: in the first setting, the platform is aware of the number of fake users, and in the second setting, it is agnostic to the number of fake users. For both these environments, we prove that our algorithms converge to the optimal ranking, while being robust to the aforementioned fraudulent behavior; we also present worst-case performance guarantees for our methods, and show that they significantly outperform existing algorithms. At a high level, our work employs several novel approaches to guarantee robustness such as: $(i)$ constructing product-ordering graphs that encode the pairwise relationships between products inferred from the customers' actions; and $(ii)$ implementing multiple levels of learning with a judicious amount of bi-directional cross-learning between levels.   Overall, our results indicate that online platforms can effectively combat fraudulent users without incurring large costs by designing new learning algorithms that guarantee efficient convergence even when the platform is completely oblivious to the number and identity of the fake users. }
\KEYWORDS{product ranking, sequential search, robust learning, fake users, online platforms} 

\maketitle

\section{Introduction}

The abundance of substitutable products on online shopping platforms combined with  consumers' limited attention has resulted in a new form of competition among products:  the race for visibility. 
For example, an Amazon user is typically presented with 
a ranking of thousands of search results---displayed in a sequence of web-pages each containing a few dozen products---even though she is unlikely to go beyond the first page \citep{Milward}. 
Consequently, the success of a product crucially depends on its position in the ranking.
Cognizant of such position effects, online platforms tend to rank more popular products higher (i.e., make them more visible). However, because the popularity of products is a priori unknown, the platform seeks to learn 
them through the same process, i.e., by presenting a ranked assortment of products to users and getting feedback from them. 
Such an online, real-time learning process opens the possibility of manipulations in the race for visibility: ``click farms'' have emerged in which firms employ {\em fake users} who would click on designated products in the hope of boosting their popularity and thus misleading the platform to rank them in top positions \citep{stevens2018sellers}.
It has been reported that some Amazon sellers pay \$10,000 a month to ``black hat'' companies  in order to  be ranked in top positions \citep{Fraud}.
The emergence and prevalence of such fraudulent behavior raises the following key question: {\em can an online platform efficiently learn the optimal product ranking in the presence of fake users?}

We pursue this question in the context of an online platform that presents each arriving customer with a fixed set of products, displayed in a particular order\footnote{This could be in response to a specific search query or tied to a product category.} (a ranking). Customers examine the products sequentially until they identify and click on the desired product, exhibiting position bias as they are more likely to only view products in top ranks. The platform then seeks to learn product preferences from click feedback in order to refine its ranking for future customers. However, it faces the threat of manipulation from fake users. In particular, $F$ out of the $T$ customers who visit the platform constitute `fake users'; such users may strategically click on certain products in order to boost their position or withhold clicks to achieve the opposite effect. Crucially, the platform is not aware of the identity of these fake users and cannot simply ignore their feedback. Therefore, their actions can distort the platform's perception of product popularity, and lead to downstream consequences for real customers who may see undesirable products at top positions. In the face of these challenges, developing learning algorithms that are robust to fake users is clearly a priority. Yet, despite the growing body of work on online learning and product ranking (see the related work section), we lack a comprehensive understanding of how to develop learning algorithms that are resilient to fake users and whether existing algorithms satisfy this criterion.

\smallskip
{\bf Summary of Contributions.} In this work, we follow a regret analysis framework and 
assess the performance of learning algorithms by proving worst-case guarantees
parameterized by the number of fake users $F$, which we refer to as the \textit{fakeness budget}. Given the above model, we show the following results.
\begin{enumerate}[leftmargin=*]
    \item We show that commonly used learning algorithms for product ranking are vulnerable to fake users in that their regret can be $\Omega(T)$, {even when the number of fake users is small.}
    
    \item For the setting where  the \emph{fakeness budget} $F$ is known to the platform, we design a deterministic  online algorithm called {\em Fake-Aware Ranking (FAR)} whose worst-case regret is  $O(\log(T) + F)$. 
    
    \item For a more challenging setting   where the fakeness budget is unknown to the platform, we design a randomized online algorithm called {\em Fake-Oblivious Ranking with
Cross-Learning (FORC)} whose worst-case regret is $O(F \log(T))$. 

\item Finally, we carry out a numerical study using synthetic data that illustrates the superior performance of FORC even though the algorithm is unaware of the fakeness budget.
\end{enumerate}
\medskip

All together,  our results show that {\em an online platform can effectively combat fake users without incurring too much cost by  employing learning algorithms that are robust to such  fraudulent behavior}. In the rest of this introduction,  we provide {a more formal} overview of our setting and the high-level ideas of our algorithms.

We study the ranking problem faced by an online platform by {adapting the well-studied \emph{stochastic cascading bandits} model~\citep{kveton2015cascading,lattimore2018toprank} to a setting  with position effects and fake users.} In particular, the platform displays $n$ products to sequentially arriving customers.  
Each product has a click probability which is a priori unknown to the platform. To capture the behavior of real customers when faced with a ranking of products, we follow the cascade model~\citep{craswell2008experimental} under which a customer
sequentially examines products from the top position, in increasing order of rank. After examining each product, the customer clicks on it with the product's click probability and {conditional on clicking}, she stops. Customers who do not click on the product at a certain position either exit the platform (with a position-specific exit probability) or proceed to examine the product placed in the next position. This process ends when the customer stops, exits, or reaches the end of the ranking. 
Each time a ({real}) customer clicks on a product, the platform earns a fixed reward, which we normalize to one. As for fake users, they may arrive at any point during the time horizon, and strategically click on specific products (or withhold clicks) to fulfill some unknown objective; we make no assumption on their behavior.  Finally, we reiterate that the platform cannot distinguish between real customers and fake users. 

Faced with such a customer population and reward structure, the platform's aim is to learn the ordering of the products that corresponds to customers' preferences, namely, which product has the largest click probability, second largest, and so on. For real customers, such an ordering is the optimal ranking,  as it maximizes the number of customers who click on products, which coincides with the 
platform's reward. Given this objective, the platform measures the 
performance of an online learning algorithm by its  expected cumulative regret, which is the gap between the reward of the optimal ranking and that of the algorithm (see Equation \eqref{eqn_regret_F}  and its related discussion in Section \ref{sec:model}). In the presence of fake users, however, it is not hard to see that the regret of any learning algorithm would depend on the number of fake users, and the exact nature of the fraudulent behavior adopted by such users. Given that fake users' strategies may be arbitrarily sophisticated and hard to quantify, we follow an \emph{adversarial model} and pursue the goal of developing online algorithms with small worst-case regret that depends only on the fakeness budget $F$, i.e., the guarantees hold regardless of the strategy adopted by the fake users.

\smallskip
{\bf Failure of Traditional Learning-To-Rank Algorithms.} 
We remark that even without fake users, the learning problem that we study is challenging as it deviates from a standard multi-armed bandit setting because of its combinatorial nature: each ranking can be viewed as one arm implying that the number of arms would be exponential in $n$. While prior work (e.g., \cite{kveton2015cascading,lattimore2018toprank}) overcomes this challenge by generalizing the ideas in Upper Confidence Bound (UCB) algorithms, we show that in the  presence of fake users, such UCB-type algorithms could lead to poor performance. Specifically, in Theorem~\ref{thm_ucb_bad}, we prove that generalizations of UCB for the ranking problem have worst-case regret that degrades linearly with the length of the time horizon $T$ (i.e. $\Omega(T)$) even when the number of fake users is only $O(\log^2(T))$.

Theorem~\ref{thm_ucb_bad} also provides insights on 
why traditional algorithms are vulnerable to manipulation even for simple instances with just two products. For the sake of illustration, consider two products (one and two) with product one having a higher click probability, and suppose that real customers only examine the top ranked product before exiting. The optimal ranking for this instance clearly involves placing product one at the top position. Consider fake users who arrive in the early rounds with the intention of promoting product two over product one (e.g., such users could be hired by the sellers of product two). We show that it is possible for these fake users to adopt actions that mislead the learning algorithm to overestimate (underestimate) the reward of the inferior (superior) product.  Consequently, the algorithm would converge to a sub-optimal ranking that places  product two at the top rank; the algorithm is unable to correct its estimates by incorporating more feedback as real customers ignore the lower ranked product. Upon convergence, the same ranking is repeated for future customers, leading to a regret of  $\Omega(T)$. Intuitively, our analysis shows that the lack of robustness of such algorithms stems from two key factors: (1) the reliance of the algorithm on the estimates formed in early rounds, and (2) the sequential nature of customers' search behavior which makes receiving feedback on lower ranked products more difficult.

In light of the aforementioned result, we develop novel online algorithms that learn the optimal ranking despite manipulation by fake users. Based on the platform's knowledge of the fakeness budget, we design two different algorithms: (i) Fake-Aware Ranking (FAR) algorithm for settings where the platform can assess the fakeness budget $F$ (for example, based on customer-level historical data) and (ii) Fake-Oblivious Ranking with Cross-Learning (FORC) algorithm for  settings where the platform is unable to estimate $F$. 

\smallskip
{\bf Fake-Aware Ranking (FAR) Algorithm.} Recall that the platform's learning goal is to determine the optimal ordering of products based on their click probabilities, in the presence of fake users. In order to achieve this, the ranking algorithm that we design (FAR), tracks the pairwise relationships between products using a directed graph that we call the {\em product-ordering  graph}. Formally, the nodes of this graph correspond to the $n$ products, and a directed edge from product $j$ to $i$ indicates that with high probability (based on  customer actions), product $i$ has a larger click probability than product $j$. The key idea behind this method involves leveraging the pairwise product comparisons, and knowledge of the fakeness budget $F$ by enlarging the confidence intervals associated with each product to add edges in a conservative fashion. In particular, the extra width of the confidence intervals is proportional to the fakeness budget $F$, and is carefully chosen to compensate for  any overestimation (or underestimation) caused by the actions of the fake users. 

Moreover, the product-ordering graph plays a crucial role in constructing partially optimal rankings based on incomplete feedback at any given point in time. That is, once we determine that product $i$ has a higher click probability compared to product $j$, product $i$  is ranked ahead of product $j$ in all future rounds irrespective of its exact position. Incorporating partial feedback of this form into the final ranking is essential to guarantee low regret. We analyze the worst-case regret\footnote{We emphasize that our regret bounds are gap-dependent and we assume that the gap between the click probabilities of any two products is bounded.} of the FAR algorithm in Section \ref{sec:FAR} and show it is on the order of $O(\log(T)+F)$; see Theorem \ref{thm_far}.

\smallskip 
{\bf Fake-Oblivious Ranking with Cross-Learning (FORC) Algorithm.} Our central contribution in this work is a novel algorithm for learning product rankings even when the platform is unaware of the fakeness budget $F$, and thus, cannot simply widen the confidence intervals proportionally. Instead FORC builds on the ideas in \cite{lykouris2018stochastic} as well as the FAR algorithm, and uses a multi-level randomized scheme in order to distribute the damage caused by fake users across $\level \triangleq \log_2(T)$ learning levels\footnote{We use $\log_2(\cdot)$ to denote logarithm to the base two and $\log(\cdot)$ for the natural logarithm.} running in parallel. Specifically, each level contains its own product-ordering graph similar to its counterpart in FAR, and the probability of sampling a level follows a geometric distribution, i.e., level $1 \leq \ell \leq \level$ is chosen with probability proportional to $2^{-\ell}$ in each round. Therefore, higher levels are exposed to fewer fake users and accurately infer pairwise product relationships but also learn conservatively due to a lower sampling frequency whereas the opposite is true for lower levels, which incur larger regret. In light of this trade-off, the crucial ingredient that binds the algorithm together and controls regret is the notion of \emph{bi-directional cross-learning} between levels, which we employ as follows:
\begin{enumerate}
    \item Any pairwise product relationship that is inferred at (say) level $\ell$ is immediately transferred to all lower levels ($\ell' < \ell$). Intuitively, since higher levels are exposed to fewer fake users, this downward cross-learning allows us to effectively utilize the accurate relationships inferred at these levels.
    
    \item However, downward cross-learning alone is not sufficient to limit the regret incurred in lower levels because  the accurate edges can be added too late due to the low frequency of sampling higher levels. We therefore complement this via a novel upward cross-learning mechanism, wherein customer feedback collected at level $\ell$ is partially used for product comparisons at levels $\ell' > \ell$---this enables the higher levels to rapidly infer accurate relationships. \end{enumerate}

 In summary, bi-directional  cross-learning between layers allow them to coordinate effectively, leading to   
a worst-case regret of $O(F \log(T))$; see Theorem \ref{thm_forc} in Section \ref{sec:FORC}. Our analysis builds on the ideas used in FAR to ensure that all product relationships inferred at levels $\ell \geq \log_2(F)$ are correct, and thus the regret of those levels can be bounded similarly. For levels lower than $\log_2(F)$, we exploit cross-learning to bound the regret incurred due to the actions of fake users.

\smallskip 
{\bf Numerical Studies.} Finally, in Section \ref{sec:numerics}, we complement our theoretical work with numerical simulations using synthetic data. Our numerical results show the non-robustness of UCB-like algorithms under manipulation, even in real-world inspired settings that are much more general than the simple example in Theorem~\ref{thm_ucb_bad}. Our results further show that FORC outperforms FAR despite its informational disadvantage (with regard to the fakeness budget $F$), highlighting the power of  randomization and multi-level learning to combat fake users (see Figure~\ref{fig:regret_main} and its related discussion).

The rest of the paper is organized as follows. In Section \ref{sec:related_work}, we review the related literature. In Section \ref{sec:model}, we introduce our model and the platform's online ranking problem. Then, we formalize the fake users' strategy space, as well as the measure of regret. Next, we illustrate the fragility  of UCB in Section \ref{sec:nonRob} by presenting a lower bound on its regret. 
Sections \ref{sec:FAR} and \ref{sec:FORC} are devoted to describing and analyzing our two new {algorithms}, respectively, FAR and FORC. 
Section \ref{sec:numerics} presents our complementary numerical studies, and
Section \ref{sec:conclude} concludes the paper. For the sake of brevity, we only include proof ideas 
in the main text. The detailed proofs of all statements are provided in appendices.

\section{Related Work}\label{sec:related_work}

Our work contributes to several streams of research within the online decision-making  literature, which we compare and contrast below.

\textit{Learning with Corrupted Data.}
The problem of designing learning algorithms that are robust to corruption has received significant interest in the last few years \citep{lykouris2018stochastic,jun2018adversarial, gupta2019better, chen2019robust, lykouris2019corruption,lykouris2020bandits}. This line of work was initiated by \cite{lykouris2018stochastic}, who studied a  multi-armed bandit problem with the input sequence of samples being a mixture of stochastic and adversarial (i.e., corrupted) components. {Our treatment of fake users in this paper mirrors the notion of corrupted samples in the works mentioned above. 

Broadly speaking, our model generalizes much of this literature (with the notable exception of~\cite{chen2019robust}) due to the combinatorial nature of the product ranking problem. Although our algorithms build on some of the ideas in \cite{lykouris2018stochastic}, including enlarging the confidence intervals and multi-level learning, the subtle differences between the two models necessitate a fundamentally different approach. First, directly applying the algorithms from~\cite{lykouris2018stochastic} in infeasible in our setting as the exponential number of arms, i.e., possible rankings, would lead to a protracted learning phase. Furthermore, one cannot employ the \emph{Active Arm-Elimination} method in~\cite{lykouris2018stochastic} as the sequential nature of the consumer search model precludes eliminating products. In fact, to obtain good guarantees for this problem, it is important to dynamically maintain a relative ordering of products, which we do using  product-ordering graphs.

Perhaps the most important difference between the two settings stems from the nature of the feedback that the learning algorithm receives, particularly: (a) a fake click on a product at (say) position $j$ also influences the empirical reward on the products in the first $j-1$ positions, and (b) unlike a typical bandit problem, the algorithm cannot control which products it receives feedback on (beyond the first position) since customers' exit position is random. Due to the latter limitation, any learning algorithm for this setting could take an inordinate amount of time to achieve a course correction after manipulation by fake users. Moreover, this limitation can lead to a lack of coordination between multiple learning levels, which we overcome by having strong cross-learning. 
Finally, we remark that the differences outlined above are also applicable for some of the other works in this literature including~\cite{jun2018adversarial,gupta2019better,lykouris2020bandits}}.

Another related paper in this literature is the work of \cite{chen2019robust}, who study the problem of designing learning algorithms for assortment planning that are robust to corruption. 
Similar to \cite{lykouris2018stochastic}, the authors use the
Active Arm-Elimination technique to eliminate products that are not in the optimal assortment with high probability. As stated earlier, in our setting, we cannot use this technique. Furthermore, unlike our setting that deals with position bias and random feedback counts, learning algorithms for assortment planning  obtain feedback for every product offered in the assortment, which aids their design and analysis.

\textit{Learning under Non-Stationary Environments.} Another line of research that is related to our work pertains to multi-armed bandits under non-stationary environments \citep{besbes2014stochastic, besbes2015non,karnin2016multi,keskin2017chasing, luo2017efficient, cheung2019hedging,licascadingn2019}. In this line of work, pioneered  by \cite{besbes2014stochastic, besbes2015non}, the reward functions  evolve over time, but the total change in the reward function across the time horizon is bounded. Therefore, while the above papers focus on designing learning policies that track a ``moving target", our work and more generally, the literature on learning with corruption deals with a stationary target that can be abruptly but temporarily displaced by fake users.

Within this literature on learning under non-stationary environments, the work that is positioned closest to ours is that of~\cite{keskin2017chasing}, who study a dynamic pricing problem where customer demand evolves in one of two ways: (a) gradual drifts and (b) bursty and big changes. For these two settings, \cite{keskin2017chasing} obtain reget bounds in the order of  $O(T^{2/3}B^{1/3})$ and $O(\sqrt{T}\log(T))$, respectively,  where $B$ is the total variation budget. Arguably, non-stationary environments with bursty, adversarial changes are somewhat analogous to our setting since we can model the actions taken by fake users as changes in the underlying environment. However, one cannot simply adapt the results in that work to design learning algorithms for our ranking problem owing to some key differences, namely: (a) unlike the setting in \cite{keskin2017chasing}, we have the additional challenge of dealing with a combinatorial environment; (b) the results in that paper for bursty changes only hold when the changes are large enough, allowing the algorithm to detect them. Such an assumption does not necessarily hold in our setting with fake users; (c) even if we ignore the requirement of having big changes,  \cite{keskin2017chasing} 
present an algorithm with a $O(\sqrt{T}\log(T))$ regret guarantee. Yet, their work does not yield insights on whether learning algorithms can yield $O(\log(T))$ gap-dependent bounds. We achieve these much-improved guarantees for our setting by leveraging the structural properties of how fake users alter the underlying rewards.

 \textit{Incentive-aware Learning.} Our work is also related to the literature on  incentive-aware learning; see, for example, \cite{amin2013learning,amin2014repeated, kanoria2017dynamic, epasto2018incentive, golrezaei2019incentive, golrezaei2019dynamic}. In this literature, it is assumed that the data (i.e., the samples) are generated by strategic agents and hence, prone to manipulation, i.e., differ from the underlying ground truth. The goal here is to design learning algorithms that incentivize the strategic agents to provide truthful feedback---i.e., to not generate corrupted data. 
More specifically, many of the papers in this literature consider the problem of learning how to set reserve prices in repeated auctions. In this scenario, the data corresponds to bids submitted by strategic bidders, and the auctioneer seeks to incentivize these bidders to submit uncorrupted (truthful) bids in order to learn the optimal reserve prices.
 We note that our work deviates from this line of research as we do  not aim to incentivize the fake users  to generate truthful data in the form of clicks. Instead, our objective is to learn the optimal ranking despite the presence of fake users.  

\textit{Robust Online Decision-making.}
Beyond learning, the problem of designing robust algorithms has been studied in the online 
decision making literatue, particularly in the case of resource allocation problems \citep{mahdian2007allocating, golrezaei2014real,esfandiari2015online, hwang2018online, bradac2019robust}. Similar to our work, these papers study settings where the arrival sequence deviates from a stochastic process. 
They highlight the vulnerability of online algorithms designed for stochastic arrival  and  
develop robust algorithms that effectively take into account the presence of an adversarial or a corrupted component; however, these works do not involve any learning.

 \textit{Product Ranking.} 
 Recently, many papers have designed algorithms for product ranking (or display) that account for the impact of position bias on customer choice  \citep{davis2013assortment,abeliuk2015benefits, aouad2015display, abeliuk2016assortment, gallego2016approximation, lei2018randomized, derakhshan2018product, asadpour2020ranking}. In contrast to our setting, the aforementioned works focus on the offline version of the product ranking problem, where the platform is aware of all the parameters that make up the customers' choice model (e.g., click probabilities). Our work is more closely aligned to the handful of papers in this domain that study the ranking problem in an online learning setting, i.e., the platform's goal is to learn  the parameters of the customers' choice model and optimize its ranking decisions at the same time. In particular, while \cite{kveton2015cascading,lattimore2018toprank,ferreira2019learning} study this problem in purely stochastic settings without fake users, \cite{BW} develop policies for adversarial settings, using Blackwell Approachability \citep{blackwell1956analog},

Closest to our work in terms of the techniques used is \cite{lattimore2018toprank}, which presents a learning algorithm called {\em TopRank} for the product ranking problem. This algorithm
constructs a directed acyclic graph analogous to our product-ordering graph to encode the pairwise
relationships between products and make ranking decisions. However, there is a key difference between our ranking algorithms and TopRank: given equivalent products in the product-ordering graph, our methods prioritize those which have recieved the smallest amount of feedback from customers so far, whereas TopRank places them in a uniformly random order. This subtle change is crucial in bounding the regret of our algorithms as it enables us to rapidly learn about under-sampled products. Moreover, our FORC algorithm also relies on multi-level learning, which is not a feature of TopRank. Finally, while \cite{lattimore2018toprank} consider the product ranking challenge in a purely stochastic setting, our work is the first  to study this problem in a setting with a mixture of stochastic and  adversarial components via the introduction of fake users. As stated earlier, this setting is inspired by the visibility race on online platforms that can motivate sellers to trick (e.g., via fake clicks) the platforms' ranking algorithms to secure better positions in the search results.

 \textit{Sequential Search Models.}
 When it comes to modeling
 customer choice in the presence of position bias, sequential search models are prominently employed. In such models, pioneered by \cite{weitzman1979optimal}, 
  products are examined one by one starting with the first rank.
  One of the most widely used sequential search models---which we also adopt in the current work---is the cascade model, first introduced by \cite{craswell2008experimental}. We note that other works such as \cite{kveton2015cascading,lattimore2018toprank,cao2019sequential,wang1901making} also consider a similar cascade model, albeit without fake users. Furthermore, this model has also been  used  for studying position auctions (e.g.,  \cite{varian2007position, kempe2008cascade, athey2011position, chu2017position}), and dynamic pricing (e.g., \cite{gao2018multi, najafi2019multi}).

\section{Model}
\label{sec:model}

Consider an online platform which displays $n$ products with labels in $[n] = \{1,2,\ldots, n\}$. 
Each product $i\in [n]$  has a {click probability}  of $\mu_i$, measuring its relevance or quality. {Click probabilities are initially unknown to the platform.}
 Without loss of generality, we assume that products are {indexed such that}\footnote{Note that we assume no two products have the same click probability. This assumption ensures the uniqueness of the optimal ranking, which we define later. A similar assumption is common in the multi-armed bandit literature; see, for example, \cite{lai1985asymptotically}.} 
\begin{align}
\label{eq:indexing}
 \mu_1 > \mu_2 > \ldots > \mu_n.   
\end{align}

For each arriving customer, the platform displays these $n$ products in the form of a ranking
 $\pi$ over $n$ positions.  Here, $\pi({\p{j}})=i$ implies that product $i$ is placed in position $\p{j}$, where positions with smaller indices have more visibility.  Similarly, $\pi^{-1}(i)$ denotes the position of product $i$ under ranking $\pi$. We informally refer to product $j$ (respectively position $\p{j}$) as being \emph{better than} product $i$ (respectively position $\p{i}$) when $\mu_j > \mu_i$ (respectively $\p{j} < \p{i}$).

\textbf{Customers' Search Model.} We divide the customers into two categories: \emph{real}  and \emph{fake}. 
{First, we describe the search behavior of a real customer.}
{Facing a ranking $\pi$, we assume that the customer sequentially examines products starting from the product in the top position, going downward. At any stage, if she finds an acceptable product, she stops and clicks on it. Otherwise, she either leaves the platform or proceeds to examine the product in the next position. } 
{Our modeling framework falls into the category of \emph{cascade models}.} 
{Such models} have been extensively studied in the context of online platforms  in a variety of applications {such as position auctions in sponsored search~(e.g., \cite{aggarwal2008sponsored, craswell2008experimental}, and \cite{kempe2008cascade}}), online retail~(e.g., \cite{cao2019sequential} and \cite{ najafi2019multi}), and web search~(e.g., \cite{craswell2008experimental} and \cite{kveton2015cascading}).
{Cascade models provide tractable frameworks to capture the impact of position on customer choice and, particularly, the impact of the externality that products in higher (better) positions impose on those in lower (worse) positions.}

Formally, under our model, a real customer {facing ranking $\pi$} begins by examining the product in the first position, i.e., $\pi({1})$. 
{She finds product $\pi({1})$ acceptable independently with probability $\mu_{\pi({1})}$. In that case, she clicks on it, stops her search, and leaves the platform.} 
{On the other hand, with probability $1 - \mu_{\pi({1})}$ she finds this product unacceptable. In that case, she either stops and leaves the platform (independently with position-dependent exit probability $q_{1}$), or she proceeds to examine the product in the second position repeating the same process.}
Our model extends the original formulation of cascading behavior proposed by~\cite{craswell2008experimental} by adding exit probabilities.
{These position-dependent exit probabilities, i.e., {$\{q_{\p{j}}, \p{j} \in [n-1]\}$}, capture the behavior that}
 customers may exit the platform if they view too many irrelevant products due to {limited} attention spans or fatigue (e.g., see~\cite{cao2019sequential,wang1901making}). 

A fake user, however, does not follow the aforementioned search pattern and we make no particular assumptions on the actions pursued by such a user in any given round. For example, in the case of click farms, fake users are hired to repeatedly click on a specific product in the hope that the platform (oblivious to their existence) would boost the position of that product. More generally, a fake user may click on any of the displayed products---regardless of the ranking and click probabilities---creating a fake click, or she may strategically not click on any product and even exit at an arbitrary position. 

Formally, the above actions can be modeled by means of a framework where all of the fake users are generated by a single entity. The entity is assumed to follow an adaptive policy $P_t$ that maps $\mathcal{H}_{t-1}$---the history of both the platform and customers' actions up to round $(t-1)$, formally defined later--- to the fake user's actions in round $t$. The fake user's action may be deterministic or randomized and includes (a probability distribution over) three components:
\begin{enumerate}
    \item whether or not the user in round $t$ is fake,
    \item if so, the identity of at most one product the fake user would click on,
    \item in the absence of a click, the position at which the fake user exits (if the fake user clicks on a product, they exit at the corresponding position to mimic real customers).
\end{enumerate}

We assume that the fake user can influence the outcome of at most $F$ rounds, which we denote as the \emph{fakeness budget}. We use $\mathcal{P}$ to denote the family of feasible policies that respect the fake entity's budget, which includes randomized policies. The generality of this framework enables us to encompass different types of fake users including those who employ sophisticated strategies; see Theorem~\ref{thm_ucb_bad} for a specific example. Our overall goal is to develop learning algorithms that are robust to any arbitrary, and unknown policy adopted by such an entity. In the rest of this work, we abuse terminology and use the term fake user to refer to both the individual customers in specific rounds as well as the overall entity that controls the adaptive policy.

\textbf{Platform's Information and Objective.} We assume that {a priori} the platform is not aware of the click probabilities {$\{\mu_i, i \in [n]\}$} and {exit} probabilities  {$\{q_{\p{j}}, \p{j} \in [n-1]\}$}. In each round, the platform only observes {(a) which (if any) product the customer clicks on, and (b) where she exits the platform in case she does not click on any product. Note that under our model, if the customer clicks on the product in position $\p{j}$, then she exits at the same position.}
 The latter is a mild assumption as many  platforms display products on devices with a small screen. On each page of such a device, only a few products, if not one, are displayed. {Thus, a customer needs to take some action observable to the platform (for example, swiping) to browse more products.}

The platform's objective is to find a ranking $\pi^{\star}$ that maximizes \emph{customer engagement}, which is the click probability of real customers.
Under the described customer search model, the optimal ranking (for real customers) is simply ordering the products in decreasing order of their {click probabilities} {$\{\mu_i, i \in [n]\}$}. That is, the product with the highest {click probability}  should be placed in position one, and the product with the second-highest  {click probability}  should be placed in position two, and so on. Since {we indexed} the products  in decreasing order of their click probabilities, the optimal ranking {is} characterized by {$\pi^{\star}(\p{i}) = i$} for all $i\in [n]$. Thus, the platform's goal is to learn this optimal ranking by observing customers' {clicks and exit positions} {without knowing whether the customer is real or fake}.

\textbf{Online Ranking Problem.} We study the platform's ranking problem in an online setting with $T$ rounds, where in each round $t\in [T]$,  the platform displays the products to an arriving customer according to ranking $\pi_t$. 
The main challenge here stems from the platform's lack of awareness regarding whether the {customer at round $t$ is real or fake.}

Formally, we use $\mathcal{C}_{r,t}(\pi) \in \{0,1\}$ to denote the {click} action of a real customer who arrives in round $t$  when presented with a ranking $\pi$. 
More specifically, $\mathcal{C}_{r,t}(\pi) = 1$ if the customer in round $t$ is real and clicks on  a product under ranking $\pi$; otherwise $\mathcal{C}_{r,t}(\pi)=0$. Analogously, we  define  $ \mathcal{C}_{f,t}(\pi)$ to indicate a click from  a fake customer. Observe that while $\mathcal{C}_{r,t}(\pi)$ is a random variable drawn from {the} distribution specified by the real customers' search model {(which is initially unknown to the platform)}, its value also depends on the action $P_t$ adopted by the fake user at this round, e.g., whether or not the user is fake. For convenience, we use $c_{t} \in (\{\emptyset\} \cup [n]) \times [n]$ to represent the (real or fake) user's actions in round $t$, comprising of the product clicked on and the exit position, where $c_t=(\emptyset,j)$ implies that no product was clicked on and the customer exits after position $j$.

The platform earns a unit of reward in round $t$ if $\mathcal{C}_{r,t}(\pi_t) =1$---i.e., \emph{only when a real customer engages with the platform.} %

The performance of any algorithm is then measured by the expected cumulative regret (or more precisely pseudo-regret, e.g., see~\cite{bubeck2015convex}), which is the gap between the reward obtained by selecting the (unknown) optimal ranking during all rounds and that of the given ranking algorithm. Let $\mathcal{H}_t = \{(\pi_1, f_1, c_{1}), (\pi_2, f_2, c_2), \ldots, (\pi_t, f_t, c_t)\}$ denote the entire history up to round $t$, where $f_{t'}=1$ implies the presence of a fake user in round $t'\le t$.  Define $\mathcal{H}^o_t = \{(\pi_1, c_{1}), (\pi_2, c_2), \ldots, (\pi_t, c_t)\}$ as the sub-history observable by the platform. The regret of an algorithm $(\pi_t: \mathcal{H}^o_{t-1} \rightarrow \Pi)_{t=1}^T$ is then defined as:\footnote{When it is clear from the context, we abuse notation and write $\pi_t$ to denote both the ranking in round $t$ and the overall ranking algorithm. Further $\Pi$ represents the set of all possible rankings.}

\begin{equation}
\reg_{T} = \sup_{\mathbf{P}\in\mathcal{P}}\left\{\E_{\mathcal{H}_T(\pi^{\star})}\Big[\sum_{t=1}^T \mathcal{C}_{r,t}(\pi^\star)\Big]- \E_{\mathcal{H}_T(\pi_t)}\Big[\sum_{t=1}^T \mathcal{C}_{r,t}(\pi_t) \Big] \right\},
\label{eqn_regret_F}
\end{equation}
where the expectations are taken over $\mathcal{H}_T(\pi^{\star})$ and $\mathcal{H}_T(\pi_t)$---the random histories of the algorithm and customer actions when the underlying algorithms select rankings $\pi^{\star}$ and $(\pi_t)_{t=1}^T$,  respectively. Note that the realization of the history  also depends on any randomness in the policies adopted by both the platform and the fake user, as well as the randomness stemming from the real customers' clicks in any given round.  Further, since the adaptive policy employed by the fake users is unknown to the platform, we seek to achieve minimal regret over all possible $\bm{P} = (P_t)_{t=1}^T$ belonging to a family $\mathcal{P}$ of feasible policies as discussed earlier.

\begin{remark}[Discussion on our definition of regret.] We highlight that  our regret notion as well as  the platform objective  only takes into account the click actions from real customers. This  is motivated by the fact that a platform does not derive any tangible benefit from the clicks generated by fake users. At the same time, although the actions of fake users do not directly alter the regret in the same round, they may significantly hurt the utility derived by a platform in future rounds by causing the platform to incorrectly estimate the click probabilities (rewards) of various products. One could alternatively consider another notion of regret under which the fake clicks are also counted. Such a notion of regret is considered in some previous works, e.g., ~\cite{lykouris2018stochastic}. Although including fake clicks does not make sense for the application that we are interested in, these two notions are actually very close to each other in that their difference cannot exceed the number of fake users, denoted by $F$. See Appendix~\ref{app:regret} for  the relationship between our definition of regret and one where fake clicks are included under any fixed policy $\bm{P}$.
 \end{remark}

\textbf{Information Settings.} We design learning algorithms under two informational environments, based on whether or not the platform can estimate the fakeness budget $F$. 
In the first setting, the platform is aware of the fakeness budget $F$, whereas
in the second setting, the  platform does not have this knowledge. The former scenario is motivated by 
 the fact that online platforms may be able to estimate the aggregate number or fraction of fake users from historical data, even if individual users cannot be verified. For this setting, {in Section~\ref{sec:FAR},} we design a learning algorithm, which we refer to as the \emph{Fake-Aware Ranking} (FAR) Algorithm. Naturally, its regret depends on the fakeness budget\footnote{In the extreme case where $F=T$, one would expect the regret to be equal to $\Omega(T)$ since no learning is possible.} $F$.
 
 We derive gap-dependent bounds for the regret of FAR in terms of the gaps between the products' click probabilities defined below:
\begin{equation}
    \Delta_{j,i} = \mu_j - \mu_i \quad \forall i,j \in [n].
\end{equation} 
 
 In the case where all gaps are bounded below by a constant, we show that the expected regret of FAR is $O(F + \log(T))$ (Theorem \ref{thm:far}).
 For the more challenging setting where we do not know $F$, we present a learning algorithm in Section~\ref{sec:FORC}, which we term \emph{Fake-Oblivious Ranking with Cross-Learning} (FORC).  We show that the expected regret of FORC is given by $O\left(F\log(T)\right)$; see Theorem \ref{thm_forc} for the gap-dependent regret bound of FORC.

Before presenting our algorithms, in the next section, we show that ignoring the existence of fake users and simply running well-established algorithms such as the Upper Confidence Bound (UCB) Algorithm  can lead to linear regret even when the fakeness budget $F$ is  small.

\section{Non-Robustness of UCB to Fake Users}
\label{sec:nonRob}

In this section, we argue that existing stochastic bandit algorithms for product ranking---such as \texttt{CascadeUCB}~\citep{kveton2015cascading} and its variants---are not naturally robust to fake users. In particular, we will show that there are instances of our problem with $F = O(\log^2 (T))$ fake users, for which these algorithms never converge to the optimal ranking, and incur regret that is \textit{linear} in $T$ (whereas the regret bounds for FAR and FORC are sublinear and given by Theorems~\ref{thm_far} and~\ref{thm_forc} respectively).

 To show this result, we construct a simple two product instance of our ranking problem which ignores its combinatorial aspect so that computing a ranking becomes equivalent to selecting a single product for the top position. The problem of learning product rankings then reduces to a simple multi-armed bandit problem, and therefore, any combinatorial generalization of UCB such as \texttt{CascadeUCB}~\citep{kveton2015cascading}  would also reduce to the standard UCB algorithm (e.g., see~\cite{bubeck2015convex}). Leveraging this equivalence, the following theorem  demonstrates that the UCB algorithm is not robust to fake clicks. 
\begin{theorem}[Non-Robustness of UCB to Fake Users]
There exists an instance of the product ranking problem with two products and  sequence of $F = O(\log^2(T))$ fake users which causes the UCB algorithm to incur $\Omega(T)$ regret. 
\label{thm_ucb_bad}
\end{theorem}

\textbf{(Proof Sketch)}. In order to show Theorem \ref{thm_ucb_bad}, we consider an instance with two products, where $\mu_1=1$ and $\mu_2=\frac{1}{2}$. In this instance,  the exit probability is given by $q_1=1$; that is, real customers stop their search after examining the product in the first position, and never examine the product in the second position. Since the number of products is two, there are only two possible rankings, namely $(1,2)$ and $(2,1)$, where in ranking $(i,j)$  product $i$ is placed in the first position and product $j$ is placed in the second position. Therefore, selecting a ranking $(i,j)$ is equivalent to picking a product $i$ for the top position. Given this reasoning, it is not hard to see that generalizations of UCB for the ranking problem such as \texttt{CascadeUCB}~\citep{kveton2015cascading} or \texttt{PBM-UCB}~\citep{lagreeVC16} would simply reduce to the UCB algorithm for a standard multi-armed bandit problem. Recall that this algorithm maintains an estimate of the reward (click probability) for each arm (product) and selects the arm in each round with the highest upper confidence bound on its empirical reward. 

Following this simplification, we now construct a strategy for the fake users under which the UCB algorithm does not learn the optimal ranking $(1,2)$, i.e., it does not select product one for the top position. To comprehend the motivation behind the this strategy, it is important to understand that UCB is significantly more robust to \textit{overestimations} of rewards than \textit{underestimations}. For example, while it might be tempting to implement a strategy which simply uses the fake users to click on the sub-optimal product (raising UCB's estimate of its mean), this alone does not suffice.  Once the fake clicks subside, UCB will continue to collect samples from this product by placing it in the top position, and quickly readjust its estimate of this product's mean. On the other hand, causing UCB to underestimate a product's mean reward can be devastating: once the upper confidence bound of a product falls below the other product's mean, with high probability, UCB will never select this product again! For this reason, instead of using the fake users to exclusively boost product two's estimated reward, we also require them to \textit{worsen} product one's empirical reward to such an extent that UCB will never select it for the remainder of the algorithm.

We now describe the fake user's strategy, which is depicted in Figure \ref{fig:UCB}. The strategy has three phases. In the first phase, which has a duration of $2\log^2(T)$ rounds, the fake user does not click on any product regardless of what ranking is presented. In the second phase, which again has a duration of $2\log^2(T)$ rounds, the fake user
clicks on product two when it is present in the top position but never clicks on product one. In the remaining rounds, which constitute the third phase, there are no fake users.

\begin{figure}[hbt!]
 \centering
 \includegraphics[width=0.9\textwidth]{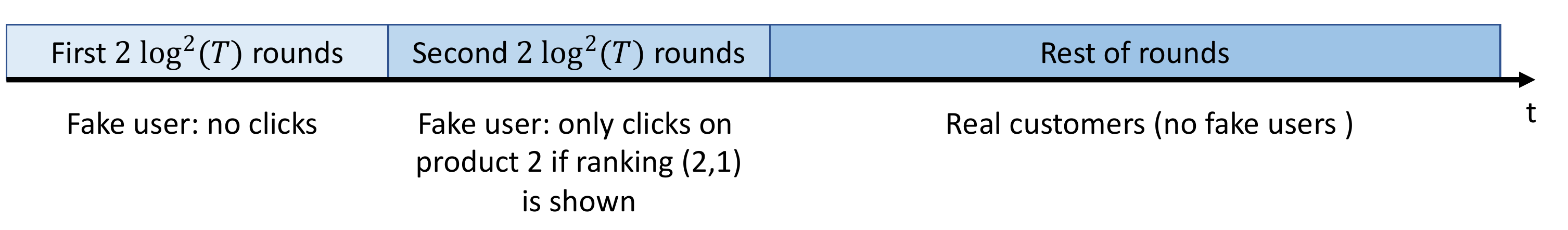}
\caption{The fake user's strategy in Theorem \ref{thm_ucb_bad}. \label{fig:UCB}}
\end{figure}

At the conclusion of the second phase, the upper confidence bound of product two is almost certainly significantly larger than that of product one; so the UCB algorithm will select ranking $(2, 1)$ over $(1, 2)$ for the remaining customers. Further, due to the absence of clicks in the first stage, the empirical mean of the click probabilities of both products at the beginning of the third stage will be smaller than their true rewards. Therefore, as the UCB algorithm selects ranking $(2,1)$ for the real customers, the estimated reward for product two cannot decrease. However, given that ranking $(2,1)$ is sub-optimal, this leads to linear regret in the remaining rounds.

Finally, we conclude by noting that although we only characterized the underperformance of UCB for a simple, two-product instance, the same behavior and poor regret are also applicable for much more general instances, which we highlight in Section~\ref{sec:numerics}. The fact that one of the most well-studied algorithms in the online learning literature can be tricked into learning a sub-optimal ranking motivates us to design new methods that are more robust to fake users. 

\section{Known Fakeness Budget: Fake-Aware Ranking (FAR) Algorithm}
\label{sec:FAR}
In this section, we present our Fake-Aware Ranking (FAR) Algorithm {for the setting where the platform knows the fakeness budget $F$}.  The design of the algorithm, which is presented in Algorithm \ref{cascading_2}, is based on the crucial observation that  {the platform} does not necessarily need to estimate the click probabilities {$\{\mu_i, i \in [n]\}$} {(nor the exit probabilities)} to identify the optimal ranking $\pi^{\star}$; instead, it suffices to correctly determine for every pair of products $(i, j)$, whether $\mu_{i}$  is greater than $\mu_{j}$. In light of this observation, our algorithm keeps track of a product-ordering (directed) graph $G$, where every node of this graph represents a product and {a} directed edge from node $i$ to node $j$ (i.e., edge $(i, j)$) implies that product $j$ dominates product $i$ in the sense that {$\mu_j > \mu_i$}. We now describe the various features of FAR, and in the process, provide an informal description of Algorithm~\ref{cascading_2}.

\textbf{Product-Ordering Graph.} Initially, the product-ordering graph $G$ does not have any edges. Gradually, as our algorithm collects more feedback on customers' preferences for various products, we can infer pairwise relationships with more certainty. Eventually, the algorithm adds an edge $(i, j)$ to graph $G$ in round $t$ when the condition in~\eqref{eq:add_edge} holds, indicating that with high probability, the true click probability of product $j$ $(\mu_j)$ is larger than that of product $i$ $(\mu_i)$.
\begin{align} \label{eq:add_edge} r_{i} + \sqrt{\frac{\log(\frac{2nT}{\delta})}{\eta_{i}}} + \frac{F}{\eta_{i}} \leq  r_{j} - \sqrt{\frac{\log(\frac{2nT}{\delta})}{\eta_{j}}} - \frac{F}{\eta_{j}}\,.\end{align}

    Here, $r_i$ and $r_j$ are the algorithm's empirical estimates of the click probabilities of products $i$ and $j$ respectively, and $\eta_i, \eta_j$ denote the number of times (so far) the algorithm has received feedback on these products, respectively. Finally, $\delta$ is a parameter, which we set in Step~\ref{step:initialization_FAR} of Algorithm~\ref{cascading_2}. For any given product $i \in [n]$, we say that the algorithm receives feedback on this product in round $t$, if the customer stops her search in position $ j$ and $\pi_t^{-1}(i)\le  j$; that is, product $i$ is placed in a position at least as visible as $ j$. Of course, such feedback is only credible if the  customer is real. However, the algorithm cannot distinguish between real and fake users, and as a result, $\eta_i$ is increased by one (in Step~\ref{step:update_FAR} of Algorithm~\ref{cascading_2}) after every round in which we receive feedback on product $i$. Similarly, $r_i$ is also updated (in Step~\ref{step:update_FAR} of Algorithm~\ref{cascading_2}) when we receive feedback on product $i \in [n]$---i.e., if the algorithm has received feedback $\eta_i$ times on product $i$, and $k$ out of these $\eta_i$ customers clicked on the product, then $r_i = k/\eta_i$.

The term $\sqrt{\tfrac{\log(\frac{2nT}{\delta})}{\eta_{i}}} + \frac{F}{\eta_{i}}$ in Equation \eqref{eq:add_edge} can be viewed as a ``fakeness-robust'' upper confidence interval for the estimate $r_i$ of $\mu_i$. This quantity depends on the fakeness budget $F$, which represents the fact that the algorithm does not fully trust its estimate of the $\mu_i$ due to the presence of fake users. Given that there are at most $F$ fake users, the term $\frac{F}{\eta_i}$ captures the maximum amount by which these fake users can distort our empirical estimate of product $i$'s reward after $\eta_i$ rounds of feedback. Therefore, according to Equation \eqref{eq:add_edge}, the algorithm adds the edge $(i, j)$ to graph $G$ when our upper bound on the reward of product $i$, i.e., $r_{i} + \sqrt{\tfrac{\log(\frac{2nT}{\delta})}{\eta_{i}}} + \frac{F}{\eta_{i}}$,  is smaller than our lower bound on that of product $j$, i.e., $r_{j} - \sqrt{\frac{\log(\tfrac{2nT}{\delta})}{\eta_{j}}} - \frac{F}{\eta_{j}}$. In simple terms, even our \emph{worst-case} estimate of the click probability of product $j$ is larger than our \emph{best-case} estimate for that of product $i$.

\begin{algorithm}[htbp]
\caption{Known Fakeness Budget: Fake-Aware Ranking (FAR)}\label{cascading_2}
\begin{algorithmic}[1]
\State \textbf{Input:} The fakeness budget $F$ and number of rounds $T$ .  
\State \textbf{Output:} For each round $t\in [T]$, a ranking $\pi_t$.
\State \textbf{Initialization.} \label{step:initialization_FAR} Let $\delta = \frac{1}{nT}$. For all $i\in [n]$, initialize the average rewards $ r_{i} \gets 0$ and  feedback counts $\eta_{i} \gets 0$. Further, initialize the product-ordering  graph $G \gets ([n],\emptyset)$.
\For {$t=1, \ldots, T$,}
\State  \textbf{Ranking Decision.} Display the products according to ranking $\pi_t=$ \Call{GraphRankSelect}{${\etab}, G$}, and observe $c_{t} \in (\{\emptyset\} \cup [n]) \times [n]$ (i.e., the clicked product and last browsed position).  
\State \textbf{Update Variables.} \label{step:update_FAR} Let $j$ be the last position that the customer examined. Increment the feedback counts $\eta_i$ for every $i\in \{\pi_t({1}),\ldots,\pi_t({j})\}$, by one. 
Then, for all $i\in \{\pi_t(1), \ldots, \pi_t(\p{j}-1)\}$, update the average rewards {as $ r_{i}\gets \frac{ r_{i}(\eta_{i}-1)}{\eta_i}$}. 
For $i = \pi_t(\p{j})$, if the customer clicked on $i$, update $ r_{i}\gets \frac{ r_{i}(\eta_{i}-1)+1}{\eta_i}$, otherwise $ r_{i}\gets \frac{ r_{i}(\eta_{i}-1)}{\eta_{i}}$.
\State \textbf{Add Edges to Graph $G$.} For each $i, j \in [n]$ with $\eta_{i},\eta_{j} >0$ such that
$$ r_{i} + \sqrt{\frac{\log(\frac{2nT}{\delta})}{\eta_{i}}} + \frac{F}{\eta_{i}} \leq  r_{j} - \sqrt{\frac{\log(\frac{2nT}{\delta})}{\eta_{j}}} - \frac{F}{\eta_{j}},$$
add a directed edge $(i,j)$ to $G$.
\EndFor
\end{algorithmic}

\end{algorithm}

\textbf{Ranking Decision.} The FAR algorithm chooses its ranking $\pi_t$ in round $t$ via the \Call{GraphRankSelect}{${\etab}, G$} function, defined in Algorithm \ref{alg_graph_rank}. In particular, this function uses the product-ordering graph $G$ and feedback counts $\etab$ to output a ranking that a corresponds to a topological ordering of graph $G$, breaking ties in favor of products with a low feedback count. Concretely, \Call{GraphRankSelect}{${\etab}, G$} assigns products to ranks sequentially starting with the top position. In each step, it selects a product that has no outgoing edge in $G$ to any other product; if multiple products meet this criterion, then the product with the smallest of $\eta_i$ is selected. Following this, we update the graph by removing this product and its edges, and then repeat the selection process for the next position. Roughly speaking, our selection algorithm balances \emph{exploitation with exploration}---(a) products without outgoing edges are placed at better positions as they are at least as good as the remaining products, thereby exploiting prior feedback; and (b) ties are broken in a manner that ensures we collect more information on products with low feedback counts, leading to more exploration. We remark that any deletion of edges inside the \Call{GraphRankSelect}{${\etab}, G$} function do not alter the product-ordering graph outside of it. Finally, we note that the function returns an arbitrary ranking when graph $G$ has a cycle, which is indicative of contradictory information regarding pairwise product relationships. Nevertheless, as we show in Lemma~\ref{lem_pairwise_edgeadd}, the probability  that graph $G$ contains a cycle is very small.

\begin{example}
In Figure \ref{fig:Graph}, we illustrate via a toy example how the function \Call{GraphRankSelect}{$\etab, G$} chooses a ranking for a given product-ordering graph $G$ and feedback counts $\etab$. In this example, there are $n=6$ products, and as shown in Figure~\ref{fig:Graph}, both products two and four do not have any outgoing edges. Then, function \Call{GraphRankSelect}{$\etab, G$} places these two products in the first two positions. However, since $\eta_4< \eta_2$, the algorithm prioritizes  product four over product two by placing it in the first position. After removing products two and four and their associated edges from graph $G$, product one is the only product with no outgoing edge, and as a result, is placed in position three. Continuing this process leads to ranking $(4,2,1, 3, 5, 6)$.\end{example}

\begin{figure}
    \centering
    \includegraphics[width=0.3\textwidth]{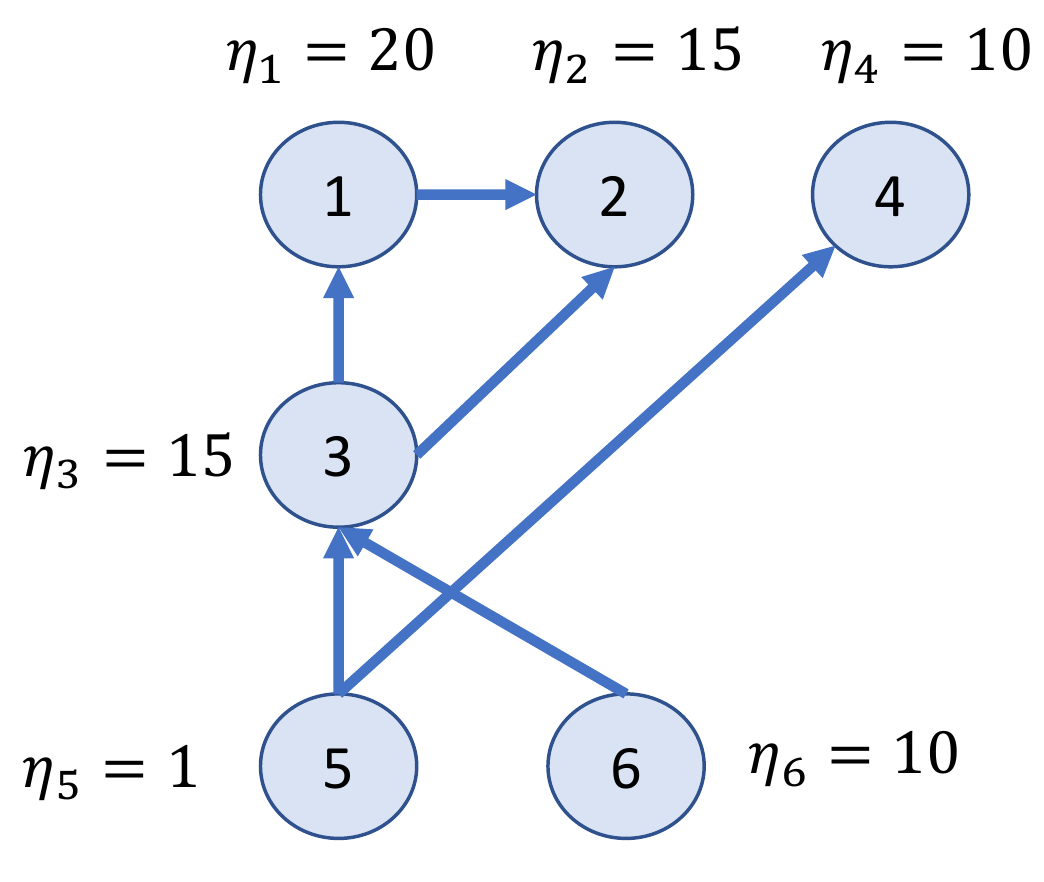}
   \caption{By applying function \textsc{GraphRankSelect}{$({\boldsymbol{\eta}}, G)$} to this instance, we obtain ranking $(4,2,1,3, 5,6)$. Here, $\boldsymbol{\eta}=(20,15, 15, 10, 1, 10)$.}
    \label{fig:Graph}
\end{figure}

{Having described the FAR algorithm, in the following theorem, we characterize its regret.}

\begin{theorem}[Known Fakeness Budget: Regret of FAR Algorithm] 
\label{thm:far} Let $F$ be the fakeness budget and assume that $F$ is known to the platform. Then, the expected regret of the FAR algorithm  satisfies \label{thm_far} 
\begin{align}
    \reg_T &= O\left(n^2F+ \sum_{\p{j}=1}^n \sum_{i=\p{j}+1}^n \frac{\log(nT)}{\Delta_{j,i}} \right)\,.
    \label{eqn_thmfar_regret}
    \end{align}

\end{theorem}
\begin{algorithm}[ht]\captionsetup{labelfont={sc,bf}}
\caption{Function \sc{GraphRankSelect} for Ranking Decision}\label{alg_graph_rank}
\begin{algorithmic}[1]
\State \textbf{Input.} Product-ordering graph $G$ and feedback count $\etab =(\eta_1, \eta_2, \ldots, \eta_n)$.
\State \textbf{Output.} Ranking $\pi$.
\Function{GraphRankSelect}{${\etab}, G$}
\State Set $S=[n]$ and $\hat G= G$. 
\State \textbf{If} graph $G$ has  a cycle, \textbf{then} return an arbitrary ranking $\pi$. {\texttt{(In the analysis of the FAR algorithm, we show that this is a rare event.)}} 
\State \textbf{Else}, \textbf{for} $\p{j}= 1,\ldots, n$ \textbf{do} 
\State \quad Let $i$ be a product in $S$ that has no outgoing edges to $S$ in graph $\hat G$. If multiple products in $S$ satisfy this condition, choose the one with smallest $\eta_i$ value.
\State \quad Place product $i$ in position $ j$, i.e., $\pi( j)= i$.
 \State \quad Remove $i$ from set $S$, i.e., {$S \gets  S \backslash \{i\}$}; remove node $i$ and all of its  incoming  edges from graph $\hat G$.
 \State \textbf{end for}
 \EndFunction
\end{algorithmic}
\end{algorithm}

Before proceeding to the proof of Theorem~\ref{thm:far}, we highlight two key ideas: (a) 
The threshold for adding a directed edge between products $i$ and $j$ (presented in Equation \eqref{eq:add_edge}) is set such that, with high probability, the product-ordering graph never contains an erroneous edge, even in the presence of fake users; (b) The \Call{GraphRankSelect}{$\etab, G$} function (Algorithm~\ref{alg_graph_rank}), when deciding on which product to place in a given position,  prioritizes the least explored product (among those with no outgoing edges). As a result, any product $i$ that is worse than another product $j$ cannot be placed in a position more visible than that of $j$ ``too many" times as its feedback count would rapidly exceed that of product $j$.

\begin{proof}{Proof of Theorem \ref{thm:far}.}
Recall that $\pi^\star$ is the optimal ranking and $\pi_t$ is the ranking selected by our learning algorithm in round  $t$. Fix some arbitrary policy $\vec{P} = (P_t)_{t \in [T]} \in \mathcal{P}$ for the fake user. Our goal is to show that for any choice of $\vec{P}$, the regret $\reg_T(\vec{P})$ incurred by our learning algorithm belongs to $O\left(n^2F+ \sum_{\p{j}=1}^n \sum_{i=\p{j}+1}^n \frac{\log(nT)}{\Delta_{j,i}} \right)$. Note that by Equation~\eqref{eqn_regret_F}, the regret $\reg_T(\vec{P})$ can be written in the form:
\begin{align}\label{eq:def:regret}\reg_{T}(\bm{P}) &~=~  \E_{\mathcal{H}_T(\pi^{\star})}\Big[\sum_{t=1}^T \mathcal{C}_{r,t}(\pi^\star)\Big]- \E_{\mathcal{H}_T(\pi_t)}\Big[\sum_{t=1}^T \mathcal{C}_{r,t}(\pi_t) \Big]\,, 
\end{align}
where the histories $\mathcal{H}_T(\pi^{\star})$ and $\mathcal{H}_T(\pi_t)$ depend on the policy $\vec{P}$, the randomness in the customer's actions in each round, and the ranking algorithm itself. When it is clear from context, we omit the history over which the expectation is taken.

Our proof now proceeds as follows. First, in Lemma~\ref{lem_regret_decomp}, we decompose the regret into terms corresponding to the loss resulting from misplacing any inferior product at a given position $\p{j}$ for all $j \in [n]$.
Then, in Lemma~\ref{lem:concent}, we show that with high probability, the algorithm's empirical means for the click probability of each product are close to the actual click probabilities. Following this, in Lemmas~\ref{lem_pairwise_edgeadd} and~\ref{lem_dagprop}, we provide an upper bound for the number of times that an inferior product is misplaced in position $\p{j} \in [n]$. The proofs for all of these lemmas are deferred to Appendix~\ref{sec:proof:FAR}.

We begin by decomposing the regret into a sum of ``pairwise regrets'' for each pair of products across the entire time horizon.

{
\begin{lemma}[Decomposing  Regret]
\label{lem_regret_decomp}
The regret of an algorithm $(\pi_t)_{t=1}^T$ compared to the optimal ranking $\pi^{\star}$ satisfies

\begin{align}
\E\Big[\sum_{t=1}^T \mathcal{C}_{r,t}(\pi^\star)\Big]- \E\Big[\sum_{t=1}^T \mathcal{C}_{r,t}(\pi_t) \Big] \leq F + \sum_{t=1}^T \sum_{{j}=1}^n \sum_{i=j+1}^n \Delta_{j, i} \E_{\mathcal{H}_t(\pi_t)}\big[\ind(\event_{r,t}(\pi_t(\p{j}) = i))\big]\,,\label{eqn_regret_decomp1}
\end{align}

where $\event_{r,t}(\pi_t(\p{j}) = i)$ denotes the event where: (a) product $i$ is placed in position $\p{j}$ under $\pi_t$, i.e., $\pi_t(\p{j}) = i$, (b) the customer in round $t$ is real, and (c) we receive feedback on product $i$ (that is, the customer does not exit before position $\p{j}$). 
\end{lemma}}

Looking at the right hand side of Equation \eqref{eqn_regret_decomp1}, it is tempting to conclude that the only impact of the fake users on regret is the additive term $F$. However, this is inaccurate as the realization of the random variable\footnote{Recall that $\ind(Y)$ is a random variable that evaluates to one when the event $Y$ is true and zero otherwise.} $\ind(\event_{r,t}(\pi_t(\p{j}) = i))$ depends on $\pi_t$, which in turn crucially depends on the fake user's actions in rounds prior to $t$. For example, the products boosted by a fake user (via clicks) in previous rounds may be perceived as high quality products and be ranked in the top positions under $\pi_t$. In the remainder of this proof, we will focus on bounding the deleterious effects of fake users on future regret.

Following Equation~\eqref{eqn_regret_decomp1}, consider the term $\E\left[\ind(\event_{r,t}(\pi_t(\p{j}) = i))\right]$, which represents the event that product $i$ is placed in a position $\p{j}$ that is better than its position in the optimal ranking $\pi^{\star}$ (i.e., $i > \p{j})$, and a real customer examines product $i$. In particular, such a product placement is clearly sub-optimal compared to the optimal ranking and could potentially contribute to the regret; so we seek to bound the number of times this sub-optimal event occurs. 

In order to analyze $\E\left[\ind(\event_{r,t}(\pi_t(\p{j}) = i))\right]$, we first condition it on the event that in every round $t \in [T]$, the true click probability $\mu_i$ of any product $i \in [n]$, is within an appropriate confidence interval around its empirical mean. More precisely, for any $t \in [T]$, and product $i \in [n]$, we extend our notation from Algorithm~\ref{cascading_2} and use $\eta_{i,t}$ to denote the number of times we receive feedback on product $i$ up to round $t$, and $r_{i,t}$ to denote the average reward we have observed for product $i$ until round $t$. For any product $i\in [n]$ and round $t\in [T]$, define the random event
\begin{align}\label{eq:event_e}\mathcal E_{i,t}\triangleq\{ r_{i,t} - w_{i,t} \leq \mu_i \leq r_{i,t} + w_{i,t}\}\,,\end{align} 
where for each $t\in [T]$ and $i\in [n]$, the window size (confidence interval) $w_{i, t}$ is defined as
 \begin{align}
 \label{eqn_window_definition1}w_{i,t}  \triangleq  \left\{ \begin{array}{ll}
         \sqrt{\frac{\log(\frac{2nT}{\delta})}{\eta_{i,t}}} + \frac{F}{\eta_{i,t}} &\quad  \text{if~} \eta_{i,t} > 0 ;\\
        \max\left(1,  \sqrt{\log(\frac{2nT}{\delta})} + F \right) &\quad  \text{if~} \eta_{i,t} = 0\,. \end{array} \right. \end{align}
Under event $\mathcal{E}_{i,t}$, the true click probability $\mu_i$ of product $i$ lies within a confidence window of width $w_{i,t}$ around the empirical reward $r_{i,t}$, at the end of round $t$.

Finally, we define $\mathcal E =\{\bigcap_{i\in [n], t\in [T]} \mathcal E_{i,t}\}$ as the intersection of all the events $\mathcal E_{i,t}$ for every product $i\in [n]$ and round $t\in [T]$. We show that $\mathcal E$ holds with high probability. 

\begin{lemma}[Concentration Inequality]\label{lem:concent} The probability that the event $\mathcal{E}$ holds is at least $1 - \frac{\delta^2}{nT}$, where $\mathcal E =\{\bigcap_{i\in [n], t\in [T]} \mathcal E_{i,t}\}$ and $ \mathcal E_{i,t}$ ($i\in [n]$ and $t\in [T]$)  is defined in Equation \eqref{eq:event_e}, and $\delta=\frac{1}{nT}$.
\end{lemma}

Going back to the second term in the right hand side of Equation~\eqref{eqn_regret_decomp1}, we have that: 
\begin{align}
    \sum_{\p{j}=1}^n \sum_{i=\p{j}+1}^n \Delta_{\p{j}, i}\sum_{t=1}^T \E_{\mathcal{H}_T(\pi_t)}\left[  \ind(\event_{r,t}(\pi_t(\p{j}) = i))\right] &=   \P(\mathcal{E})\sum_{\p{j}=1}^n \sum_{i=\p{j}+1}^n \Delta_{\p{j}, i}\sum_{t=1}^T\E_{\mathcal{H}_T(\pi_t)}\left[\ind(\event_{r,t}(\pi_t(\p{j}) = i))~ \big| \mathcal E \right] \nonumber \\  
    & +  \P(\mathcal{E}^c)\sum_{\p{j}=1}^n \sum_{i=\p{j}+1}^n \Delta_{\p{j}, i}\sum_{t=1}^T\E_{\mathcal{H}_T(\pi_t)}\left[\ind(\event_{r,t}(\pi_t(\p{j}) = i))~\big| \mathcal{E}^c\right].
     \label{eqn_conditional_regret_1}
\end{align}
Here, the term $\E_{\mathcal{H}_T(\pi_t)}\left[\ind(\event_{r,t}(\pi_t(\p{j}) = i)) \big| \mathcal E \right]$ restricts the expectation to (be conditional on) histories where event $\mathcal E$ is true. We omit the  subscript $\mathcal{H}_T(\pi_t)$ from the expectation when its meaning is clear. From Lemma \ref{lem:concent}, note that 
\begin{equation}\label{eq:UB:comp}
\P(\mathcal{E}^c) \leq \frac{\delta^2}{nT}
\end{equation}
Using this upper bound on the probability of $\mathcal{E}^c$, 
we simplify the second term on the right hand side of Equation \eqref{eqn_conditional_regret_1} as follows:
\begin{align}
  \P(\mathcal{E}^c)\sum_{\p{j}=1}^n \sum_{i=\p{j}+1}^n \Delta_{\p{j}, i}\sum_{t=1}^T\E\left[\ind(\event_{r,t}(\pi_t(\p{j}) = i))~\big| \mathcal{E}^c\right]
& \leq 
\frac{\delta^2}{nT} \sum_{\p{j}=1}^n \sum_{i=\p{j}+1}^n  \sum_{t=1}^T\E\left[\ind(\event_{r,t}(\pi_t(\p{j}) = i))~\big| \mathcal{E}^c\right] \nonumber \\
& \leq 
\frac{\delta^2}{nT} \sum_{\p{j}=1}^n \sum_{i=\p{j}+1}^n  \sum_{t=1}^T\E\left[\ind(\pi_t(\p{j}) = i)~\big| \mathcal{E}^c\right] \nonumber \\
& =
\frac{\delta^2}{nT} \sum_{\p{j}=1}^n \sum_{t=1}^T \E\left[\sum_{i=\p{j}+1}^n \ind(\pi_t(\p{j}) = i)\Big| \mathcal{E}^c\right]  \nonumber\\
& \leq 
\frac{\delta^2}{nT} \sum_{\p{j}=1}^n \sum_{t=1}^T 1 = \delta^2. 
\label{eqn_regret_conditional}
\end{align}
The first inequality above follows from Equation~\eqref{eq:UB:comp} and the fact that $\Delta_{\p{j}, i} \leq 1$. The second one is due to the observation that $\ind(\event_{r,t}(\pi_t(\p{j}) = i)) \leq \ind(\pi_t(\p{j}) = i)$ because $\pi_t(\p{j}) = i$ is a necessary condition for the event $\event_{r,t}(\pi_t(\p{j}) = i)$ to occur. Finally, the last inequality holds because $\ind(\pi_t(\p{j}) = i)$ is true for at most one product $i > \p{j}$ in round $t$.

Next, we turn our attention to characterizing the loss under event $\mathcal{E}$.  Namely, we focus on providing an upper bound for the expression $\sum_{t=1}^T \E\left[ \ind(\event_{r,t}(\pi_t(\p{j}) = i)) \big| \mathcal E \right]$ from Equation~\eqref{eqn_conditional_regret_1} for some arbitrary $i,j \in [n]$ such that $i > \p{j}$. In the following, we first show that if we have enough samples (or feedback) from both products $i$ and $\p{j}$, we will not place product $i$ at a better position compared to product $j$. Then, we upper bound the number of rounds needed to ensure that we have enough samples. To that end, for  $i>\p{j}$, we define $\gamma_{\p{j},i}$ as:
\begin{equation}
    \label{eqn_sub_optplays}
    \gamma_{\p{j},i} \triangleq \max\left\{\frac{64\log(2nT/\delta)}{\Delta^2_{\p{j},i}}, \frac{8F}{\Delta_{\p{j},i}}\right\}\,.
    \end{equation}  
Observe that for any $i > \p{j}$, if $\eta_{i,t} \geq \gamma_{\p{j},i}$, then $w_{i,t} \leq  \frac{\Delta_{\p{j}, i}}{4}$ as per our definition in Equation~\eqref{eqn_window_definition1}. The same holds for product $\p{j}$. Therefore, if we receive enough samples from both products $i$ and $\p{j}$, under event $\mathcal{E}$, the condition~\eqref{eq:add_edge} will be satisfied. This allows our algorithm to correctly add an edge $(i,\p{j})$ to the product-ordering graph  $G$. 
The following lemma formalizes the above observation. More precisely, it states that conditioned on event $\mathcal{E}$, the graph $G$ does not contain  erroneous edges.

\begin{lemma}[Properties of the Graph $G$ under Event $\mathcal{E}$]
\label{lem_pairwise_edgeadd}
Assume that the event  $\mathcal E$ holds.
Then, for any two products $i, \p{j}$ with $i > \p{j}$:
\begin{enumerate}
    \item In any round $t$, the product-ordering graph $G$ does not contain the (incorrect) edge $(\p{j},i)$.
      \item  Suppose that in round $t$, $\eta_{i,t}, \eta_{\p{j},t} \geq \gamma_{\p{j},i}$, where $\gamma_{\p{j},i}$ is defined in~\eqref{eqn_sub_optplays}. Then, the product-ordering graph $G$ contains {the (correct) edge $(i,\p{j})$} for all rounds $t'\ge t$.
\end{enumerate}

\end{lemma}

Lemma \ref{lem_pairwise_edgeadd} implies that under event $\mathcal E$,  graph $G$ does not have any erroneous edges, and when the number of {samples (feedback) collected on any two  products $i$ and $\p{j}$} is large enough, there is a correct edge between their corresponding nodes in the graph $G$. Recall that a directed edge from product $i$ to product $j$ in $G$ is indicative of the fact that product $j$ has a higher click probability.  Our next lemma shows that, conditioned on event $\mathcal E$, any mistake in ranking an inferior product $i$ ahead of product $\p{j} < i$ is due to lack of feedback on product $i$. Further, this is also true when product $i$ is placed at position $\p{j}$, a scenario whose occurrence we are seeking to bound.
\begin{lemma}[{Necessary Conditions for an Incorrect Ordering}]
	\label{lem_dagprop}
	Assume that event  $\mathcal E$ holds. 
	Suppose that the function \textsc{GraphRankSelect}{($\etab, G$)} {(Algorithm~\ref{alg_graph_rank})} ranks product $i$ at a better position compared to product $\p{j}$ in round $t+1$ for some $i >\p{j}$. Then, there must exist some product $\ki \leq \p{j}$ such that {(a)} the {product-ordering graph $G$ does not include the edge $(i,k)$} after round $t$ {(b)} $\eta_{i,t} \leq \eta_{\ki,t}$ (i.e. we have less feedback on product $i$ than on product $\ki$),
	and {(c)} $\eta_{i,t} < \gamma_{\ki,i}$, where $\gamma_{\ki,i}$ is as defined in Equation \eqref{eqn_sub_optplays}. 
\end{lemma}

By utilizing Lemma~\ref{lem_dagprop}, we will now show that for  any  $i>\p{j}$, we have   $\sum_{t=1}^T \E\big[ \ind(\event_{r,t}(\pi_t(\p{j}) = i)) \big| \mathcal E \big] \le \gamma_{j,i}$. Note that this completes the proof because by Equations \eqref{eqn_regret_decomp1}, \eqref{eqn_conditional_regret_1}, and  \eqref{eqn_regret_conditional}, we get 
\begin{align}
    \reg_T(\vec{P}) & ~\leq~  F+ \sum_{\p{j}=1}^n \sum_{i=\p{j}+1}^n\Delta_{\p{j}, i} \left(\sum_{t=1}^T \E\left[ \ind(\event_{r,t}(\pi_t(\p{j}) = i)) \big| \mathcal E \right] \right) + \delta^2 \notag\\
    & \leq~ F+ \sum_{\p{j}=1}^n \sum_{j=i+1}^n\Delta_{\p{j}, i}\gamma_{\p{j},i} +  \delta^2\notag\\
     & \leq F+1 + \sum_{\p{j}=1}^n \sum_{i=\p{j}+1}^n \max\left\{\frac{64\log(2nT/\delta)}{\Delta_{j, i}}, 8F\right\}\notag .
   \end{align}

The final inequality comes from substituting the expression for $\gamma_{\p{j},i}$ (as given in~\eqref{eqn_sub_optplays}), and using $\delta^2 \leq 1$. Finally, plugging $\delta=\frac{1}{nT}$, we get the following expression, which leads to the theorem. 
\begin{align}
    \reg_T(\vec{P}) & \leq F+1 + \sum_{\p{j}=1}^n \sum_{i=\p{j}+1}^n \max\left\{\frac{128\log(\sqrt{2}nT)}{\Delta_{j, i}}, 8F\right\} \label{eqn_insidethm_regret}
   \end{align}

It remains for us to show that for  any  $i>\p{j}$, we have   $\sum_{t=1}^T \E\big[ \ind(\event_{r,t}(\pi_t(\p{j}) = i)) \big| \mathcal E \big] \le \gamma_{j,i}$. Recall that event $\event_{r,t}(\pi_t(\p{j}) = i)$ is true only if  (a) product $i$ is placed in position $\p{j}$ under ranking $\pi_t$, and (b) a real customer examines it (i.e., we receive feedback on product $i$)\footnote{By definition, $\event_{r,t}(\pi_t(\p{j}) = i)=0$ if the user in round $t$ is fake.}.

Suppose that  event $\event_{r,t}(\pi_t(\p{j}) = i)$  occurs. Since $i > \p{j}$, this implies there must exist at least one product as good as  product $j$ that is (sub-optimally) ranked below position $\p{j}$---this follows from the pigeonhole principle. Mathematically, we can write this as
\begin{equation}
\label{eqn_ordercatch}
     \pi_t(\p{j}) = i; ~~ i > \p{j} \implies \exists {s} \text{~with~} s \leq \p{j} \text{~s.t.~} \pi^{-1}_t({s}) > \p{j}.
\end{equation}
 Note that here, we identify a product $s \leq \p{j}$ which is placed below product $i$ under ranking $\pi_t$.
 That is, $\pi_t^{-1}(s) > \pi_t^{-1}(i)$ while $s < i$. Thus, we can apply Lemma~\ref{lem_dagprop} to the pair of products $i$ and $s$. 
 Lemma~\ref{lem_dagprop} implies that  there exists some product $\ki$ better than (or equal to) $s$ such that $i$ has no edge to $\ki$ in graph $G$ at round $t$ and that $\eta_{i,t-1} \leq \eta_{\ki,t-1}.$  Note that this also implies that $\eta_{i,t-1} < \gamma_{\ki, i}$; if not, then we would have that $\eta_{i,t-1}, \eta_{\ki, t-1} \geq \gamma_{\ki, i}$. By Lemma~\ref{lem_pairwise_edgeadd}, this would imply that $G$ contains an edge from $i$ to $\ki$, contradicting our construction of $\ki$.
 
 Combining this and the fact that $\mu_{i} < \mu_{\p{j}} \leq \mu_{s} \leq \mu_{\ki}$, we can conclude that
\begin{align*}
  \eta_{i,t-1} < \gamma_{\ki,i}  \leq \gamma_{s,i} \leq \gamma_{\p{j},i}.  
\end{align*}
We have therefore shown that conditioned on event $\mathcal E$, if in some round $t$ and for some $i>\p{j}$, $\event_{r,t}(\pi_t(\p{j}) = i)$ holds, we must have that $\eta_{i,t-1} < \gamma_{\p{j},i}$. Conditional on $\mathcal E$, consider any arbitrary instantiation of our algorithm, and let $t_{\p{j}, i}$ denote the first round at which $\eta_{i,t} = \gamma_{\p{j},i}$. {We claim that} for any $t > t_{\p{j}, i}$, our algorithm  never {places} product $i$ at position $\p{j}$ since this would result in a contradiction. To see why this would be a contradiction, assume that $\pi_t(\p{j}) = i$ {in some round} $t > t_{\p{j}, i}$. According to Equation \eqref{eqn_ordercatch}, we can infer the existence of a product $s$ such that $s < \p{j}$ and $\pi^{-1}_t({s}) > \p{j} = \pi^{-1}_t({i})$. Next, we invoke Lemma~\ref{lem_dagprop} with the pair of products $i$ and $s$ to conclude that  $\eta_{i,t-1} < \gamma_{s,i}$. {However, we have} $\gamma_{s,i} \leq \gamma_{\p{j},i}$, which is a contradiction because, by definition, $\eta_{i,t-1} \geq \gamma_{\p{j},i}$ for any $t > t_{\p{j}, i}$.

Therefore, conditional on $\mathcal E$, we have 
\begin{align*}
\sum_{t=1}^T \E\left[ \ind(\event_{r,t}(\pi_t(\p{j}) = i)) \big| \mathcal E \right] & = \E\left[\sum_{t=1}^{t_{\p{j}, i}} \ind(\event_{r,t}(\pi_t(\p{j}) = i)) \big| \mathcal E \right]\\
    & \leq \E\left[\sum_{t=1}^{t_{\p{j}, i}}\eta_{i,t} - \eta_{i,t-1} {\Big| \mathcal E} \right]  = \E\left[\eta_{i,t_{\p{j}, i}} {\Big| \mathcal E} \right]  = \gamma_{\p{j},i}\,,
\end{align*}
{where the first and last equalities follow from the definition of $t_{\p{j}, i}$, and the inequality follows from the observation that event $\event_{r,t}(\pi_t(\p{j}) = i)$ is only a sufficient condition on getting feedback on product $i$ in any round $t$.} This completes our proof. \hfill \Halmos

\end{proof}

\begin{remark}
The vanilla version of the FAR algorithm needs to know the time horizon $T$ in order to set its parameter $\delta$, and the window size. However, by using the standard \emph{doubling trick} (see, e.g., \cite{bubeck2015convex}), one can modify the FAR algorithm so that the algorithm does not need to know $T$ in advance. Instead, the algorithm starts with an initial guess value for $T$, and every time, it reaches  its guess, the algorithm doubles it.
\end{remark}

\section{Unknown Fakeness Budget: Fake-Oblivious Ranking with Cross-Learning (FORC) Algorithm}
\label{sec:FORC}
In this section, we consider a setting where the platform does not know the fakeness budget $F$ in advance. We build on the ideas used in FAR to develop our second algorithm called {\em  Fake-Oblivious Ranking with Cross-Learning (FORC)}, which is agnostic to the fakeness budget. The formal description of FORC is presented in Algorithm \ref{alg_cascading_agnostic}. Here, we highlight the main ideas of the algorithm, and provide an informal description.

\textbf{Multi-Level Learning.} Recall that in the FAR algorithm, we account for the presence of fake users by widening the confidence intervals used to determine the superiority of one product over another, by an additive factor proportional to $F$; see  Equation \eqref{eq:add_edge}. Since we do not  know the value of $F$ in this case, we cannot simply widen the confidence interval. Instead, we use a multi-level randomized algorithm that at each step, selects one of $\level = \log_2(T)$ sub-algorithms running in parallel, with different probabilities. More precisely, in each round $t \in [T]$, (the learning algorithm corresponding to) level $\ell$ is selected with probability $2^{-\ell}$ for $\ell \in \{2, \ldots, \level\}$, and level $\ell=1$ is selected with probability $1 - \sum_{\ell = 2}^{\level} 2^{-\ell}$. Observe that levels with higher indices are sampled with smaller probabilities and as a result, they are exposed to a fewer number of fake users. Adopting such a multi-level mechanism allows us to limit the damage caused by fake users to the learning algorithm at any one level. Leveraging this, we later provide a concrete bound on the number of fake users that higher levels are exposed to, which serves as an ``effective fakeness budget".

As stated earlier, our idea of employing multi-level learning in FORC is inspired by the algorithm in \cite{lykouris2018stochastic}. However, due to the randomness in our feedback structure---i.e., we cannot predict exactly which products will be examined by the customers---coordinating these learning levels is rather challenging, compared to the multi-armed bandit setting in \cite{lykouris2018stochastic}. Later, we discuss  how we facilitate efficient coordination between different learning levels using careful upward and downward cross-learning across levels. 

We begin by formalizing our main rationale for a multi-level scheme by establishing a bound on the number of fake users that higher levels, i.e., those with index $\ell\ge \log_2(F)$, are exposed to. Specifically, the following lemma proves that with high probability  any level $\ell\ge \log_2(F)$ is exposed  to $O(\log(\log(T)))$ fake users.  

\begin{lemma}[Effective Fakeness Budget for High Learning Levels]
\label{lem_maxfake_levels}
For  any fake user policy $\mathbf{P}$ with fakeness budget $F$, with probability at least $1-\frac{\delta}{2}$, 
{the following holds for all levels $\ell \geq \log_2(F)$: level $\ell$ in Algorithm~\ref{alg_cascading_agnostic} is exposed to at most $\log(2\frac{\level}{\delta})+3$ fake users, where $\delta =\frac{1}{n^3T}$ and $\level =\log_2(T)$.} 
\end{lemma}

The proof of this lemma is presented in Appendix~\ref{app:FORC}. Informally, the above lemma implies that the effective fakeness budget for the $\ell$-th learning level (when $\ell \geq \log_2 (F)$) is $\log(\frac{2\level}{\delta})+3$, where $\delta =\frac{1}{n^3T}$. As such, it seems intuitive to replace $F$ in Equation \eqref{eq:add_edge} (window size used in the FAR algorithm) by this term. Based on this, we set the confidence interval for each level in Algorithm~\ref{alg_cascading_agnostic} to be slightly larger than $\log(2\frac{\level}{\delta})+3$ to account for this effective fakeness budget, as well as some extra corruption that arises due to cross-learning (see Step \eqref{step_addedge_forc} in Algorithm~\ref{alg_cascading_agnostic}).

\textbf{Product-Ordering Graphs and Cross-Learning.} In the FORC algorithm, each learning level $\ell$ is represented by a product-ordering graph $G^{(\ell)}$. Analogous to Algorithm~\ref{cascading_2}, the edges in the graph represent dominance relationships between pairs of products, which in turn determine the ranking decisions when level $\ell$ is randomly selected. As a consequence of Lemma~\ref{lem_maxfake_levels}, one can show using similar techniques as in the proof of FAR that levels $\ell \geq \log_2(F)$ always converge to the optimal ranking as their product-ordering graphs do not contain erroneous edges. However this claim does not hold for levels $\ell < \log_2(F)$, which witness more fake users, and can incur large regret. At the same time, since the fakeness budget and hence, $\log_2(F)$, are unknown, one cannot simply ignore these lower levels. Instead, we introduce the notion of bi-directional cross-learning---described below---to a) transfer edges from higher levels to lower ones in order to identify which levels contain incorrect edges , and b) speed up the learning process at higher levels.

\begin{itemize}[leftmargin=*
]

 \item \textbf{Downward Cross-Learning.}  In a given round, if the FORC algorithm adds an edge $(i,j)$ to graph $G^{(\ell)}$ corresponding to level $\ell$, the same edge is also added to graphs $G^{(\ell')}$ for all $\ell' < \ell$. The intuition behind this downward propagation is that the higher levels are more cautious in learning product orderings and an inference made at this level is more likely to be accurate than one made at a lower level. Further, downward cross-learning may lead to the formation of cycles in the product-order graphs at lower levels due to the presence of contradictory edges; we eliminate these graphs (Step~\ref{step_cycle_FORC} of Algorithm~\ref{alg_cascading_agnostic}) to curtail the damage caused by such edges. We note that in any given round, if the algorithm selects a level $\ell$ such that its corresponding graph $G^{(\ell)}$ was previously eliminated, then we default to graph $G^{(\ell')}$ for our ranking decisions, where $\ell' > \ell$ is the smallest level whose product-ordering graph has not yet been eliminated (Step~\ref{step_if_eliminated} of Algorithm~\ref{alg_cascading_agnostic}).
 
 

    \item \textbf{Upward Cross-Learning.} Simply widening the confidence intervals along with downward cross-learning is not sufficient to ensure low regret. By the time a conservative higher level learns accurate product orderings (e.g., $i > j$), the lower levels---that are sampled more frequently---can make too many mistakes in their ranking decisions. We control for this by incorporating an upward cross-learning mechanism (as defined in Step~\ref{step_upwardcl_forc} in  Algorithm~\ref{alg_cascading_agnostic}). In particular, the empirical product rewards, e.g., $\hat{r}^{(\ell)}_{i}$, that determine the edges in graph $G^{(\ell)}$, depend on a weighted average of the observations from the corresponding level $\ell$ and samples obtained from lower levels $\ell' < \ell$. Using a weighted average can impact our algorithm in two ways. On one hand,
    they can increase the learning rate at higher levels by including samples acquired in lower levels. On the other hand, they can increase the number of fake users that higher levels are exposed to. We choose the weights carefully to balance the trade-off between these two factors. We highlight that upward cross-learning is one of the novel features of the FORC algorithm  that differentiates it from the prior work on learning with corrupted samples, including  that of \cite{lykouris2018stochastic}.

\end{itemize}

We defer presenting an illustrative instance to Section~\ref{sec:numerics}, where we give a detailed example to demonstrate how (i)  product-ordering graphs of different levels evolve over time (in Figure~\ref{fig_forc_visualization}), (ii)  downward cross-learning leads to the elimination of lower-level graphs, and (iii) upward cross-learning expedites the formation of correct edges in higher levels. {Having highlighted the main ideas of the FORC algorithm,
we now characterize its regret in the following theorem.  }

\begin{theorem}[Unknown Fakeness Budget: Regret of FORC Algorithm]
\label{thm_forc}
Let $F$ be the fakeness budget and assume that $F$ is unknown to the platform. 
Then, the expected regret of the FORC algorithm satisfies
\begin{align} \reg_T &  = O\left( \left(n^2F+ \log(T)\right) \sum_{\p{j}=1}^n \sum_{i=\p{j}+1}^n  \frac{\log(nT)}{\Delta_{j,i}} \right)\,.\label{eq:thm:FORS}
\end{align}

\end{theorem}

\begin{algorithm}[htbp]
\caption{Unknown Fakeness Budget: Fake-Oblivious Ranking with Cross-Learning (FORC)\label{alg_cascading_agnostic}}
\begin{algorithmic}[1]
\State \textbf{Input.}  Parameter $\delta=\frac{1}{n^3T}$, $\level =\log_2(T)$, and number of rounds $T$.
\State \textbf{Output.}
Ranking $\pi_t$, $t\in [T]$.
\State {\textbf{Initialization.} For all  $i\in [n]$ and  $\ell \in [\level]$, initialize the average rewards $ r_{i}^{(\ell)} \gets 0$ and its cross-learning version $\hat r_{i}^{(\ell)} \gets 0$ and  feedback counts $\eta_{i}^{(\ell)} \gets 0$ and their cross-learning version $\hat \eta_{i}^{(\ell)} \gets 0$.
Further, for any level $\ell \in [ \level]$,  initialize the product-ordering  graph $G^{(\ell)} \gets ([n],\emptyset$).}

\For {$t=1, \ldots, T$,}
\State  \textbf{Choose Level.} Let $\ell_t$ be the level in round $t$, where $\P(\ell_t = \ell)= 2^{-\ell}$ for $\ell\in \{2,\ldots, \level\}$. Choose $\ell_t=1$ with the remaining probability  $1-\sum_{\ell=2}^{\level}2^{-\ell}$.
 \State \label{step_if_eliminated} \textbf{If} {graph $G^{(\ell_t)}$ is not eliminated,} \textbf{then} set
$\emph{G} \leftarrow G^{(\ell_t)}$, \textbf{else}
{set $G \leftarrow G^{(\ell)}$ for the smallest index $\ell> \ell_t$ such that $G^{(\ell)}$ is not eliminated.}
\State \textbf{Ranking Decision.} Display the products according to ranking $\pi_t=$ \Call {GraphRankSelect}{${\boldsymbol{\eta}^{(\ell_t)}}, G$}, and observe $c_{t} \in (\{\emptyset\}\cup[n] \times [n])$, i.e., the clicked product and last examined position. 
\State \textbf{Update Variables.} Let $j$ be the last position that the customer  examined.  Update the feedback counts for any  $i\in \{\pi_t(1),  \ldots, \pi_t(j)\}$: $\eta^{(\ell_t)}_{i}\gets \eta^{(\ell_t)}_{i}+1$.  
Then, for any $i\in \{\pi_t(1), \ldots, \pi_t(j-1)\}$,  update the average rewards: $r^{(\ell_t)}_{i}\gets \frac{ r^{(\ell_t)}_{i}(\eta^{(\ell_t)}_{i}-1)}{\eta^{(\ell_t)}_{i}}$. If position $j$ is clicked, for $i=\pi_t(j)$, update $ r^{(\ell_t)}_{i}\gets \frac{ r^{(\ell_t)}_{i}(\eta^{(\ell_t)}_{i}-1)+1}{\eta^{(\ell_t)}_i}$, otherwise $ r^{(\ell_t)}_{i}\gets \frac{ r^{(\ell_t)}_{i}(\eta^{(\ell_t)}_{i}-1)}{\eta^{(\ell_t)}_{i}+1}$.
\State \textbf{Update Cross-Learning Variables.} \label{step_upwardcl_forc} For any  $\ell \ge \ell_t$ and $i\in \{\pi_t(1),  \ldots, \pi_t(j)\}$, first update $\hat\eta_i^{(\ell)}$
\begin{align}\hat\eta_{i}^{(\ell)} &\gets \frac{\sum_{g=1}^{\ell-1}\eta^{(g)}_{i}}{2^{\ell}}+ \eta_{i}^{(\ell)}\qquad \ell \ge \ell_t, i\in \{\pi_t(1),  \ldots, \pi_t(j)\}\label{eq:eta_cr}\end{align}
and then update $\hat r_i^{(\ell)}$
\begin{align} \label{eq:r_cr} \hat r_{i}^{(\ell)} \gets \frac{\sum_{g=1}^{\ell-1}\eta^{(g)}_{i} r_{i}^{(g)}}{2^{\ell}\hat\eta^{(\ell)}_{i}}+ \frac{\eta_{i}^{(\ell)}}{\hat\eta^{(\ell)}_{i}} r_{i}^{(\ell)} \qquad \ell \ge \ell_t, i\in \{\pi_t(1),  \ldots, \pi_t(j)\}\,.\end{align}
\State \label{step_addedge_forc}
\textbf{Add Edges to Graphs.} If for any  $\ell$ where $G^{(\ell)}$ is not eliminated, there exist $i,j \in [n]$  with $\hat \eta^{(\ell)}_{i},\hat \eta^{(\ell)}_{j} >0$ such that:
\begin{align} \label{eq:alg:edge} 
\hat {r}^{(\ell)}_i + w^{(\ell)}_{i} < {\hat r}^{(\ell)}_{j} - w^{(\ell)}_{j}
\end{align}
add a directed edge $(i,j)$ to  $G^{(g)}$ (i.e., $G^{(g)}\gets G^{(g)} \cup (i,j)$) for all $g\le \ell$ where $G^{(g)}$ is not eliminated. Here, for any $i\in [n]$,  $w^{(\ell)}_{i}  \triangleq  
          \sqrt{\frac{3}{2}\frac{\log(\frac{4nT}{\delta})}{{\hat \eta}^{(\ell)}_{i}}} +  \frac{\log(\frac{2\level}{\delta})+4}{{\hat \eta}^{(\ell)}_{i}}$.

\State \textbf{Eliminate  graphs.} \label{step_cycle_FORC}  If graph {$G^{(\ell)}$ has a cycle for any $\ell \in [\level]$}
eliminate $G^{(g)}$ for any $g\le \ell$. 
\EndFor
\end{algorithmic}

\end{algorithm}

\begin{proof}{Proof of Theorem \ref{thm_forc}}
 Throughout this proof, in order to track the variables over time,  we denote the values of  $r_{i}^{(\ell)}$, $\eta_{i}^{(\ell)}$, $\hat r_{i}^{(\ell)}$,  $\hat \eta_{i}^{(\ell)}$, and $w_{i}^{(\ell)}$  (for $i\in [n]$ and $\ell\in [\level]$),  at the end of round $t\in [T]$  by  $r_{i,t}^{(\ell)}$, $\eta_{i,t}^{(\ell)}$, $\hat r_{i,t}^{(\ell)}$,  $\hat \eta_{i,t}^{(\ell)}$, and $w_{i,t}^{(\ell)}$, respectively. Recall from Step~\ref{step_upwardcl_forc} of Algorithm~\ref{alg_cascading_agnostic} that $\hat r_{i}^{(\ell)}$ is the cross-learned empirical reward for product $i$ corresponding to level $\ell$, which is derived via a weighted average of the reward from level $\ell$ (weight: $1$) and levels smaller than $\ell$ (weight: $\frac{1}{2^{(\ell)}}$); $\hat \eta_{i}^{(\ell)}$ denotes the cross-learned feedback count, which is defined analogously.

Similar to the proof of Theorem \ref{thm_far}, we start our analysis by  decomposing the regret into a sum of pairwise regrets using  Lemma~\ref{lem_regret_decomp}. 
However, we further partition the event $\mathcal{A}_{r,t}(\pi_t(\p{j}) = i)$ into two parts based on whether the learning level sampled in round $t$ is greater or less than $\ell^{\star} \triangleq \ceil*{\log_2(F)}$.
\begin{align}\reg_{T} &~\le~ F+   \sum_{\p{j}=1}^n \sum_{i=\p{j}+1}^n \Delta_{j, i} \E\left[\sum_{t=1}^T  \ind(\event_{r,t}(\pi_t(\p{j}) = i))\right] \notag \\
& = F+   \sum_{\p{j}=1}^n \sum_{i=\p{j}+1}^n \Delta_{j, i} \left(\E\left[\sum_{t=1}^T \ind(\event_{r,t}(\pi_t(\p{j}) = i) \cap \ell_t \geq \ell^{\star})\right]+ \E\left[\sum_{t=1}^T  \ind(\event_{r,t}(\pi_t(\p{j}) = i) \cap \ell_t < \ell^{\star})\right]  \right)\,, \label{eqn_regret_conditional_agnostic}
\end{align}

Recall that $\ell_t$ denotes the learning level sampled in round $t$ and that $\E\left[\sum_{t=1}^T  \ind(\event_{r,t}(\pi_t(\p{j}) = i))\right]$ is the expected number of times that our algorithm misplaces an inferior product $i$ in position $j<i$ when the customer is real and ends up examining this product. In the rest of the proof, 
we separately bound the loss terms $\E\left[\sum_{t=1}^T \ind(\event_{r,t}(\pi_t(\p{j}) = i) \cap \ell_t \geq \ell^{\star})\right]$ and $\E\left[\sum_{t=1}^T  \ind(\event_{r,t}(\pi_t(\p{j}) = i) \cap \ell_t < \ell^{\star})\right]$.  

In order to analyze the above loss terms,  we first define a high-probability good event, denoted by $\eventforc$, under which the product-ordering graphs for levels higher than $\ell^{\star}$ do not have erroneous edges. 
In order to define $\eventforc$, we first describe the following targeted event that is specific to a given round, product, and level: for any $t \in [T]$, $i \in [n]$, and $\ell \geq \ell^{\star}$, we define $\eventforc^{(\ell)}_{i,t}$ 
to be the event that 
 the empirical cross-learning mean $\hat{r}^{(\ell)}_{i,t}$, defined in Equation \eqref{eq:eta_cr}, is within a confidence interval of width $w^{(\ell)}_{i,t}$ around the true click probability $\mu_i$ for product $i$. Recall that the window size $w^{(\ell)}_{i,t}$ is defined in Step~\ref{step_addedge_forc} in Algorithm \ref{alg_cascading_agnostic}. We
 repeat the definition here for convenience and extend it to include the case where $\hat{\eta}^{(\ell)}_{i,t} =0$. For every $i\in[n], t\in [T]$, and $\ell \in [\level]$, define $w^{(\ell)}_{i,t}$ as follows:
 \begin{align}
 \displaystyle
 w^{(\ell)}_{i,t} & =  
          \sqrt{\frac{3}{2}\frac{\log(\frac{4nT}{\delta})}{{\max(\hat \eta}^{(\ell)}_{i,t},1)}} +  \frac{\log(\frac{2\level}{\delta})+4}{{\max(\hat \eta}^{(\ell)}_{i,t},1)}
\label{eqn_windowsize_agnostic} \end{align}

Event $\eventforc^{(\ell)}_{i,t}$ is then defined as follows:
\begin{equation}
  \eventforc^{(\ell)}_{i,t} \triangleq \{ \hat{r}^{(\ell)}_{i,t} - w^{(\ell)}_{i,t} \leq \mu_i \leq \hat{r}^{(\ell)}_{i,t} + w^{(\ell)}_{i,t} \}\,.
    \label{eqn_armwidth_agnostic_2}
\end{equation}

\noindent Now, we are ready to define event $\eventforc$:
\begin{definition}[Event $\eventforc$]
\label{def:event}
Let $\eventforc_1 \triangleq  \bigcap_{\substack{t \in [T]\\ \ell \geq \ell^{\star} \\ i \in [n]}} \eventforc^{(\ell)}_{i,t}$, where for any $t \in [T]$, $i \in [n]$, and $\ell  \geq \ell^{\star}= \ceil*{\log_2(F)}$, the event $\eventforc^{(\ell)}_{i,t}$ is defined in Equation \eqref{eqn_armwidth_agnostic_2}. Further,
let $\eventforc_2 
$ be the event that in the entire time horizon,  any level $\ell \geq \ell^{\star}$ is exposed to at most $\log(2\frac{\level}{\delta})+3$ fake users. We define the event $\eventforc$ to be the intersection 
of $\eventforc_1$ and $\eventforc_2$, i.e., 
$\eventforc \triangleq \eventforc_1 \cap  \eventforc_2.$
\end{definition}

In the following lemma, we show that event $\mathcal{G}$ occurs with high probability:

\begin{lemma}[Lower Bounding $\P(\eventforc)$]
\label{lem:even:E}
The probability of event $\eventforc$, as in Definition \ref{def:event}, is at least $1-\delta$, where $\delta =\frac{1}{n^3T}$.
\end{lemma}

The proof of the lemma is deferred to Appendix \ref{sec:proof:lem:even:E}. Here,  we sketch the main arguments. Note that by definition, under event $\eventforc$, event  $\eventforc_2$ which we analyzed in Lemma \ref{lem_maxfake_levels} also holds. More specifically, under event $\eventforc_2$, the effective fakeness budget for levels $\ell  \geq \ell^{\star}$ is $\log(2\frac{\level}{\delta})+3$. This, in turn, ensures that we have set the confidence intervals, i.e., $w^{(\ell)}_{i,t}$, to be large enough to account for the effective fakeness budget  that a learning level $\ell  \geq \ell^{\star}$ is exposed to. 
Using this observation,  we can establish a concentration result similar to Lemma \ref{lem:concent} to bound the probability of event $\eventforc_1$ conditioned on $\eventforc_2$. Recall that $\eventforc= \eventforc_1 \cap \eventforc_2$.
Finally note that in Lemma \ref{lem_maxfake_levels}, we show that event
$\eventforc_2$ also occurs with high probability. This leads to our result in Lemma \ref{lem:even:E}.

 Utilizing the regret decomposition in Equation~\eqref{eqn_regret_conditional_agnostic} and  event $\eventforc$, the rest of the proof consists of 
 three parts which we summarize in the following three lemmas.  As in Theorem~\ref{thm_far}, we also define constants $(\gamma_{j,i})_{j,i \in [n]}$ that capture the minimum number of plays for any pair of products so that the window size or confidence interval is smaller than the gap in rewards $\Delta_{j,i}$. Formally, for any given $j,i \in [n]$, we have that:
         \begin{equation}
        \label{eqn_gamma_agnostic}
    \gamma_{j,i} \triangleq \frac{64\log(4nT/\delta)}{\Delta^2_{j,i}}\,.
         \end{equation}

\begin{lemma}[Upper Bounding Loss for $\ell \geq \ell^{\star}$ under Event $\eventforc$]
\label{lem:parts1}
For any $j<i$, we have
\begin{align*}
       \P(\eventforc) \E\left[\sum_{t=1}^T \ind(\event_{r,t}(\pi_t(\p{j}) = i) \cap \ell_t \geq \ell^{\star}) \bigg| \eventforc \right] 
       \leq \frac{64\level\log(4nT/\delta)}{\Delta^2_{j,i}} \,,
    \end{align*}
    where event $\eventforc$ is defined in Definition \ref{def:event}, $\delta =\frac{1}{n^3T}$, and $\level =\log_2(T)$. 
\end{lemma}

Lemma \ref{lem:parts1} shows that under the good event $\eventforc$, the expected  total loss of our FORC algorithm due to misplacing product $i$ in position $j<i$ over all the rounds $t$ with $\ell_t\ge \ell^{\star}$ does not depend on the fakeness budget $F$ and is upper bounded by  $\frac{64\log(4nT/\delta)}{\Delta^2_{j,i}} \level = O(\frac{\log(nT)}{\Delta^2_{j,i}} \log(T))$.  To show this result, we first establish that under event $\eventforc$, at any level $\ell\ge \ell^{\star}$, (a) the product-ordering graph $G^{(\ell)}$ does not contain any erroneous edge and (b) the graph has 
the  correct edge between any two products $i$ and $j$ after obtaining enough feedback from them
i.e., when $\hat{\eta}^{(\ell)}_{i,t}, \hat \eta^{(\ell)}_{j,t} \geq \gamma_{j,i}$, where $\gamma_{j,i}$ is defined in Equation \eqref{eqn_gamma_agnostic}. We further show that the only reason for mistakenly placing product $i$ in position $j<i$ is lack of sufficient feedback.
 These results then allow us to bound $\E\left[\sum_{t=1}^T \ind(\event_{r,t}(\pi_t(\p{j}) = i) \cap \ell_t \geq \ell^{\star}) \big| \eventforc \right]$. In particular, we show that after obtaining $\gamma_{j,i}$ rounds of feedback on product $i$, under the event $\eventforc$ and learning level $\ell\ge \ell^{\star}$, event $\event_{r,t}(\pi_t(\p{j}) = i)$ does not happen, where $j<i$. 
Recall that for any $j<i$, event $\event_{r,t}(\pi_t(\p{j}) = i)$ holds when we place product $i$ in position $j$ and we receive feedback on product $i$ from a real customer.

\begin{lemma}[Upper Bounding Loss for $\ell < \ell^{\star}$ under Event $\eventforc$]
\label{lem:parts2} For any $j<i$, we have 
\begin{align*}
       \P(\eventforc)\E\left[\sum_{t=1}^T \ind(\event_{r,t}(\pi_t(\p{j}) = i) \cap \ell_t < \ell^{\star}) \Big| \eventforc\right] \leq
       jF\left(8n\gamma_{j,i} +9\gamma_{j,i} + 2T\delta\right)\,,
    \end{align*}
    where event $\eventforc$ is defined in Definition \ref{def:event}, $\delta =\frac{1}{n^3T}$, and $\gamma_{j,i}$ is defined in Equation \eqref{eqn_gamma_agnostic}. 
\end{lemma}

The proof of Lemma \ref{lem:parts2} differs substantially from our previous analysis and crucially uses the downward and upward cross-learning mechanisms included in the FORC algorithm. We dedicate Section \ref{subsec:lem:parts2} to highlighting some of the main ideas. 
But first, we finish the proof of Theorem \ref{thm_forc} by stating an upper bound on the loss when event ${\eventforc}$ does not hold.

\begin{lemma}[Upper Bounding Loss under Event ${\eventforc}^c$]
\label{lem:parts3}
Let event $\eventforc^c$ be the complement of event $\eventforc$,  specified in Definition \ref{def:event}. We have 
\begin{align*}
      \P({\eventforc}^c)\sum_{\p{j}=1}^n \sum_{i=\p{j}+1}^n   \Delta_{j,i} \E\left[\sum_{t=1}^T \ind(\event_{r,t}(\pi_t(\p{j}) = i)  \Big| {\eventforc}^c\right] \leq \frac{1}{n^2} \,.
    \end{align*}

\end{lemma}  
The proof of Lemma  \ref{lem:parts3} uses the exact same steps as the derivation in \eqref{eqn_regret_conditional}---which establishes an upper bound  for the loss under event $\mathcal{E}^c$ in the proof of Theorem \ref{thm_far}---and is therefore omitted. Recall that since event $\mathcal{G}$ holds with probability at least $1-\delta$, we have that $\P({\eventforc}^c) \leq \delta = \frac{1}{n^3T}$. Putting the upper bounds established in the above three lemmas back into Equation \eqref{eqn_regret_conditional_agnostic},  we complete the proof of Theorem \ref{thm_forc} as follows:
\begin{align*}\reg_{T} 
& \le  F+   \sum_{\p{j}=1}^n \sum_{i=\p{j}+1}^n \P(\eventforc)\Delta_{j, i} \E\left[\sum_{t=1}^T  \ind(\event_{r,t}(\pi_t(\p{j}) = i) \cap \ell_t < \ell^{\star})\bigg | \eventforc\right] \\
&+ \sum_{\p{j}=1}^n \sum_{i=\p{j}+1}^n \P(\eventforc)\Delta_{j, i} \E\left[\sum_{t=1}^T \ind(\event_{r,t}(\pi_t(\p{j}) = i) \cap \ell_t \geq \ell^{\star})\bigg | \eventforc\right]\\
&+ \sum_{\p{j}=1}^n \sum_{i=\p{j}+1}^n  \P({\eventforc}^c)\Delta_{j, i}  \E\left[\sum_{t=1}^T \ind(\event_{r,t}(\pi_t(\p{j}) = i)  \big| {\eventforc}^c\right] \\
&\le F + \sum_{\p{j}=1}^n \sum_{i=\p{j}+1}^n \left( \frac{64\log(4nT/\delta)}{\Delta_{j,i}} (\level +9jF + 8njF)+ 2\delta jFT\right) + \frac{1}{n^2}
\\
&\le F + \sum_{\p{j}=1}^n \sum_{i=\p{j}+1}^n \frac{64\log(4nT/\delta)}{\Delta_{j,i}} (\level +9jF + 8njF)+ \delta n^3 FT + \frac{1}{n^2}\,.
\end{align*} 

Substituting $\delta = \frac{1}{n^3T}$, we get the desired regret bound:
\begin{align*}\reg_{T} &\le 2F + 1 + \sum_{\p{j}=1}^n \sum_{i=\p{j}+1}^n  \frac{64\log(4nT/\delta)}{\Delta_{j,i}} (\level +9jF + 8njF)\\
&= O\left( \left(n^2F+ \log(T)\right) \sum_{\p{j}=1}^n \sum_{i=\p{j}+1}^n  \frac{\log(nT)}{\Delta_{j,i}} \right)\,,
\end{align*}
where the last equation holds because $\level =\log_2(T)$.
\end{proof}

\subsection{Proof Ideas of Lemma \ref{lem:parts2}}
\label{subsec:lem:parts2}

 In Lemma \ref{lem:parts2}, we aim to bound $\E\left[\sum_{t=1}^T \ind(\event_{r,t}(\pi_t(\p{j}) = i) \cap \ell_t < \ell^{\star}) \big| \eventforc\right]$, which is the expected  number of rounds that a learning level smaller than $ \ell^{\star}$ is sampled and an inferior product $i$ is misplaced in position $j < i$ for a real customer, who in turn, provides feedback on that product. Similar to Lemma \ref{lem_dagprop},  we wish to show that such misplacement can only happen if we do not have enough samples from product $i$. 
To do so, we relate the number of samples from product $i$ in levels smaller than $ \ell^{\star}$ to its counterpart in level $ \ell^{\star}$---this is made possible since level $ \ell^{\star} = \ceil*{\log_2(F)}$ is sampled independently with probability at least $\frac{1}{2F}$ and the probability of sampling any lower level cannot exceed one. Further, the number of times we receive feedback on product $i$ is easier to characterize for level $\ell^*$ since the product-ordering graph corresponding to this level does not contain any incorrect edges. 

To that end, we consider an arbitrary instantiation of our algorithm and define
  two \emph{milestones} for   product $i$ in level $ \ell^{\star}$. We show that in any round $t$ that occurs after product $i$ reaches these  two milestones, the event $\event_{r,t}(\pi_t(\p{j}) = i)$ does not happen for any level $\ell_t < \ell^{\star}$.

\begin{definition}[First Milestone $\tj$]
Define $\tj$  as the  smallest round $t$  {in} which $\hat{\eta}^{(\ell^{\star})}_{i,t} \geq 4\gamma_{\p{j},i}$, where {$\hat{\eta}^{(\ell^{\star})}_{i,t}$ and }$\gamma_{j,i}$ are respectively defined in Equations  \eqref{eq:eta_cr} and \eqref{eqn_gamma_agnostic}. 
\label{defn_firstmilestone}
\end{definition}
\medskip

In words, $\tj$ denotes the first time that the number of cross-learned samples  for product $i$ in level $\ell^{\star}$ crosses $4\gamma_{\p{j},i}$. Thus, by definition, in round  $\tj$, $\hat{\eta}^{(\ell^{\star})}_{i,t}$ is upper bounded by $4\gamma_{\p{j},i}+1$, since it cannot increase by more than one in any round. Upward cross-learning (see definition of $\hat{\eta}^{(\ell^{\star})}_{i,t}$ in Step~\ref{step_upwardcl_forc} of Algorithm \ref{alg_cascading_agnostic})  then enables us to use this upper bound on $\hat{\eta}^{(\ell^{\star})}_{i,t}$ to establish an upper bound on the number of samples for product $i$ in lower levels.\footnote{{We formally establish the upper bound in Lemma \ref{lem:tjmax}.}} To see why, note that for any $t\in[T]$, we have 
\begin{align}
\sum_{\ell=1}^{\ell^{\star}-1}\eta^{(\ell)}_{i,t} \leq 2^{\ell^{\star}} \hat{\eta}^{(\ell^{\star})}_{i,t} \leq 2F \hat{\eta}^{(\ell^{\star})}_{i,t}\,,    \label{eqn_lem8_upwardcl}
\end{align}
where the first inequality follows from the definition of the cross-learning variables and the second one follows from the observation that $2^{\ell^{\star}} \leq 2F$. Furthermore, having a sufficient number of samples for product $i$ at level $\ell^{\star}$  (since $\hat{\eta}^{(\ell^{\star})}_{i,\tj} \geq 4\gamma_{j,i}$) is a prerequisite for adding (correct) edges to the product-ordering graph $G^{(\ell^{\star})}$ containing this product---this follows from arguments similar to those used in the proof of Lemma~\ref{lem:parts1}

Next, we proceed to introduce a second milestone whose definition relies on the construction of the product-ordering graph in level $\ell^{\star}$. In particular, for all $j \geq 1$, define $S_j \triangleq \{i_1, i_2, \ldots, i_{j}\}$ as the set of the first $j$ products (chronologically) that product $i$ {forms}  outgoing edges to in the product-ranking graph $G^{(\ell^{\star})}$. For consistency,  we also define $S_0 = \emptyset$. Next, for a given $i \in [n]$, define $\G{\Delta} = \{ k ~|~ \Delta_{k,i} \geq \frac{\Delta}{2}\}$---i.e., {this set includes any product} better than $i$ ($k < i$) whose  reward{-gap} to $i$ is at least half as much as the input parameter $\Delta$. We are now ready to define the second milestone. 
\medskip

\begin{definition}[Second Milestone $\tdelta$]
Define $\tdelta$ as the smallest round $t$ in which the following inequality holds:
\begin{align} \label{eq:condition:delta}
    \max_{k \in \G{\Delta_{j,i}} \setminus S_{j-1}}\left\{\eta^{(\ell^{\star})}_{k,t}\right\} \geq 4\gamma_{\p{j},i},
\end{align}
where $\gamma_{j,i}$ is  defined in Equation \eqref{eqn_gamma_agnostic}.\label{defn_secondmilestone}
\end{definition}

Just as $\tj$ represents the first time that we see $4\gamma_{\p{j},i}$ samples on product $i$, $\tdelta$ represents the first time that we see $4\gamma_{\p{j},i}$ samples on a product ``comparable'' to $j$; i.e., a product whose gap to product $i$ is at least half of the corresponding gap between product $i$ and product $j$. Our insistence on $4\gamma_{j,i}$ samples (as opposed to just $\gamma_{j,i}$) stems from this halving of the reward gap since $\gamma_{k,i}$ for any $k \in [n]$ is inversely proportional to the square of $\Delta_{k,i}$ as can be gleaned from Equation~\eqref{eqn_gamma_agnostic}.

Equipped with these two milestones, we finally define $\tjmax$ as the maximum of $\tj$ and $\tdelta$; that is, 
\begin{align}
    \label{eq:tjmax:def}
    \tjmax \triangleq \max\{\tj, \tdelta\}\,.
\end{align}
Finally, as we will show later, product $i$ has outgoing edges to at least $j-1$ products by round $\tjmaxminus{}$. Thus, the set $S_{j-1}$ is well-defined by round $\tjmaxminus{}$.
 
 Then, since $\G{\Delta_{j,i}}$ contains at least $j$ products ($[j] \subseteq \G{\Delta_{j,i}}$) and $|S_{j-1}| = j-1$, the set $\G{\Delta_{j,i}} \setminus S_{j-1}$ is non-empty, and hence $\tdelta$ is well defined; see Equation \eqref{eq:condition:delta}. 
 
Our high level argument is as follows: by round $\tjmax= \max\{\tj, \tdelta\}$,  i.e., after the two milestones, we have  ``enough'' crossed-learned samples from product $i$ in level $\ell^{\star}$ and another product that is better than $i$ but not already in $S_{j-1}$. Let us call such a product $\sigma$. Using the properties of  the product-ordering graph  $G^{(\ell^{\star})}$ (similar to Lemma~\ref{lem_pairwise_edgeadd}),  we can show  that  there is an edge  from $i$ to $\sigma$ at the end of round $\tjmax$. Downward cross-learning (Step~\ref{step_addedge_forc} of Algorithm \ref{alg_cascading_agnostic}) ensures that the same edge exists in all product-ordering graphs in lower levels at the end of round $\tjmax$. Consequently, for any level $\ell < \ell^{\star}$, product $i$ has (correct) outgoing edges to at least $j$ products (i.e., $S_{j-1} \cup \{\sigma\}$), all of which will be ranked above $i$ since our ranking decision follows the \Call{GraphRankSelect}{$\boldsymbol{\eta}^{(\ell)}_t, G$} function. This property implies  that after our two milestones, the event $\event_{r,t}(\pi_t(\p{j}) = i)$ does not occur.

\begin{lemma}[Properties of $\tjmax$ {under Event $\eventforc$}] \label{lem:tjmax}
Define $\eventz_{i,t}$ as the event that (a) the customer {in} round $t$ is real, and (b) we receive feedback on product $i$ in round $t$.  Then, 
for any $j\in \{1, \ldots, i-1\}$, we have
\begin{align*} \E\Big[\sum_{t=1}^{\tjmax}\ind(\event_{r,t}(\pi_t(\p{j}) = i) \cap \ell_t < \ell^{\star})\big| \eventforc\Big]&\le \E\Big[\sum_{t=1}^{\tjmax}\ind(\eventz_{i,t} \cap \ell_t < \ell^{\star})~|~\eventforc\Big] \\
&\leq jF\left(8n\gamma_{j,i} +9\gamma_{j,i} + 2T\frac{\delta}{\P(\eventforc)}\right)\,,\end{align*} 
where $\tjmax$ and $\gamma_{j,i}$ are respectively  defined in Equations  \eqref{eq:tjmax:def} and  \eqref{eqn_gamma_agnostic}, and $\delta = \frac{1}{n^3T}$. Furthermore,  for any instantiation,  conditional on $\eventforc$, at the end of round $\tjmax$, product $i$ has at least $j$ outgoing edges in graph $G^{(\ell^{\star})}$. 
\end{lemma}

The proof of Lemma \ref{lem:tjmax} is presented in Appendix \ref{sec:proof:tjmax}. The first inequality follows from the fact that  
 $\mathcal{A}_{r,t}(\pi_t(j) = i) \leq \eventz_{i,t}$ since the latter captures a superset of events. The second inequality is more challenging to show. One of the main challenges in proving this inequality involves upper-bounding the number of samples we receive for product $i$ (which helps us upper bound $\sum \ind(\eventz_{i,t}$)) when the learning level is smaller than $\ell^{\star}$. The difficulty here stems from the fact that, due to the frequent exposure of lower levels to  fake users, their corresponding graphs may have incorrect edges. Consequently, in those levels, product $i$ may be placed in a disproportionately high rank which, in turn, results in receiving `too many' samples for this product. To control for the occurrence of this scenario, we have to show that every time $\eventz_{i,t}$ is true for a level $\ell < \ell^{\star}$, some progress is made at level $\ell^{\star}$ towards adding a correct outgoing edge from $i$---we know that such an edge will be transferred to the lower levels via downward cross-learning. Our main technique here is a mapping that connects receiving feedback on product $i$ at a lower level to receiving feedback on a product $\sigma \notin S_{j-1}$ at level $\ell^{\star}$; see Lemma \ref{lem_agnostic_conditional} in the Appendix. Secondly, we also leverage  the upward cross-learning argument from Equation~\eqref{eqn_lem8_upwardcl} to show that feedback on product $i$ at a lower level contributes to adding a correct outgoing edge containing this product in the graph $G^{(\ell^{\star})}$.

Finally, we show how we use Lemma \ref{lem:tjmax} to complete the proof of Lemma \ref{lem:parts2}.
{Fixing $j < i$, we divide the rounds into two groups based on $\tjmax$ as follows} 
\begin{align} \nonumber
    \E\Big[\sum_{t=1}^T \ind(\event_{r,t}(\pi_t(\p{j}) = i) \cap \ell_t < \ell^{\star}) \big| \eventforc\Big]& =  \E\Big[\sum_{t=1}^{\tjmax}\ind(\event_{r,t}(\pi_t(\p{j}) = i) \cap \ell_t < \ell^{\star})\big| \eventforc\Big] \\
    &+  \E\Big[\sum_{t=\tjmax+1}^T\ind(\event_{r,t}(\pi_t(\p{j}) = i) \cap \ell_t < \ell^{\star})\big| \eventforc\Big]\,, \notag 
\end{align}
where by Lemma \ref{lem:tjmax},  the first term, i.e., $\E\Big[\sum_{t=1}^{\tjmax}\ind(\event_{r,t}(\pi_t(\p{j}) = i) \cap \ell_t < \ell^{\star})\big| \eventforc\Big]$,  is upper bounded by  $jF\left(8n\gamma_{j,i} +9\gamma_{j,i} + 2T\frac{\delta}{\P(\eventforc)}\right)$. The second term, i.e., $\E\Big[\sum_{t=\tjmax+1}^T\ind(\event_{r,t}(\pi_t(\p{j}) = i) \cap \ell_t < \ell^{\star})\big| \eventforc\Big]$, is zero because by Lemma \ref{lem:tjmax},
 after round $\tjmax$, product $i$ has at least $j$ outgoing edges in $G^{(\ell^{\star})}$. Therefore, product $i$ has at least $j$ outgoing edges in $G^{(\ell)}$ for all $\ell \leq \ell^{\star}$ too. As a result, the algorithm will never place this product at position $j$ when $\ell_t < \ell^{\star}$. Note that this claim is valid even if the graph $G^{(\ell_t)}$ itself is eliminated. This is because in any round $t$, (a) if graph $G^{(\ell_t)}$ is eliminated, 
 our algorithm would select a ranking $\pi_t$ based on a graph $G^{(\ell)}$, where $\ell$ is the smallest index for which the corresponding graph is not eliminated, and (b) conditioned on event  {$\eventforc$}, graph $G^{(\ell^{\star})}$ is not eliminated because it does not have any erroneous  edges; see the proof of Lemma \ref{lem:parts1}.  \hfill \Halmos
 
\section{Numerical Studies}\label{sec:numerics} 
In the previous sections, we proposed two  algorithms for learning how to rank  products in the presence of  fake users: FAR, which incorporates the fakeness budget $F$ into its design, and FORC, which learns the optimal ranking even without knowing $F$. In this section, we evaluate the performance of these two algorithms in a synthetically generated setting. Simulating our algorithms on this synthetic data serves the following purposes: (a) it provides insights into the performance of our two algorithms on a more average, real-world inspired setting, as opposed to a worst-case bound, (b) it allows us to test these algorithms  against a benchmark algorithm, namely the popular UCB algorithm\footnote{In particular, we use an adaptation of the UCB algorithm that has been proposed for the ranking problem by \cite{kveton2015cascading}.}, and (c) it provides an opportunity to assess the value of knowing the fakeness budget $F$ in advance. Additionally, we use the setup presented below to further illustrate the inner workings of the FORC algorithm and how these carefully designed steps lead to a low regret.

\subsection*{Setup and Parameter Choices}
We begin by describing our model primitives for the underlying product ranking problem, as well as the customer behavior, and fake users. We consider a cascading bandits scenario with $n=10$, where all customers only examine the products in the first $k= 4$ positions  (i.e., $q_j = 0$ for all $j \in [n-1] \setminus \{k\}$ and $q_k = 1$), and exit if none of these elicit a click. We believe that such a cascade model reasonably approximates customers who pay attention to the products in top ranks but rarely pursue later products. Second, by adopting the aforementioned strict  cascade model, we can compare and contrast our algorithms  against a benchmark that was developed particularly for this setting (albeit without fake users), namely the Cascade UCB algorithm from~\cite{kveton2015cascading}.

Conditional upon viewing a product $i \in [n]$, a real customer clicks on it with probability $\mu_i$. These probabilities are chosen uniformly at random from the range $[0.02,0.3]$ with the added constraint that for any two products $i,i'$, the click probabilities satisfy $|\mu_i - \mu_{i'}| \geq 0.02$---this avoids degenerate instances where the products are too similar and thus, their exact order is mostly irrelevant from customers' perspectives. As in Section~\ref{sec:model}, we label the products in the decreasing order of their click probabilities.

We assume that the fake user is particularly interested in boosting the click probabilities of two sub-optimal products, in this case products six and seven. In the absence of the fake user, these products would not be placed in the top four positions  due to their low popularity, and thus they lose out on many potential customers due to position bias. The fake user adopts a two-pronged strategy similar to what we outlined in Theorem~\ref{thm_ucb_bad}. In the first prong,  the fake user tries to lower the estimated click probability for competitors (e.g., products three and four) by avoiding clicks entirely
and in the second prong, the fake user only clicks on the corresponding product that they seek to boost (product six or seven). We note that the fake user never examines products beyond position $k=4$, so that the platform cannot distinguish them from real customers.

We consider a time horizon of $T=2 \times 10^6$ users with the total fakeness budget given by $F = 14 \sqrt{T} \approx 20,000$.  Fake users are predominantly concentrated at the early stages of the algorithm. Specifically, until 
reaching the limit on the number of fake users, $F$, each initial user is fake with probability $\frac{3}{4}$ and real otherwise. This modeling choice captures situations where the retailer faces a new set of products and is unsure of their relative popularity, and some of these product manufacturers hire fake users at the early stages to boost their visibility.

Given this setup, we evaluate both the FAR and FORC algorithms along with the Cascade UCB algorithm from~\cite{kveton2015cascading}. Our choice of UCB as a benchmark is well-motivated given the popularity of the algorithm and its variants in both theory and  practice, and its optimal performance across a wide range of environments (e.g., see~\cite{auer2000using, garivier2011kl, balseiro2019contextual}). For all three algorithms, we set the parameter $\delta = 0.02$ and the window size to be $\sqrt{\frac{\log(2nT/\delta)}{\eta}} + \frac{\tilde{F}}{\eta}$, where $\eta$ is the number of times we receive feedback on a given arm (product). Further (a) $\tilde{F} = 0$ for Cascade UCB since it does not account for fake users, (b) $\tilde{F} = \frac{F}{2}$ for FAR, and (c) $\tilde{F} = \frac{1}{2}\log(\frac{2L}{\delta})$ for each level in FORC where $L = \log_2(T)$.\footnote{We set the first term in the window size to be identical for all three algorithms for the sake of consistency. Further, as is common practice in the literature (e.g., see~\cite{auer2000using}), we set the confidence intervals to be slightly smaller than what is required for our theory to ensure faster convergence. Despite this, we note that both our algorithms converge to the optimal ranking in all of the instances we simulate.} 

\subsection*{Results and Discussion}
Figure~\ref{fig:regret_main} presents the regret of each of the three implemented algorithms in comparison to the optimal ranking $\pi^{\star} = \{1,2,3,4\}$ for $100$ simulations with the parameter choices mentioned above.\footnote{Note that the ranking of products beyond the fourth position does not matter, and hence, in the optimal ranking,  $\pi^{\star} = \{1,2,3,4\}$, we assign only the top four products.} Our simulations lead to two core findings. First, Cascade UCB incurs linear regret for the type of instances we study while both of our methods have regret that is sublinear in the time horizon $T$. In particular, we observe that Cascade UCB quickly converges to a sub-optimal ranking that places products $\{6,7\}$ at positions $j \leq 4$; upon convergence it continuously incurs regret until the end of the time horizon as this ranking appeals to fewer customers than $\pi^{\star}$. In contrast, our algorithms  are more conservative in the early rounds but eventually converge to $\pi^{\star}$. This can be verified by observing in Figure~\ref{fig:regret_main} that in the initial rounds, Cascade UCB outperforms both FAR and FORC as the former rapidly converges to a ranking (albeit a sub-optimal one) while the latter algorithms are still learning the correct product-ordering graph. 
This numerical result shows that the poor performance of Cascade UCB is not limited to a worst-case instance (as illustrated in Theorem~\ref{thm_ucb_bad}) and can indeed occur for more realistic, randomly generated instances.

\begin{figure}
    \centering
    \includegraphics[width=0.85\textwidth]{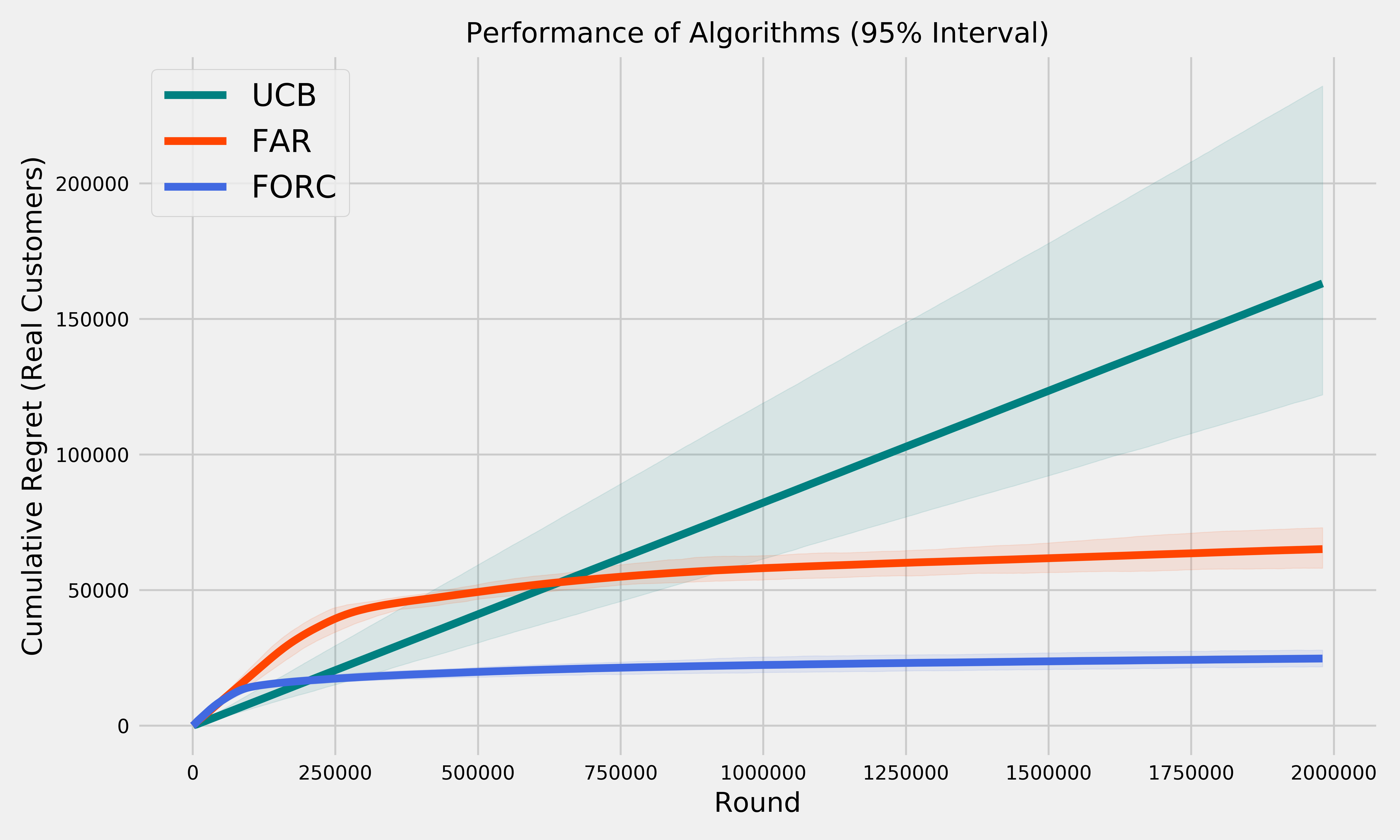}
    \caption{Regret incurred by Cascade UCB (UCB), FAR, and FORC in comparison to the optimal ranking $\pi^{\star}$. We only calculate the reward (and therefore, regret) stemming from clicks by real customers. The shaded regions mark the $95\%$ performance interval---i.e., in $95\%$ of our simulations, the regret of each algorithm lies within the corresponding shaded region. }
    \label{fig:regret_main}
\end{figure}

\begin{figure}[hbtp]
\centering
\includegraphics[scale=0.8]{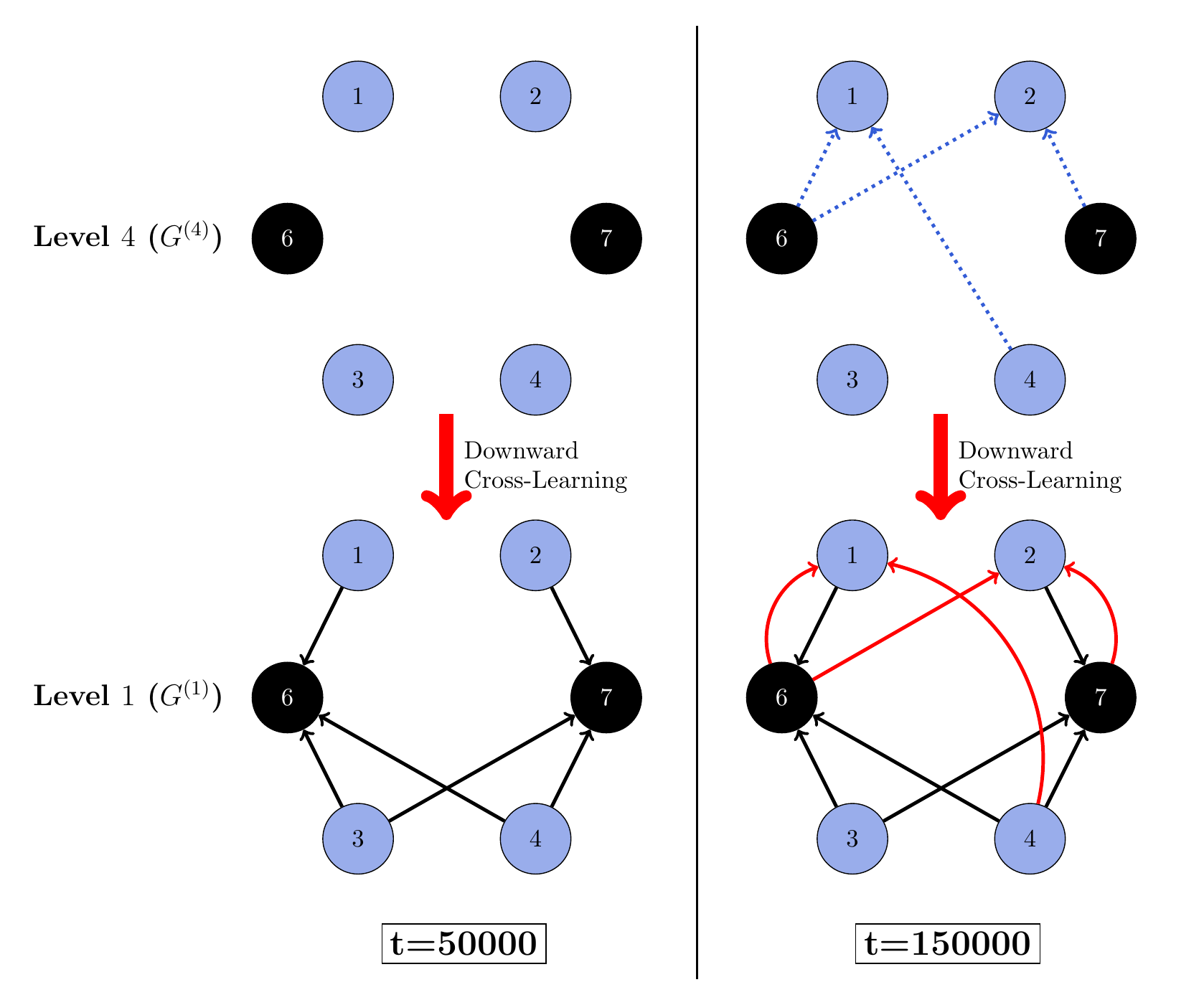}
\caption{The state of the product-ordering graphs at levels one ($G^{(1)}$) and four ($G^{(4)}$) in two different rounds ($50,000$ and $150,000$) during a single run of the FORC algorithm. The nodes in black represent the products that fake users promote (products six and seven). The black edges denote the incorrect edges added due to fake users' actions. The red edges are those that are transmitted from level four to one by downward cross-learning; the correct edges in level four are displayed in blue dotted lines. Some edges in these graphs have been intentionally omitted for the sake of clarity. At $t=150,000$, level one has been eliminated as $G^{(1)}$ contains at least one cycle (e.g., $1 \rightarrow 6 \rightarrow 1$). \label{fig_forc_visualization}}
\end{figure}

Second, and surprisingly, the regret achieved by our FORC algorithm that is completely oblivious to the fakeness budget $F$ is better than the regret of the FAR algorithm that uses knowledge of $F$. On one hand, one could take this finding as indicative of the divergence between a worst-case bound (in Theorem~\ref{thm_forc}) and the regret incurred on a more realistic instance. Yet, we believe this highlights the salience of multi-level learning and cross-learning in achieving superior performance. We also find this numerical observation practically appealing as knowing $F$ in advance may not be feasible for many ranking platforms. 
Next, we provide intuition on why FORC outperforms FAR for this class of instances.

Let us first focus our attention on FORC. Figure~\ref{fig_forc_visualization} provides a representative snapshot of the product-ordering graphs for levels one and four at two different points in time ($t=50,000$ and $t=150,000$). 
We note that by round $50,000$, most likely, all the fake users have already arrived. Since level one is selected with probability at least one-half, it witnesses a majority of the fake users, which in turn results in the addition of  incorrect edges. For instance, at $t=50,000$, based on the product-ordering graph of level one, our algorithm would always place products six and seven in the top positions when level one is selected as the learning level. Given that this does not coincide with the optimal ranking $\pi^{\star}$, the FORC algorithm incurs high regret early in the process. Level four however, has not collected enough samples up to round $50,000$ to make edge determinations (as evident in Figure~\ref{fig_forc_visualization}).

On the other hand, by round $150,000$, level four---having witnessed fewer fake users---has already correctly determined the superiority of product one over six. Due to downward cross-learning, this edge is then transferred to level one, leading to a cycle in $G^{(1)}$ which in turn results in the elimination of this product-ordering graph (and potentially those of levels two and three).  Following the elimination of the lower levels' graphs, FORC quickly converges to the optimal ranking as no fake users arrive after this point. Finally, we also highlight the importance of upward cross-learning here: given that level four is  chosen with low probability (i.e., $\frac{1}{16}$), it is unlikely that it would receive sufficient feedback on products one and six 
by round $150,000$. However, after the arrival of all fake users (which occurs before round $50,000$), level one receives a large number of samples corresponding to these products, all of which are from real users. Upward cross-learning enables FORC to successfully use these samples to derive a tighter estimate for the given products at level four, which 
expedites the formation of the correct edge between products one and six, eventually leading to the elimination of the product-ordering graph at level one.

In contrast, the FAR algorithm does not utilize multi-level learning and takes much longer to add the correct edges to its product-ordering graph, and consequently converge to the optimal ranking. Specifically, the actions of the fake users lead the algorithm to initially overestimate the popularity of products six and seven, and underestimate that of other products; the algorithm then requires a large number of samples from real customers to recover from this manipulation and identify, for example,  that product one dominates product six.  Arguably, FAR's slower convergence compared to FORC could be partially  attributed to its large window size. However, due to this larger window size, FAR never includes incorrect edges in its product-ordering graph, unlike Cascade UCB. Further, additional numerical analyses (omitted for the sake of brevity) highlight the importance of this conservative approach as lowering the window size increases the possibility that FAR converges to a sub-optimal ranking.

We complete this section by discussing the generalizability of these results. Although our observations pertain to a specific simulation setup, we believe that these findings are broadly applicable to a wide range of scenarios. First, the restricted  Cascade model (customers exit at position $k$) is not essential for our results as the same phenomena can be recreated even when the exit probabilities are non-zero across all $n$ positions. That is, once the fake user suppresses the empirical rewards of certain products, UCB and other traditional learning algorithms would place these at lower positions, and the ensuing position bias would lead to fewer clicks, and slow recovery. Second, in practice, one would expect the click probabilities to be smaller and closely clustered. Such a scenario would actually favor fake users as the smaller reward gap (i.e., $\Delta_{j,i}$) implies that fewer fake clicks are needed to boost the position of an inferior product. At the same time, there is a trade-off involved as the resulting ranking would not be too sub-optimal (due to products' rewards being close to each other), and hence a longer time horizon is required to achieve the same regret levels as in Figure~\ref{fig:regret_main}. Finally, altering the arrival and behavior of fake users may affect our results, e.g., all three algorithms would perform better if the fake users always click on products and arrive later in time. In that sense, our results could be interpreted as arising from the adversarial actions of fake users---this provides valuable insights as in reality, one would expect the fake users to be adaptive and employ whatever strategy provides them with the maximum benefit. 

\section{Concluding Remarks} \label{sec:conclude}
As the volume and granularity of available data increases
at an unprecedented rate, many platforms rely on data-driven algorithms to optimize their operational decisions. Moving to a data-centric environment, however, can put online platforms in a vulnerable situation when the generated data is prone to manipulation. Indeed, as we show in this work, popular learning strategies such as those based on upper confidence bound ideas are not particularly robust to fake users, whose actions may mislead the algorithm to make poor decisions. In the context of product ranking specifically, such sub-optimal decisions may result in the most visible positions being occupied by unpopular products, which in turn, can significantly hurt customer engagement and other metrics of interest. Motivated by these challenges, we develop new  algorithms that are robust to the actions of fake users and converge to the optimal product ranking, even when we are completely oblivious to the identity and number of fake users. While many recent works (e.g., see~\cite{ursu2016power}) have explored how position bias can prevent platforms from accurately inferring customer preferences, ours is among the first to highlight how sellers can exploit this in their favor by employing fake users, and develop constructive solutions for preventing such a situation.

At a high level, our work presents a number of insights on how to design methods for uncertain environments to guarantee robustness in the face of manipulation. These include: (a) being more conservative in inferring key parameters and changing decisions based on limited data; (b) employing parallelization and randomization to limit the damage caused by fake users; and (c) augmenting a conservative approach via cross-learning.  
We believe that the ideas proposed in this work can serve as a starting point for designing robust data-driven algorithms to tackle other operational challenges. 

Our work opens up a number of avenues for future investigation. A natural direction is to pursue the design of algorithms resilient to manipulation for alternative customer choice models and objective functions. While the exact details of the algorithm may be context specific, as stated earlier, the general insights provided in this work could be valuable in other settings as well. More generally, modeling other channels through which fake users can manipulate a platform's ranking algorithm (e.g., fake orders, fake reviews) and developing learning algorithms that ensure robustness to multiple sources of manipulation is a key challenge for many platforms. In this regard, developing prescriptive solutions for product ranking in presence of fake users, as we do in this paper, could complement the recent body of work that has analyzed this problem from a more descriptive or passive angle~\citep{luca2016fake,ivanova2017can,jin2019brush}.

\bibliographystyle{informs2014} 

\bibliography{sample-bibliography}

\medskip

\newpage

\begin{APPENDICES}

\section{Model: Alternative Notions of Regret}
\label{app:regret}
In our notion of regret, we only consider  real customers. Recall that  for any fixed policy ${\bm P}= (P_t)_{t \in [T]}$ followed by the fake users, we define
$\reg_{T}(\bm{P}) ~=~  \E_{\mathcal{H}_T(\pi^{\star})}\Big[\sum_{t=1}^T \mathcal{C}_{r,t}(\pi^\star)\Big]- \E_{\mathcal{H}_T(\pi_t)}\Big[\sum_{t=1}^T \mathcal{C}_{r,t}(\pi_t) \Big]$. In this expression, the realization of the histories depend on the randomness in the clicks and exit positions of the customers as well as the fake users (if $\bm{P}$ is not deterministic), and any randomness in the algorithm itself. 
 The relationship between our definition of regret and one where fake users are included is formalized below:
\begin{align}
\reg_{T}(\bm{P}) &= \E_{\mathcal{H}_T(\pi^{\star})}\Big[\sum_{t=1}^T \mathcal{C}_{r,t}(\pi^\star)\Big]- \E_{\mathcal{H}_T(\pi_t)}\Big[\sum_{t=1}^T \mathcal{C}_{r,t}(\pi_t) \Big] \notag\\
& \geq \E_{\mathcal{H}_T(\pi^{\star})}\Big[\sum_{t=1}^T \left(\mathcal{C}_{r,t}(\pi^\star) + \mathcal{C}_{f,t}(\pi^\star) \right)\Big] - \E_{\mathcal{H}_T(\pi_t)}\Big[\sum_{t=1}^T \left(\mathcal{C}_{r,t}(\pi_t) + \mathcal{C}_{f,t}(\pi_t)\right)\Big]\notag \\
&- \E_{\mathcal{H}_T(\pi^{\star})}\Big[\sum_{t=1}^T\ind(\text{user $t$ is fake})\Big] \notag\\
& \geq \E_{\mathcal{H}_T(\pi^{\star})}\Big[\sum_{t=1}^T \left(\mathcal{C}_{r,t}(\pi^\star) + \mathcal{C}_{f,t}(\pi^\star) \right)\Big] - \E_{\mathcal{H}_T(\pi_t)}\Big[\sum_{t=1}^T \left(\mathcal{C}_{r,t}(\pi_t) + \mathcal{C}_{f,t}(\pi_t)\right)\Big]\notag -F \notag 
\end{align}

Recall from our definition of $\mathcal{C}_{f,t}(\pi)$ that $\E_{\mathcal{H}_T(\pi^{\star})}\Big[\sum_{t=1}^T\ind(\text{user $t$ is fake})\Big] = \E_{\mathcal{H}_T(\pi^{\star})}\Big[\sum_{t=1}^T \mathcal{C}_{f,t}(\pi^\star) \Big]$---this gives rise to the second inequality above. Based on the final line, we can conclude that the difference between our notion of regret $\reg_{T}(\bm{P})$ and one where fake users are included, i.e., $\E_{\mathcal{H}_T(\pi^{\star})}\Big[\sum_{t=1}^T \left(\mathcal{C}_{r,t}(\pi^\star) + \mathcal{C}_{f,t}(\pi^\star) \right)\Big] - \E_{\mathcal{H}_T(\pi_t)}\Big[\sum_{t=1}^T \left(\mathcal{C}_{r,t}(\pi_t) + \mathcal{C}_{f,t}(\pi_t)\right)\Big]$, is at most the fakeness budget $F$.  

\section{Proof of Theorem~\ref{thm_ucb_bad}
}
We now formally prove that for the two-product instance described in Section~\ref{sec:nonRob}, UCB incurs linear regret. Suppose that the click probabilities for the products are given by $\mu_1 = 1$ and $\mu_2 = 1/2$, and the exit probability is $q_{1} = 1$---i.e. the customer only ever observes the product in the first position and never moves on to the product in the second position. Since we only have two products and customers exit after the first position, computing a product ranking is equivalent to simply selecting a single product for the top position, and receiving feedback only on this product. The total fakeness budget for this instance is $F=4\log^2(T)$ and the fake users' actions are as described in Figure~\ref{fig:UCB}. In particular, the fake users who arrive in the first $2\log^2 (T)$ rounds of the time horizon do not click on any product and exit after the first position (so as to appear similar to a real customer). The second wave of fake users who arrive in the subsequent $2\log^2(T)$ rounds always click on product two if it is in the first position; if not, the fake user will not click on any product and exit after the first position. We assume that the time horizon $T$ is sufficiently large, and expand on this later.

Our main claim is that after such a sequence of fake users, with high probability, the UCB algorithm will present the sub-optimal ranking $\pi$ given by $\pi({1}) = 2$ and $\pi({2}) = 1$ (henceforth referred to as $(2, 1)$) for the remaining $T-4\log^2(T)$ rounds, in which only real customers arrive. On the other hand, the optimal ranking $\pi^{\star}$ is given by $\pi^{\star}({1}) = 1$ and $\pi^{\star}({2}) = 2$ (henceforth referred to as $(1, 2)$). Note that the optimal ranking achieves a reward of one per-round since all real customers click on the first product, and $\pi$, in expectation, leads to a click with probability $\frac{1}{2}$. Therefore, the regret achieved by selecting ranking $\pi$ in the final $T-4\log^2(T)$ rounds is $\Omega(T)$ (recall that regret is only computed for real customers).

To prove this claim, we will first characterize the state of the UCB algorithm after the first and second waves of fake users respectively. First, let us recall that in any given round, the UCB algorithm tracks the empirical mean rewards $\hat \mu_i$ and feedback counts $\eta_i$ for $i \in \{1,2\}$ and selects the product with the highest upper confidence bound given by Equation~\eqref{eqn_ucb_ucb} for the first position.
\begin{equation}
\text{UCB}_i = \hat{\mu_i} + \sqrt{\frac{\log T}{\eta_i}},\quad i\in \{1,2\}.
\label{eqn_ucb_ucb}
\end{equation}
Based on this, we argue that: 
\begin{enumerate}
    \item After the first $2\log^2(T)$ rounds containing only fake users, UCB has received feedback on both products exactly $\log^2(T)$ times. In particular, during any round $t \in [1,2\log^2(T)]$, the empirical average reward of both products will be zero since the fake users within this period never click on any product. Therefore, the algorithm always selects the product with a smaller feedback count in the top position as it maximizes the upper confidence bound given in Equation~\eqref{eqn_ucb_ucb}; ties are broken arbitrarily. 
    
    \item During the subsequent $2\log^2(T)$ rounds, we claim that the ranking $(2, 1)$ is played at least half of the time---i.e., product two is selected by the UCB algorithm for at least $ \log^2(T)$ rounds when $t \in [2\log^2(T)+1,4\log^2(T)]$. To see why, note that the empirical average reward of product two within this period is at least as high as that of product one, since the fake users never click on the latter. In fact, the empirical click probability for product one remains at zero. Then, it is not hard to see that UCB would place product two in the top position at least whenever the feedback count on this product is smaller than that of the other product (and arguably in other rounds as well). The feedback count of product one cannot be larger than that of product two for more than half the rounds during this phase given that both products started out with $\eta_1 = \eta_2 = \log^2(T)$ at $t=2\log^2(T)$.  It follows that UCB will select the ranking $(2, 1)$ at least $\log^2(T)$ times.
\end{enumerate}

We now summarize the state of the algorithm---empirical rewards and feedback counts---at round $t=4\log^2(T)$ after the departure of the last fake user. As per our earlier arguments, we have that:
\begin{align}
    \text{Product One}: ~~& \hat{\mu}_1 = 0, ~~\eta_1 \geq \log^2(T), \\
    \text{Product Two}: ~~& \hat{\mu}_2 \geq \frac{1}{2}, ~~\eta_2 \geq 2\log^2(T).\label{eqn_thm1_stateof2}
\end{align}

Observe that $\hat\mu_2 \geq \frac{1}{2}$ is due to the fact that exactly $\log^2(T)$ users examine but do not click on product two (first wave), whereas at least $\log^2(T)$ users click on this product (second wave). Plugging this into Equation~\eqref{eqn_ucb_ucb}, we get that at the end of $t=4\log^2(T)$ rounds: $\text{UCB}_{1} \leq 1/\sqrt{\log T}$, and $\text{UCB}_{2} \geq 1/2$.

In order to complete this proof, we finally show that, with high probability, the UCB algorithm will never play the ranking $(1, 2)$ for the remaining $T-4\log^2(T)$ rounds where real customers arrive, and consequently, never receives feedback on product one. First, note that our claim stating `the UCB algorithm will never select product one for the top position' during $t \in [4\log^2(T)+1, T]$ is equivalent to saying $\text{UCB}_{1} < \text{UCB}_{2}$; moreover, $\text{UCB}_{1} < \text{UCB}_{2}$ is certainly true at the end of $t=4\log^2(T)$. Using Chernoff bounds and the fact that $\mu_2 = \frac{1}{2}$, we will now prove that with high probability $\hat{\mu}_2 \geq \frac{1}{\sqrt{\log T}}$ for the remaining rounds. Since $\text{UCB}_2 > \hat{\mu_2}$, this in turn implies that product one will never be placed in the top position leading to linear regret. 

Let $\alpha_t$ be the probability that in round $4\log^2(T) + t$, the ranking $(1, 2)$ is chosen or equivalently product one is placed in the top position. From our upper bound above, we know that $\P(\alpha_t) < \P(\hat{\mu}_2 < \text{UCB}_1)$. We will show that for sufficiently large $T$ ($T \geq e^
{16}$), $\P(\alpha_t) \leq 1/T^2$ for any given $t \in [1,T-4\log^2(T)]$. Therefore, by the union bound, the probability that the ranking $(1,2)$ is ever chosen in the final $T-4\log^2(T)$ rounds is at most $\frac{1}{T}$. In terms of regret, this implies that with probability $(1-\frac{1}{T})$, UCB incurs an expected regret of $\frac{1}{2}(T-4\log^2(T))$, and therefore:
$$\reg_T =  \frac{1}{2}(1-\frac{1}{T})(T-4\log^2(T)) = \Omega(T)$$

To bound the probability of $\alpha_t$, note first that if $t \leq 2\log^2(T)$, then it is impossible for the empirical click probability of product two---i.e., $\hat{\mu}_2$---to drop below $\frac{1}{\sqrt{\log(T)}}$ since the empirical click probability of product two after $4\log^2(T)$ rounds is at least $1/2$ according to~\eqref{eqn_thm1_stateof2}; even if all $t$ samples resulted in no click for this product, the empirical probability cannot drop below $\frac{1}{4}$ since the algorithm has recorded at least $\log^2(T)$ clicks for this product (from fake users).

More generally, the above argument implies that in order for $\hat{\mu}_2 < \frac{1}{\sqrt{\log(T)}}$ to be true at some point during the algorithm, we require at least $2\log^2(T)$ samples after round $4\log^2(T)$ during which the customer examines but does not click on product two. With this in mind, fix an  arbitrary $t \geq 2\log^2(T)$: it is not hard to see that if $\hat\mu_2 < \frac{1}{\sqrt{\log(T)}}$ after round $t$, then the average reward corresponding to product two from the samples obtained during the $t$ rounds $[4\log^2(T)+1, 4\log^2(T)+t]$ cannot be larger than $\frac{1}{\sqrt{\log(T)}}$---once again, this follows from~\eqref{eqn_thm1_stateof2}, since the empirical reward was at least $\frac{1}{2}$ after round $4\log^2(T)$. However, given that these customers are real, each sample is simply an independent Bernoulli random variable with mean $\frac{1}{2}$; denote the $i$--th such sample obtained during the above period by $X_i$. Suppose that we have $k \geq 2\log^2(T)$ samples\footnote{Ideally, $k=t$, however, we prove this more generally here since our goal is to prove that $\hat\mu_2 \geq \frac{1}{\sqrt{\log(T)}}$ in round $4\log^2(T)+t$.} corresponding to product two which were obained during the rounds $[4\log^2(T)+1, 4\log^2(T)+t]$. By Hoeffding's inequality, the probability that the sum of $k$ such Bernoulli variables is at most $(1/2 - \eps)k$ is given by:
$$\P\left(\sum_{i=1}^{k}X_i \leq (1/2 - \eps)k\right) \leq \exp(-2\eps^2 k).$$

In our case, $\eps = 1/2 - 1/\sqrt{\log(T)} \geq 1/4$ since $T > e^{16}$, so 
$$\P(\alpha_t) \leq \P(\hat\mu_2 < 1/\sqrt{\log(T)} ) \leq \P\left(\sum_{i=1}^{k}X_i \leq \frac{1}{\sqrt{\log(T)}} k \right)\leq \exp(-k/8).$$
Since $k \geq 2\log^2(T) \geq 32\log(T)$, we have that $\exp(-k/8) \leq e^{-4\log(T)} = \frac{1}{T^4} \leq \frac{1}{T^2}$ as desired. The theorem statement therefore follows.\hfill\Halmos


\section{{Proof of Statements in Section~\ref{sec:FAR}}}
\label{sec:proof:FAR}
In this section, we present the proofs of Lemmas \ref{lem_regret_decomp} and \ref{lem:concent}, followed by the proofs of Lemmas \ref{lem_pairwise_edgeadd} and \ref{lem_dagprop}.

\subsection{Proof of Lemma~\ref{lem_regret_decomp}} \label{sec:proof_lemma_regret_decomp}
Fix any arbitrary round $t$. Let $h_t(\pi^{\star}) = \mathcal{H}_t(\pi^{\star}) \setminus \mathcal{H}_{t-1}(\pi^{\star})$ and similarly $h_t(\pi_t) = \mathcal{H}_t(\pi_t) \setminus \mathcal{H}_{t-1}(\pi_t)$. That is, $h_t(\pi_t)$ consists of the ranking in round $t$, $\pi_t$, the binary variable $f_t$, indicating if the fake user is present in round $t$, and $c_t$, which includes the product clicked on in round $t$ (if any) and the exit position if not.  
Since the actions occurring after round $t$ do no affect the reward in this round, the incremental regret from this round can be written as:
\begin{align*}
\E_{\mathcal{H}_T(\pi^{\star})}\Big[ \mathcal{C}_{r,t}(\pi^\star)\Big]- \E_{\mathcal{H}_T(\pi_t)}\Big[ \mathcal{C}_{r,t}(\pi_t) \Big] & =  \E_{\mathcal{H}_t(\pi^{\star})}\Big[ \mathcal{C}_{r,t}(\pi^\star)\Big]- \E_{\mathcal{H}_t(\pi_t)}\Big[ \mathcal{C}_{r,t}(\pi_t) \Big].
\end{align*}
Note that we have replaced $\mathcal{H}_T(\pi^{\star})$ and $\mathcal{H}_T(\pi_t)$ with $\mathcal{H}_t(\pi^{\star})$ and $\mathcal{H}_t(\pi_t)$ respectively. Further, by the law of total expectation, we have:
\begin{align}
    \E_{\mathcal{H}_t(\pi^{\star})}\Big[ \mathcal{C}_{r,t}(\pi^\star)\Big] & = \E_{\mathcal{H}_{t-1}(\pi^{\star})}\Big[\E_{h_t(\pi^{\star})}[ \mathcal{C}_{r,t}(\pi^\star) ~|~ \mathcal{H}_{t-1}(\pi^{\star})]\Big]\label{eqn_iteratedexp_lem1_1}\\
    \E_{\mathcal{H}_t(\pi_t)}\Big[ \mathcal{C}_{r,t}(\pi_t) \Big] & =\E_{\mathcal{H}_{t-1}(\pi_t)}\Big[\E_{h_t(\pi_t)}[ \mathcal{C}_{r,t}(\pi_t) ~|~ \mathcal{H}_{t-1}(\pi_t)]\Big]. \label{eqn_iteratedexp_lem1_2}
\end{align}
Consider any arbitrary instantiations of $\mathcal{H}_{t-1}(\pi^{\star})$ and $\mathcal{H}_{t-1}(\pi_t)$. Given these instantiations, and a fixed policy $P_t$ for the fake user in round $t$, we recall that $h_t(\pi_t)$ depends only on the randomness in $P_t$ and conditional upon the customer being real, the randomness in her click behavior. For this fixed instantiation, the regret in round $t$ can be expressed in the form: 
\begin{equation}
   \E_{h_t(\pi^{\star}), h_t(\pi_t)}\big[\mathcal{C}_{r,t}(\pi^\star) - \mathcal{C}_{r,t}(\pi_t)\big].
    \label{eqn_simplified_reg_lem1}
\end{equation}
We omit the conditional operator inside the expectation to keep the notation readable. Let $f_t(\pi^{\star})$ and $f_t(\pi_t)$ denote the events that $h_t(\pi^{\star})$ and $h_t(\pi_t)$ contain a fake user in round $t$ respectively---i.e., $f_t(\pi_t)=1$ when our algorithm sees a fake user in round $t$ under $h_t(\pi_t)$. The rest of the proof proceeds in two cases depending on the realized value of $f_t(\pi_t)$.
\begin{itemize}
    \item \textbf{Case I.} $f_t(\pi_t) = 1$. Since $\mathcal{C}_{r,t}(\pi_t)=0$ in this case, we have that:
    \begin{align}
        \mathcal{C}_{r,t}(\pi^\star) - \mathcal{C}_{r,t}(\pi_t) & \leq 1 
         ~=~ f_t(\pi_t). \label{eqn_lem1_case1}
    \end{align}

    \item \textbf{Case II.} $f_t(\pi_t) = 0$. This is the main case, so we proceed carefully. For any arbitrary ranking $\pi$, the expected reward from a real customer can be quantified as:
\begin{align*}
    \E\Big[\mathcal{C}_{r,t}(\pi) ~|~ f_t(\pi) = 0\Big] = \sum_{\p{j}=1}^{n} Q(\p{j}) \prod_{\p{r}=1}^{\p{j}-1}(1-\mu_{\pi(\p{r})})\mu_{\pi(\p{j})}\,.
\end{align*}
Here,  $Q(\p{j}) = \prod_{\p{r} = 1}^{\p{j}-1}(1-q_{\p{r}}),$ is the probability that the customer does not exit before viewing the products in the first $\p{j}$ positions.  By definition, $\pi^{\star}(\p{j}) = j$ for all $j \leq n$. Therefore, when $f_t(\pi_t) = f_t(\pi^{\star}_t) = 0$, both our algorithm and the optimal algorithm see a real customer in round $t$ and the incremental regret conditional on this can be written as:
\begin{align}
\E\big[\mathcal{C}_{r,t}(\pi^{\star})- \mathcal{C}_{r,t}(\pi_t)~ | ~ f_t(\pi^{\star}) =0, f_t(\pi_t) = 0\big] = \sum_{\p{j}=1}^n Q(\p{j}) \left[ \prod_{\p{r}=1}^{\p{j}-1}(1-\mu_{\pi^{\star}(\p{r})}) \mu_{\pi^{\star}(\p{j})} -  \prod_{\p{r}=1}^{\p{j}-1}(1-\mu_{{\pi_t(\p{r})}})\mu_{{\pi_t(\p{j})}}\right]\,. \label{eqn_incregret_optreal}
\end{align}

Next, consider the case where $f_t(\pi_t) = 0$ but $f_t(\pi^{\star}) = 1$---i.e., the optimal algorithm sees a fake user in round $t$. When $f_t(\pi^{\star}_t) = 1$, we have that $\mathcal{C}_{r,t}(\pi^{\star}) = 0$ due to the presence of a fake user and therefore,
\begin{align}
\E\big[\mathcal{C}_{r,t}(\pi^{\star})- \mathcal{C}_{r,t}(\pi_t)~ | ~ f_t(\pi^{\star}) =1, f_t(\pi_t) = 0\big] & \leq 0 \notag\\
& \leq \sum_{\p{j}=1}^n Q(\p{j}) \left[ \prod_{\p{r}=1}^{\p{j}-1}(1-\mu_{\pi^{\star}(\p{r})}) \mu_{\pi^{\star}(\p{j})} -  \prod_{\p{r}=1}^{\p{j}-1}(1-\mu_{{\pi_t(\p{r})}})\mu_{{\pi_t(\p{j})}}\right]\,. \label{eqn_incregret_optfake}
\end{align}
Note that the inequality above is due to the fact that the expression in~\eqref{eqn_incregret_optfake} is non-negative due to the optimality of $\pi^{\star}.$

Combining the two cases from~\eqref{eqn_incregret_optreal} and~\eqref{eqn_incregret_optfake}, we get have:
\begin{align}\E\big[\mathcal{C}_{r,t}(\pi^{\star})- \mathcal{C}_{r,t}(\pi_t)~ | ~  f_t(\pi_t) = 0\big] & \leq \sum_{\p{j}=1}^n Q(\p{j}) \left[ \prod_{\p{r}=1}^{\p{j}-1}(1-\mu_{\pi^{\star}(\p{r})}) \mu_{\pi^{\star}(\p{j})} -  \prod_{\p{r}=1}^{\p{j}-1}(1-\mu_{{\pi_t(\p{r})}})\mu_{{\pi_t(\p{j})}}\right]\nonumber\\
& \leq \sum_{\p{j}=1}^{n} Q(\p{j}) \left[ \prod_{\p{r}=1}^{\p{j}-1}(1-\mu_{{\pi_t(\p{r})}}) \mu_{\pi^{\star}(\p{j})} -  \prod_{\p{r}=1}^{\p{j}-1}(1-\mu_{{\pi_t(\p{r})}})\mu_{{\pi_t(\p{j})}}\right]
\notag
\\
& = \sum_{\p{j}=1}^{n} Q(\p{j}) \prod_{\p{r}=1}^{\p{j}-1}(1-\mu_{{\pi_t(\p{r})}}) \Delta_{\pi^{\star}(\p{j}),{\pi_t(\p{j})}}\nonumber\\
& \leq \sum_{\p{j}=1}^{n} \sum_{i=\p{j}+1}^n Q(\p{j}) \ind(\pi_t(\p{j})=i) \prod_{\p{r}=1}^{\p{j}-1}(1-\mu_{{\pi_t(\p{r})}}) \Delta_{\p{j},i}\nonumber\\
& = \sum_{\p{j}=1}^{n} \sum_{i=\p{j}+1}^n \E\big[\event_{r,t}(\pi_t(\p{j}) = i) ~|~ f_t(\pi_t)=0\big] \Delta_{\p{j},i} \label{eqn_lem1_case3}\,.
\end{align}
The second inequality follows from the fact that for any position $\p{j}$, the probability that a customer does not click on any of the first $\p{j}-1$ ranked products is
smallest under $\pi^{\star}$ among all possible rankings. This is because $\pi^{\star}$ ranks products in the decreasing order of the click probabilities. The final equation follows from the definition of the event $\event_{r,t}(\pi_t(\p{j}) = i)$ as stated in Lemma~\ref{lem_regret_decomp}. 
\end{itemize}

Combining~\eqref{eqn_lem1_case1} and~\eqref{eqn_lem1_case3}, and plugging these back into~\eqref{eqn_simplified_reg_lem1}, we get that:
\begin{align}
\E\big[\mathcal{C}_{r,t}(\pi^\star) - \mathcal{C}_{r,t}(\pi_t)\big] & \leq \E\big[\mathcal{C}_{r,t}(\pi^\star) - \mathcal{C}_{r,t}(\pi_t) ~|~f_t(\pi_t) = 1\big]\cdot \P\left(f_t(\pi_t)=1 \right) \notag \\ &  + \E\big[\mathcal{C}_{r,t}(\pi^\star) - \mathcal{C}_{r,t}(\pi_t) ~|~f_t(\pi_t) = 0\big]\cdot \P\left(f_t(\pi_t) = 0\right) \notag \\
& \leq \E[f_t(\pi_t)] + \sum_{\p{j}=1}^{n} \sum_{i=\p{j}+1}^n \E\big[\event_{r,t}(\pi_t(\p{j}) = i)  ~|~ f_t(\pi_t)=0 \big] \Delta_{\p{j},i} \P(f_t(\pi_t) = 0) \notag \\ \notag
& \leq \E[f_t(\pi_t)] + \sum_{\p{j}=1}^{n} \sum_{i=\p{j}+1}^n \E\big[\event_{r,t}(\pi_t(\p{j}) = i)\big] \Delta_{\p{j},i}\,.
\end{align}

Since the above bound is valid for any arbitrary instantiations of $\mathcal{H}_{t-1}(\pi^{\star})$ and $\mathcal{H}_{t-1}(\pi_t)$, we can utilize~\eqref{eqn_iteratedexp_lem1_1} and~\eqref{eqn_iteratedexp_lem1_2} and simply sum over all rounds $t$ to get:
   $$\sum_{t=1}^T\E_{\mathcal{H}_t(\pi^{\star})}\Big[ \mathcal{C}_{r,t}(\pi^\star)\Big] - \sum_{t=1}^T\E_{\mathcal{H}_t(\pi_t)}\Big[ \mathcal{C}_{r,t}(\pi_t)\Big] \leq  \sum_{t=1}^T\sum_{\p{j}=1}^{n} \sum_{i=\p{j}+1}^n \E_{\mathcal{H}_t(\pi_t)}\big[\event_{r,t}(\pi_t(\p{j}) = i)\big] \Delta_{\p{j},i}  + \sum_{t=1}^T\E_{\mathcal{H}_t(\pi_t)}[f_t(\pi_t)]. $$

   To conclude, recall that the total number of fake users that our algorithm is exposed to is bounded by $F$---i.e., there are at most $F$ rounds in $[T]$ where fake users arrive. This, in turn,  implies that $\sum_{t=1}^T\E[f_t(\pi_t)~ |~\mathcal{H}_T(\pi_t)] \leq F$ for any given choice of the history $\mathcal{H}_T(\pi_t)$. Substituting this back into the above expression gives us Equation \eqref{eqn_regret_decomp1} and the lemma statement.\hfill\Halmos

\subsection{Proof of Lemma~\ref{lem:concent}}

We will first prove a more abstract claim. Let $X_1, X_2, \dots, X_T$ be a sequence of $\{0, 1\}$ random variables such that at least $T-F$ of these random variables are i.i.d. Bernoulli random variables with expectation $\mu$ (with the remainder chosen adversarially). Let $R_{\S} = \frac{1}{\S}\sum_{i=1}^{\S} X_{i}$ be the average of the first ${\S}$ values in the sequence $(X_1, X_2, \ldots, X_T)$, and define 

$$\varepsilon({\S}) = \sqrt{\frac{\log(2nT/\delta)}{{\S}}}$$

\noindent
and 
$$w({\S}) = \eps({\S}) + \frac{F}{{\S}}\,.$$
We then claim that with high probability, $R_{\S} \in [\mu - w({\S}), \mu + w({\S})]$ for all ${\S} \in [T]$. To see why, first define a related sequence $\tilde{X}_{\S}$ of random variables such that $\tilde{X}_{\S} = X_{\S}$ if $X_{\S}$ is one of the $T-F$ non-adversarial random variables; otherwise, let $\tilde{X}_{\S}$ be an i.i.d. Bernoulli random variable with mean $\mu$. Note that $|\sum_{i=1}^{T} X_{i} - \sum_{i=1}^{T} \tilde{X}_{i}| \leq F$, since there are at most $F$ adversarial random variables and $X_i \in \{0,1\}$ for all $i \in [T]$. It follows that if we let $\tilde{R}_{\S} = \frac{1}{{\S}}\sum_{i=1}^{{\S}}\tilde{X}_{i}$, then $|\tilde{R}_{\S} - R_{\S}| \leq \frac{F}{{\S}}$. Therefore, in order to prove our earlier claim regarding $R_{S}$, it suffices to show that with high probability $\tilde{R}_{{\S}} \in [\mu - \eps({\S}), \mu + \eps({\S})]$ for all $S \in [T]$.

Since $\tilde{R}_{{\S}}$ is the average of ${\S}$ i.i.d. Bernoulli random variables with mean $\mu$, we can directly apply Hoeffding's concentration inequality. Hoeffding's inequality immediately implies that 

\begin{equation}
\P\left(|\tilde{R}_{{\S}} - \mu| \geq \eps({\S})\right) \leq 2\exp\left(-2\eps({\S})^2 \cdot S\right).
\end{equation}

Substituting our expression for $\eps({\S})$, we have that

$$\P\left(|\tilde{R}_{{\S}} - \mu| \geq \eps({\S})\right) \leq \frac{\delta^2}{2n^2T^2} \leq \frac{\delta^2}{n^2T^2}$$

This holds for any fixed ${\S} \in [T]$. Taking the union bound over all ${\S} \in [T]$, it follows that the probability there exists an ${\S} \in [T]$ such that $\tilde{R}_{{\S}} \not\in [\mu - \eps({\S}), \mu + \eps({\S})]$ is at most $\delta^2/n^2T$. As a corollary, we also have that the probability there exists an ${\S} \in [T]$ such that $R_{{\S}} \not\in [\mu - w({\S}), \mu + w({\S})]$ is at most $\delta^2/n^2T$.

We now describe how to apply this to the problem at hand. Fix a specific product $i$. Our goal is show that (with high probability) for all rounds $t \in [T]$, $r_{i, t} \in [\mu_i - w_{i, t}, \mu_i + w_{i, t}]$. Recall that $r_{i,t}$ is the empirical average reward for product $i$ at the end of round $t$ and $w_{i,t}  = \sqrt{\frac{\log(\frac{2nT}{\delta})}{\eta_{i,t}}} + \frac{F}{\eta_{i,t}}$ as per our definition in Equation~\eqref{eqn_window_definition1}.

Define a sequence of random variables $Y_1, Y_2, \dots$, so that $Y_{j} = 1$ if product $i$ recevied a click during the $j$--th time the algorithm received feedback on it (and $Y_{j} = 0$ if product $i$ did not get clicked on). Note that the length of this sequence is the number of times the algorithm receives feedback about product $i$, i.e., $\eta_{i,T}$. If this number is less than $T$, pad the sequence with i.i.d. Bernoulli  random variables with mean $\mu_i$ until it has length $T$.

Note that this sequence satisfies the properties of our earlier claim; it is a sequence of $T$ random variables where all but (at most) $F$ are i.i.d. Bernoulli random variables with mean $\mu_i$. Specifically, for any $1 \leq m \leq T$, let $\mathcal{Y}_{i,m}$ denote the event that $\frac{1}{m}\sum_{k=1}^{m}Y_k \in [\mu_i - w(m), \mu_i + w(m)]$; then by our above claim $\P\left(\bigcap_{m=1}^{T}\mathcal{Y}_{i, m}\right) \geq 1 - \frac{\delta^2}{n^2T}$. Recall that $w({\S}) = \eps({\S}) + \frac{F}{{\S}}$, where  $\varepsilon({\S}) = \sqrt{\frac{\log(2nT/\delta)}{{\S}}}$.

We are now ready to complete the proof. As mentioned earlier, our goal is to show that $\mathcal{E}$ holds with probability $1-\frac{\delta^2}{nT}$, where $\mathcal{E} = \bigcap_{t\in [T], i \in [n]}\mathcal{E}_{i,t}$ and $\mathcal{E}_{i,t}$ denotes the event that $r_{i, t} \in [\mu_i - w_{i, t}, \mu_i + w_{i,t}]$. We claim that for a fixed product $i \in [n]$, if events $\mathcal{Y}_{i, m}$ hold for all $m \in [T]$, this implies that events $\mathcal{E}_{i, t}$ also hold for all $t \in [T]$ and the same product $i$. This is true since $r_{i, t}$ is always the average of one particular prefix of the sequence $(Y_1, Y_2, \ldots, Y_T)$; in particular,

$$r_{i, t} = \frac{1}{\eta_{i, t}}\sum_{k=1}^{\eta_{i, t}}Y_{k}.$$
Therefore, the validity of event $\mathcal{Y}_{i,\eta_{i, t}}$ implies that event $\mathcal{E}_{i, t}$ also holds. Note here that $\eta_{i,t} \in [T]$, and $w_{i, t}$ is simply $w(\eta_{i, t})$ as per our definition earlier. In conclusion, for a fixed product $i \in [n]$, we have that:
$$(\mathcal{Y}_{i,m})_{m=1}^{T} \text{~ is true~} \implies (\mathcal{E}_{i,t})_{t=1}^{T} \text{~ is true~}$$

It therefore follows that with probability at least $1 - \frac{\delta^2}{n^2T}$, it is true that $\mathcal{E}_{i, t}$ holds for all $t \in [T]$ for a fixed product $i \in [n]$. Taking the union bound over all $n$ possibilities for $i$, we have that the event $\mathcal{E}$ holds with probability at least $1 - \frac{\delta^2}{nT}$, as desired.\hfill\Halmos

\subsection{Proof of Lemma~\ref{lem_pairwise_edgeadd}}
The first part of the lemma states that under event $\mathcal E$, the graph $G$ does not contain any inaccurate  edges in any round $t$. We show this result by contradiction. Let $i>\p{j}$. Contrary to our claim, assume that in some round $t$,  the graph $G$ contains an incorrect edge from $\p{j}$ to $i$. According to Condition~\eqref{eq:add_edge}, this implies that  the upper confidence bound on the estimate of the click probability of product $\p{j}$ must be smaller than or equal to the lower confidence bound on the estimate of the click probability of product  $i$. However, comparing the actual value of these two quantities, we see that:
\begin{align*}
    r_{\p{j},t} + w_{\p{j},t} & \geq \mu_{j}  > \mu_{i} \geq r_{i,t} - w_{i,t}\,,
\end{align*}
where the first and third inequality hold because we assume that event $\mathcal E$ holds.  The second inequality holds because $i> \p{j}$ and hence $\mu_{i} < \mu_{\p{j}}$. The above inequality shows that the upper bound on product $\p{j}$ cannot be smaller than or even equal to the lower bound on product $i$ and consequently, the product-ordering graph $G$ will not contain the incorrect edge $(\p{j},i)$. 

To show the second part of the lemma, we verify that for any $i>\p{j}$ when $\eta_{i,t}, \eta_{\p{j},t} \geq \gamma_{\p{j},i}$, we have  $r_{i,t} + w_{i,t} \leq r_{\p{j},t} - w_{\p{j},t}.$ As we argued earlier, when $\eta_{i,t}, \eta_{\p{j},t} \geq \gamma_{\p{j},i}$, it must be the case that $w_{i,t}, w_{\p{j},t} \leq \tfrac{\Delta_{\p{j}, i}}{4}$; see our discussion after Equation \eqref{eqn_sub_optplays}.  Proceeding along these lines, we have that:
\begin{align*}
r_{i,t} + w_{i,t} & \leq \mu_{i} + w_{i,t} + w_{i,t}\\
& = \mu_{i} + 2w_{i,t}\\
& = \mu_{\p{j}} - \Delta_{\p{j}, i} + 2w_{i,t}\\
& \leq r_{\p{j},t} + w_{\p{j},t} + 2w_{i,t} - \Delta_{\p{j}, i} \\
& = \left(r_{\p{j},t} - w_{\p{j},t}\right) + 2w_{\p{j},t} + 2w_{i,t} - \Delta_{\p{j}, i} \\
& \le \left(r_{\p{j},t} - w_{\p{j},t}\right) + \frac{\Delta_{\p{j}, i}}{2} + \frac{\Delta_{\p{j}, i}}{2} - \Delta_{\p{j}, i} \\
& = r_{\p{j},t} - w_{\p{j},t}\,,
\end{align*}
where the first and second inequalities follow from our assumption that event $\mathcal E$ holds. The third inequality holds because  $w_{i,t}, w_{\p{j},t} \leq \tfrac{\Delta_{\p{j}, i}}{4}$. 
 Since the upper bound on product $i$, i.e., $r_{i,t} + w_{i,t}$, is smaller than the lower bound on product ${j}$, i.e., $r_{j,t} - w_{j,t}$, by Condition~\eqref{eq:add_edge}, our graph should contain a correct edge from $i$ to $\p{j}$ in any round $t' \geq t$.\hfill\Halmos

\subsection{Proof of Lemma~\ref{lem_dagprop}}

Consider products $i, \p{j}$ as mentioned in the statement of the lemma and suppose that the \Call{GraphRankSelect}{} algorithm ranks $i$ ahead of $\p{j}$ in round $t+1$ despite it being the case that $i > \p{j}$. First, we note that the lemma follows trivially if either $\eta_{i,t}=0$ or $\eta_{\p{j},t}=0$. Indeed, if $\eta_{i,t} = 0$, then the lemma holds trivially for $\ki=\p{j}$. If $\eta_{\p{j},t}=0$ and $\eta_{i,t} > 0$, then there will be no edge attached to product $j$ and as such, our algorithm will select $\p{j}$ ahead of $i$, so the condition mentioned in the lemma cannot be true. For the rest of this proof, we will assume that $\eta_{i,t}, \eta_{\p{j},t} > 0$.

Since event $\mathcal E$ holds, by Lemma \ref{lem_pairwise_edgeadd},   the graph $G$ does not contain any erroneous edges. This implies that there is no edge between $i$ and $\p{j}$ in either direction in the graph at this time. Indeed an edge from $\p{j}$ to $i$ would be an incorrect edge and an edge going from $i$ to $\p{j}$ would lead to $\p{j}$ being ranked  ahead of $i$, which violates the conditions of the lemma.

Now, during the execution of \Call{GraphRankSelect}{}, suppose that product $i$ was added to position $\p{r}$ in $\pi_{t+1}$ (i.e., $\pi_{t+1}(\p{r}) = i$) and let $S_{\p{r}}$ denote the set of products with no outgoing edges in $\hat{G}$ (as per Algorithm~\ref{alg_graph_rank}) at the beginning of the \Call{GraphRankSelect}{} algorithm's $\p{r}$--th iteration before product $i$ was chosen. Our main claim is that $S_{\p{r}}$ contains some product $\ki$ such that $\ki \leq \p{j}$. Define $\ki$ to be the smallest indexed product such that $\pi^{-1}_{t+1}(\ki) > \p{r}$. Clearly, $\ki \leq \p{j}$ since $\pi^{-1}_{t+1}(\p{j}) > \pi^{-1}_{t+1}(i)= \p{r}$ by definition. Further, product $\ki$ must belong to the set $S_{\p{r}}$ (i.e., have no outgoing edges in round $t$ in $\hat{G}$) since all the products better than product $\ki$ have already been ranked at positions smaller or more visible than $\p{r}$ and graph $G$ (and consequently graph $\hat G$) does not have any erroneous edges. Then, since  $i, \ki \in S_{\p{r}}$,  we must have, $\eta_{i,t} \leq \eta_{\ki,t}$ because \Call{GraphRankSelect}{} selects the product with the smallest feedback count ($\eta$ value) when multiple products have no outgoing edges. 

Part (c) of the lemma now follows trivially from Lemma~\ref{lem_pairwise_edgeadd} because if it were true that $\eta_{\ki,t} \geq \eta_{i,t} \geq \gamma_{k,i}$, then according to Lemma~\ref{lem_pairwise_edgeadd}, there would be an edge from product $i$ to product $\ki$ at the end of round $t$, which is a contradiction. This concludes our proof of the lemma. \hfill\Halmos

\section{Proof of Statements in Section \ref{sec:FORC}
} 
\label{app:FORC}
In this section, we first present the proof of Lemma \ref{lem_maxfake_levels} and a corollary of this lemma. We then provide the proof of Lemmas \ref{lem:even:E}, \ref{lem:parts1}, and \ref{lem:tjmax}.

\subsection{Proof of Lemma~\ref{lem_maxfake_levels}}
In order to show Lemma~\ref{lem_maxfake_levels}, we first establish the following auxiliary lemma that proves the same claim but for each individual level $\ell \geq \log_2(F)$.   
\begin{lemma}
\label{lem_maxfake}
For  any policy $\mathbf{P}$ adopted by the fake users with fakeness budget $F$ and any given $\delta \in (0,1)$, with probability $1-\frac{\delta}{2\level}$,  any level $\ell \geq \log_2(F)$ is  exposed to at most $\log(2\frac{\level}{\delta})+3$ fake users. Here, $\level = \log_2(T)$.

\end{lemma}
\begin{proof}{Proof of Lemma~\ref{lem_maxfake}}
The proof is very similar to that of Lemma 3.3 in~\cite{lykouris2018stochastic} and we only sketch the key differences here to minimize redundancy. First, note that any level $\ell \geq \log_2(F)$ is selected with probability $2^{-\ell}$ in a given round and the total fakeness budget $F = 2^{\log_2(F)} \leq 2^{\ell}$. Therefore, the expected number of fake users that any level $\ell \geq \log_2(F)$ is exposed to is at most $1$. In light of this, our goal here is to show a high probability bound on this quantity for a fixed level $\ell \geq \log_2(F)$.

Let $f_t$ be a random variable such that $f_t=1$ denotes the presence of a fake user in round $t$. Note that $f_t$ can depend on the history prior to round $t$---i.e., $\mathcal{H}_{t-1}$---but is independent of the realization of the sampled level in round $t$, i.e., $\ell_t.$ Since the total fakeness budget is $F$, we know that $\sum_{t=1}^T f_t \leq F$ for any realization of the underlying randomness. 

Define the random variable $\ZE^{(\ell)}_t$ as follows: 
\begin{align*}
  \ZE^{(\ell)}_t \triangleq f_t\cdot \ind(\ell_t = \ell).  
\end{align*}
Note that $\ZE^{(\ell)}_t=1$ if the user in round $t$ is fake and the level is $\ell$ and evaluates to zero if either of these two conditions are not met. 

The objective of this lemma is to obtain a high probability bound on $\sum_{t=1}^T \ZE^{(\ell)}_t$. Consider the martingale $X_t \triangleq \ZE^{(\ell)}_t - \E_{\ell_t}[\ZE^{(\ell)}_t ~|~ \mathcal{H}_{t-1}]$ where $\E_{\ell_t}[\cdot]$ highlights that the expectation is w.r.t. the random choice of the level in round $t$. Since $\ZE^{(\ell)}_t = f_t$ with probability $\frac{1}{2^{\ell}}$, we have that $\E_{\ell_t}[\ZE^{(\ell)}_t ~|~ \mathcal{H}_{t-1}] = \frac{f_t}{2^{\ell}}.$ Next, we bound the variance of this martingale: 
\begin{align*}
    \E_{\ell_t}[X^2_t | X_1, \ldots, X_{t-1}] & = \frac{1}{2^{\ell}}\Big(f_t - \frac{f_t}{2^{\ell}}\Big)^2 + (1-\frac{1}{2^{\ell}})\Big(\frac{f_t}{2^{\ell}}\Big)^2  = 
    \frac{f_t^2}{2^{\ell}} - \Big(\frac{f_t}{2^{\ell}}\Big)^2
     \leq \frac{f_t}{2^{\ell}}.
\end{align*}

 Next, we bound the total variance over all rounds as 
\begin{align*}
\sum_{t=1}^T\E_{\ell_t}[X^2_t | X_1, \ldots, X_{t-1}] \leq \frac{1}{2^{\ell}}\sum_{t=1}^T f_t = \frac{F}{2^{\ell}} \leq 1.
\end{align*}
Here, the inequality holds because 
 $F \leq 2^{\ell}$. 
 
Finally, we apply a martingale concentration inequality, i.e., Bernstein's inequality as stated in \citet{beygelzimer2011contextual}. We start by repeating this lemma for convenience and then proceed with applying it to our setting.

\begin{lemma}[Lemma 1 in \citet{beygelzimer2011contextual}]
Let $X_1, X_2, \ldots, X_T$ be a sequence of real-valued random numbers. Assume, for all $t$, that $X_t \leq R$ and that $\E[X_t | X_1, \ldots, X_{t-1}] = 0$. Also, let
\begin{align*}
    V = \sum_{t=1}^{T} \E[X_t^2 | X_1, \ldots, X_{t-1}].
\end{align*}
Then, for any $\epsilon  >0$: 
\begin{align*}
    \P\Big(\sum_{t=1}^{T} X_t > R \log(1/\epsilon ) + \frac{e - 2}{R} . V\Big) \leq \epsilon 
\end{align*}
\end{lemma}

Applying the above lemma to our setting---in which we can set $R$ to be $1$ and $V$ to be $1$---we 
 get that with probability $1-\epsilon$, we have $\sum_{t=1}^T X_t \leq \log(\frac{1}{\epsilon}) + 1.$ Substituting $\epsilon = \frac{\delta }{2\level}$, we can now use this high-probability upper bound on $\sum_{t=1}^T X_t$ to achieve the desired result on $\sum_{t=1}^T \ZE^{(\ell)}_t$, which denotes the number of fake users that level $\ell$ is actually exposed to. 
\begin{align*}
    \sum_{t=1}^T \ZE^{(\ell)}_t & = \sum_{t=1}^T X_t + \sum_{t=1}^T \E_{\ell_t}[\ZE^{(\ell)}_t ~|~  \mathcal{H}_{t-1}]\\
    & \leq \sum_{t=1}^T \frac{f_t}{2^{\ell}} + \log(\frac{2\level}{\delta}) + 1 \\
    & \leq \log(\frac{2\level}{\delta}) + 2. \hfill\Halmos
\end{align*}

\end{proof}

\noindent With Lemma~\ref{lem_maxfake}, we complete the proof of Lemma~\ref{lem_maxfake_levels} by applying the union bound for all levels $\ell \geq \log_2(F)$ and noticing that $|\{\ell: \ell \geq \log_2(F)\}| \leq  \level$. \hfill\Halmos

\subsection{Corollary of Lemma \ref{lem_maxfake_levels}} \label{sec:corollary}

We start with a relevant definition and then proceed to a corollary of Lemma~\ref{lem_maxfake_levels}, which will be utilized in future proofs. For any product $i$, round $t$, and level $\ell$, we define $\tau^{(\ell)}_{i,t}$ as the set of rounds up to $t$ in which we receive feedback on product $i$ at level $\ell$, i.e.,
\begin{equation}
\tau^{(\ell)}_{i,t} \triangleq \{ t' \leq t ~|~ \ell_{t'} = \ell, \eta^{(\ell)}_{i,t'} = \eta^{(\ell)}_{i,t'-1} + 1 \}\,.
    \label{defn_tau}
\end{equation}
We note that $\tau^{(\ell)}_{i,t}$ is a random set that depends on the history.

\begin{corollary}[Corollary of Lemma \ref{lem_maxfake_levels}]
\label{corr_maxfake}
Given level $\ell \geq \ell^{\star}$, and conditioned on the event $\eventforc_2$, i.e., the event that 
 in the entire time horizon,  every level $\ell \geq \ell^{\star}$ is exposed to at most $\log(2\frac{\level}{\delta})+3$  fake users (see Definition~\ref{def:event}), the gap between the real reward and the observed reward for any product $i$ up to round $t \leq T$ can be bounded as follows: 
$$\Bigg|\sum_{t \in \tau_{i,t}^{(\ell)}}X_{i,t} - r^{(\ell)}_{i,t}\eta^{(\ell)}_{i,t}\Bigg| \leq \log(\frac{2\level}{\delta})+3\,.$$
Here, $X_{i,t}$ is the reward of product $i$ in round $t$, and  $\tau_{i,t}^{(\ell)}$ denotes the rounds up to $t$ in which we received feedback on product $i$ at level $\ell$ as defined in~\eqref{defn_tau}, $\delta =\frac{1}{n^3T}$, and $\level =\log_2(T)$.

\end{corollary}

\subsection{Proof of Lemma \ref{lem:even:E}} \label{sec:proof:lem:even:E} Our goal here is to show
the probability of event $\eventforc = \eventforc_1 \cap  \eventforc_2$ is at least $1-\delta = 1-\frac{1}{n^3T}$. Here,   $\eventforc_1 = \bigcap_{\substack{t \in [T]\\ \ell \geq \ell^{\star} \\ i \in [n]}} \eventforc^{(\ell)}_{i,t}$, where for any $t \in [T]$, $i \in [n]$, and $\ell  \geq \ell^{\star}= \ceil*{\log_2(F)}$, the event $\eventforc^{(\ell)}_{i,t}$ is defined as follows:
\begin{equation*}
  \eventforc^{(\ell)}_{i,t} = \{ \hat{r}^{(\ell)}_{i,t} - w^{(\ell)}_{i,t} \leq \mu_i \leq \hat{r}^{(\ell)}_{i,t} + w^{(\ell)}_{i,t} \}\,.
\end{equation*}
For the sake of brevity, we also define \begin{align}\eventforc^{(\ell)}_{i} \triangleq \bigcap_{t \in [T]} \eventforc^{(\ell)}_{i,t}\label{eq:event_G_i}\end{align}
for any product $i \in [n]$ and level $\ell \geq \ell^{\star}$. Finally,
 $\eventforc_2 
$ is the event that in the entire time horizon,  every level $\ell \geq \ell^{\star}$ is exposed to at most $\log(2\frac{\level}{\delta})+3$ fake users; see Lemma \ref{lem_maxfake_levels}. We begin the proof by bounding the probability of the event $\eventforc^{(\ell)}_{i}$ in the following claim---i.e., for a fixed product $i$, and level $\ell \geq \ell^{\star}$, the probability that the cross-learning empirical mean differs from the true mean of product $i$'s rewards. 

\begin{claim}\label{claim:event_2}
Conditioned  on event $\eventforc_2$, for any given product $i$ and level $\ell \geq \ell^{\star}$, event  $\eventforc^{(\ell)}_{i}$, defined in Equation \eqref{eq:event_G_i},  holds with probability at least $(1-\frac{1}{\P(\mathcal{G}_2)}\frac{\delta^3}{8n^3T})$, where $\delta = \frac{1}{n^3 T}$.
\end{claim}

The proof of the claim is deferred to the end. The probability that event $\eventforc$ holds can now be derived as follows. 
\begin{align*}
    \P(\eventforc) & = \P(\eventforc_1 \cap \eventforc_2)  = 1 - \P(\eventforc^c_1 \cup \eventforc^c_2) 
     = 1 - \P(\eventforc^c_2) - \P(\eventforc^c_1 \cap \eventforc_2) \\& = 1 - \P(\eventforc^c_2) - \P(\eventforc^c_1 ~|~  \eventforc_2) \P(\eventforc_2) 
     \geq 1 - \frac{\delta}{2} - \frac{\delta}{2}\frac{1}{\P(\mathcal{G}_2)} \times \P(\eventforc_2)  \ge 1-\delta\,.
\end{align*}
 In the above expressions,  we use (a) the trivial inequality that $\P(\eventforc_2) \leq 1$, (b)
 Lemma \ref{lem_maxfake_levels}, where we show $\P(\eventforc^c_2) \le \delta/2$, and (c) Claim  \ref{claim:event_2}. Specifically, by this claim, 
 $$\P(\eventforc_1^c | \eventforc_2) \leq \sum_{i=1}^n\sum_{\ell=\ell^{\star}}^{\log_2(T)} \P((\eventforc^{(\ell)}_{i})^c | \eventforc_2) \leq  n\log_2(T)\cdot \frac{1}{\P(\mathcal{G}_2)}\frac{\delta^3}{8n^3 T} \leq \frac{1}{\P(\mathcal{G}_2)}\frac{\delta}{2}. \hfill \Halmos$$

\subsubsection*{Proof of Claim \ref{claim:event_2}}
Fix product $i$ and level $\ell \geq \ell^{\star}$. We begin by noting that the following event is trivially true at any round $t$ in which $\hat{\eta}^{(\ell)}_{i,t} = 0$
\begin{equation*}
  \eventforc^{(\ell)}_{i,t} = \{ \hat{r}^{(\ell)}_{i,t} - w^{(\ell)}_{i,t} \leq \mu_i \leq \hat{r}^{(\ell)}_{i,t} + w^{(\ell)}_{i,t} \}\,.
\end{equation*}
This is because by definition---see~ Equation \eqref{eqn_windowsize_agnostic}---$w^{(\ell)}_{i,t} > 1$ when $\hat{\eta}^{(\ell)}_{i,t} = 0$. Moreover, since $0 \leq \mu_i \leq 1$, this implies that $\hat{r}^{(\ell)}_{i,t} - w^{(\ell)}_{i,t} \leq \mu_i \leq \hat{r}^{(\ell)}_{i,t} + w^{(\ell)}_{i,t}$. For the remainder of this proof, without loss of generality, we assume that  any round $t$ considered satisfies $\hat{\eta}^{(\ell)}_{i,t} > 0$.

Next, the cross-learning empirical mean in round $t$ can be expressed as
\[ \hat r_{i,t}^{(\ell)} = \frac{\sum_{g=1}^{\ell-1}\eta^{(g)}_{i,t} r_{i,t}^{(g)}}{2^{\ell}\hat\eta^{(\ell)}_{i,t}}+ \frac{\eta_{i,t}^{(\ell)}}{\hat\eta^{(\ell)}_{i,t}} r_{i,t}^{(\ell)}\,, \]
where $\hat\eta_{i,t}^{(\ell)} = \frac{\sum_{g=1}^{\ell-1}\eta^{(g)}_{i,t}}{2^{\ell}}+ \eta_{i,t}^{(\ell)}$. Note that $r^{(\ell)}_{i,t}\eta^{(\ell)}_{i,t}$ is the sum of observed rewards for product $i$ at level $\ell$ over all the rounds where the algorithm selects level $\ell$. Some of these observations could be from fake users. However, we know from Lemma~\ref{lem_maxfake_levels} that with high probability, at most $\log(\frac{2\level}{\delta})+3$ of these rounds  contain fake users and potentially corrupted rewards. Therefore, before showing a concentration inequality for $\hat r_{i,t}^{(\ell)}$ as is required for this claim, we first prove bounds on the empirical mean (for the hypothetical case) when there is no fake user and all of the rewards are sampled i.i.d. from a Bernoulli distribution with mean $\mu_i$. We then use the triangle equality and an upper bound on the number of fake users to transform this hypothetical scenario to the case with fake users.

Formally, we can interpret the process by which rewards are obtained from real customers as follows. Suppose that there two sequences of real rewards $(X_1, X_2, \ldots,X_T)$ and $(Y_1, Y_2, \ldots, Y_T)$ of each of length $T$ such that:
\begin{enumerate}
    \item All $T$ samples in both sequences are independently drawn from a Bernoulli distribution with mean $\mu_i$.
    \item If our algorithm picks level $\ell$ at some round $t$ and we receive feedback on product $i$ during this round from a real customer, then we observe sample $X_{\eta_{1,t}}$ where (for simplicity) $\eta_{1,t} \triangleq \eta^{(\ell)}_{i,t}$ is the total number of times we have received feedback on product $i$ up to round $t$.
    \item If our algorithm picks level $g < \ell$ at some round $t$ and we receive feedback on product $i$ during this round from a real customer, then we observe sample $Y_{\eta_{2,t}}$ where $\eta_{2,t} \triangleq \sum_{g=1}^{\ell-1}\eta^{(g)}_{i,t}$ is the total number of times we have received feedback on product $i$ on all levels smaller than $\ell$ up to round $t$.
\end{enumerate}
Note that when there is a fake user in a particular round, the corresponding sample ($X_{\eta_{1,t}}$ or $Y_{\eta_{2,t}}$) is ignored and a potentially different sample is obtained. 

Consider the term:
\[ \hat{x}^{(\ell)}_{i,t} \triangleq \frac{\sum_{t'=1}^{\eta_{1,t}}X_{t'}}{\hat\eta^{(\ell)}_{i,t}} + \frac{\sum_{t' =1}^{\eta_{2,t}}Y_{t'} }{2^{\ell}\hat\eta^{(\ell)}_{i,t}}, \]
where $\hat\eta^{(\ell)}_{i,t} = \eta_{1,t} + \eta_{2,t}/2^{\ell}$, as defined in Algorithm~\ref{alg_cascading_agnostic}. This quantity denotes the cross-learning empirical mean for level $\ell$ in the absence of fake customers. We seek to apply Lemma~\ref{lem_hoeffdingweighted_mean2} to bound the probability that $\hat{x}^{(\ell)}_{i,t}$, defined above, exceeds its true mean. First, for any given $\eta > 0$, define 
$$\epsilon_{\eta} \triangleq  \sqrt{\frac{3}{2}\frac{\log(\frac{4nT}{\delta})}{\eta}}\,,$$
where $\delta =\frac{1}{n^3T}$. Note that $\epsilon_{{\hat \eta}^{(\ell)}_{i,t}}= w^{(\ell)}_{i,t} - \frac{\log(\frac{2\level }{\delta})+4}{{\hat \eta}^{(\ell)}_{i,t}}$, 
where $w^{(\ell)}_{i,t}$ is the window size as defined in Equation \eqref{eqn_windowsize_agnostic}.
\begin{align*}
\P\left( \exists t:  |\hat{x}^{(\ell)}_{i,t} - \mu_i| > \epsilon_{{\hat \eta}^{(\ell)}_{i,t}} \right) & = \P\left( \exists (\eta_{1,t}, \eta_{2,t}): \left| \frac{\sum_{t'=1}^{\eta_{1,t}}X_{t'}}{\hat\eta^{(\ell)}_{i,t}} + \frac{\sum_{t' =1}^{\eta_{2,t}}Y_{t'} }{2^{\ell}\hat\eta^{(\ell)}_{i,t}} - \mu_i \right| > \epsilon_{{\hat \eta}^{(\ell)}_{i,t}} \right) \notag \\
& \leq \P\left( \exists (\eta_{1}, \eta_{2}): \left| \frac{\sum_{t=1}^{\eta_{1}}X_{t}}{\hat\eta} + \frac{\sum_{t =1}^{\eta_{2}}Y_{t} }{2^{\ell}\hat\eta} - \mu_i \right| > \epsilon_{\hat\eta}~,~  0 <\eta_1, \eta_2\le T~,~ \hat \eta = \eta_1+\frac{\eta_2}{2^{\ell}} \right) \\
& \leq \sum_{\eta_1=0}^T \sum_{\substack{\eta_2=0}}^T \ind(\eta_1 + \eta_2 > 0)\P\left( \left| \frac{\sum_{t=1}^{\eta_{1}}X_{t}}{\hat\eta} + \frac{\sum_{t =1}^{\eta_{2}}Y_{t} }{2^{\ell}\hat\eta} - \mu_i \right| > \epsilon_{\hat\eta} \right),
\end{align*}
 where we use  $\hat\eta$  as short hand for $\eta_1 + \eta_2/2^{\ell}$. We now apply Lemma~\ref{lem_hoeffdingweighted_mean2} directly on the expression in the last inequality  taking $\epsilon = \epsilon_{\hat\eta} = \sqrt{\frac{3\log(\frac{4nT}{\delta})}{2\hat\eta}}$. This gives us:
\begin{align}
    \sum_{\eta_1=0}^T \sum_{\eta_2=0}^T \ind(\eta_1 + \eta_2 > 0)\P\left( \left| \frac{\sum_{t=1}^{\eta_{1}}X_{t}}{\hat\eta} + \frac{\sum_{t =1}^{\eta_{2}}Y_{t} }{2^{\ell}\hat\eta} - \mu_i \right| > \epsilon_{\hat\eta} \right) & \leq \sum_{\eta_1=0}^T \sum_{\eta_2=0}^T 2\exp(-3\cdot\log(\frac{4nT}{\delta})) \notag\\
    & = 2 (T+1)^2 \cdot \frac{\delta^3}{4^3 n^3T^3}\notag\\
    & \leq 2 \cdot  (2T)^2 \cdot \frac{\delta^3}{4^3 n^3T^3}\notag\\
    & = \frac{\delta^3}{8n^3T}.\notag
\end{align}

Note that the second inequality  is due to $T +1 \leq 2T$ since $T \geq 1$. In summary, we have that:
\begin{equation}
\P\left(\exists t: \left| \hat{x}^{(\ell)}_{i,t} - \mu_{i} \right| > w^{(\ell)}_{i,t} - \frac{\log(\frac{2\level}{\delta})+4}{{\hat \eta}^{(\ell)}_{i,t}}\right) \leq \frac{\delta^3}{8n^3T}\,.
\label{eqn_hoeffding_true}
\end{equation}

Using this, we can bound the desired quantity $\hat{r}^{(\ell)}_{i,t}$ which differs from $\hat{x}^{(\ell)}_{i,t}$ only in terms of the feedback due to fake users. 
\begin{align}
    |\hat{r}^{(\ell)}_{i,t} - \mu_{i} | & \leq |\hat{r}^{(\ell)}_{i,t} - \hat{x}^{(\ell)}_{i,t} | + |\hat{x}^{(\ell)}_{i,t} - \mu_{i} | \notag\\
    & \leq \frac{1}{2^{\ell}\hat {\eta}^{(\ell)}_{i,t}}\left|\sum_{g=1}^{\ell-1}r_{i,t}^{(g)}\eta_{i,t}^{(g)} -  \sum_{g=1}^{\ell-1}\sum_{t' \in \tau^{(g)}_{i,t}}X_{i,t'} \right| +  \frac{1}{\hat {\eta}^{(\ell)}_{i,t}}|r_{i,t}^{(\ell)}\eta_{i,t}^{(\ell)} - \sum_{t' \in \tau^{(\ell)}_{i,t}}X_{i,t'}  | + |\hat{x}^{(\ell)}_{i,t} - \mu_{i} | \notag\\
    & \leq \frac{1}{2^{\ell}\hat {\eta}^{(\ell)}_{i,t}} F + \frac{\log(\frac{2\level }{\delta})+3}{{\hat \eta}^{(\ell)}_{i,t}}  + |\hat{x}^{(\ell)}_{i,t} - \mu_{i} | \label{eqn_fakeness_measure}\\
    & \leq \frac{1}{\hat {\eta}^{(\ell)}_{i,t}} + \frac{\log(\frac{2\level}{\delta})+3}{{\hat \eta}^{(\ell)}_{i,t}}  + |\hat{x}^{(\ell)}_{i,t} - \mu_{i} | \label{eqn_max_fake}\\
    & = \frac{\log(\frac{2\level}{\delta})+4}{{\hat \eta}^{(\ell)}_{i,t}}  + |\hat{x}^{(\ell)}_{i,t} - \mu_{i} |\notag\,.
\end{align}
Equation~\eqref{eqn_fakeness_measure} follows from the following two arguments:
\begin{itemize}
    \item $\left|\sum_{g=1}^{\ell-1}r_{i,t}^{(g)}\eta_{i,t}^{(g)} -  \sum_{g=1}^{\ell-1}\sum_{t' \in \tau^{(g)}_{i,t}}X_{i,t'} \right| \leq F$ because the maximum number of fake samples (arising from fake users) over all the rounds in $(\tau^{(g)}_{i,t})_{g=1}^{\ell-1}$ is at most $F$.
    
    \item $|r_{i,t}^{(\ell)}\eta_{i,t}^{(\ell)} - \sum_{t' \in \tau^{(\ell)}_{i,t}}X_{i,t'}  |  \leq \log(\frac{2\level}{\delta})+3$ as per Corollary~\ref{corr_maxfake}, stated in Section \ref{sec:corollary}, and conditioned on $\mathcal{G}_2$.
\end{itemize}
Finally Equation ~\eqref{eqn_max_fake} is a consequence of the fact that $F \leq 2^{\ell}$ since $\ell \geq \ell^{\star} = \ceil*{\log_2(F)}$. We are now ready to complete the proof. Going back to the statement of this claim and applying~\eqref{eqn_hoeffding_true}, we have that: 
\begin{align*}
\P\left(\mathcal{G}^{(\ell)}_i ~|~ \mathcal{G}_2\right) & =  1- \P\left(\left| \hat{r}^{(\ell)}_{i,t} - \mu_{i} \right| > w^{(\ell)}_{i,t} ~|~\mathcal{G}_2\right)\\ & \geq 1-\P\left(\left|  \hat{x}^{(\ell)}_{i,t} - \mu_{i} \right| + \frac{\log(\frac{2\level}{\delta})+4}{{\hat \eta}^{(\ell)}_{i,t}} > w^{(\ell)}_{i,t}~|~\mathcal{G}_2\right) \\
    & = 1-\P\left(\left|  \hat{x}^{(\ell)}_{i,t} - \mu_{i} \right|   > w^{(\ell)}_{i,t} - \frac{\log(\frac{2\level}{\delta})+4}{{\hat \eta}^{(\ell)}_{i,t}}~|~\mathcal{G}_2\right) \\
        & \geq 1-\frac{1}{\P(\mathcal{G}_2)}\P\left(\left|  \hat{x}^{(\ell)}_{i,t} - \mu_{i} \right|   > w^{(\ell)}_{i,t} - \frac{\log(\frac{2\level}{\delta})+4}{{\hat \eta}^{(\ell)}_{i,t}}\right) \\
    & \geq 1-\frac{1}{\P(\mathcal{G}_2)}\frac{\delta^3}{8n^3T}. \hfill \Halmos
\end{align*}

\begin{lemma}
\label{lem_hoeffdingweighted_mean2}
Fix a positive integer $\ell > 0$, and $\eta_1, \eta_2 \in \{0,1,\ldots,T\}$ and define $\hat{\eta} = \eta_1 + \frac{\eta_2}{2^{\ell}}$. Suppose that $X_1, X_2, \ldots, X_T$ and $Y_1, Y_2, \ldots, Y_T$ are two sequences of $\{0, 1\}$-valued i.i.d. random variables with mean $\mu$. Given that $\hat{\eta}> 0$, for any $\epsilon > 0$, we have that:

$$\P\left(\left|\frac{1}{\hat\eta}\left(\sum_{i=1}^{\eta_1} X_i + \sum_{i=1}^{\eta_2}\frac{Y_i}{2^{\ell}}\right) - \mu \right| > \epsilon\right) \leq 2\exp(-2\eps^2\hat{\eta})\,.$$
\end{lemma}
\begin{proof}{Proof of Lemma \ref{lem_hoeffdingweighted_mean2}}
Note that since $\hat\eta > 0$, it must be the case that at least one of $\eta_1$ or $\eta_2$ is non-zero. Since $\E[X_i] = \E[Y_i] = \mu$ (for all $i$), we have that 

$$\E\left[\sum_{i=1}^{\eta_1}X_i + \sum_{i=1}^{\eta_2} \frac{Y_i}{2^{\ell}}\right] = \hat{\eta}\mu\,.$$

By applying Hoeffding's inequality (Lemma \ref{lem_hoeffding_sum}), it follows that

$$\P\left(\left|\frac{1}{\hat{\eta}}\left(\sum_{i=1}^{\eta_1}X_i + \sum_{i=1}^{\eta_2}\frac{Y_i}{2^{\ell}}\right) - \mu \right| > \epsilon\right) \leq 2\exp\left(-\frac{2\hat{\eta}^2\eps^2}{\eta_1 + \eta_2\cdot 2^{-2\ell}}\right) = 2\exp\left(-2\hat{\eta}\eps^2\right).\Halmos$$

\end{proof}

\begin{lemma}{(Hoeffding's Sum Inequality~\citep{hoeffding1994probability})} 
\label{lem_hoeffding_sum}
Suppose that $X_1, X_2, \ldots, X_m$ are independent random variables such that for all $i$, $a_i \leq X_i \leq b_i$, and let $S_m = \sum_{i=1}^m X_i$. Then, we have that: 

$$\P(|S_n - \E[S_n]| > t) \leq 2\exp\left( - \frac{2t^2}{\sum_{i=1}^m (b_i - a_i)^2} \right).\Halmos$$
\end{lemma}

\subsection{Proof of Lemma \ref{lem:parts1}}

We first present two auxiliary lemmas that provide (a)  properties of the product-ordering graphs under  the good event $\eventforc$, and (b) the necessary and sufficient conditions  to mistakenly place product $i$ in position $j<i$ under event $\eventforc$. These two lemmas allow us to  bound $\E\left[\sum_{t=1}^T \ind(\event_{r,t}(\pi_t(\p{j}) = i) \cap \ell_t \ge  \ell^{\star})~|~ \eventforc \right]$.

Let us start by a simple observation. Recall from Equation \eqref{eqn_gamma_agnostic} that for any $i>j$, $\gamma_{j,i}$ is defined as: 
\begin{equation}
    \gamma_{j,i} = \frac{64\log(4nT/\delta)}{\Delta^2_{j,i}}\,.
    \end{equation}  Observe that for any $i > j$ and level $\ell$ if $\hat{\eta}^{(\ell)}_{i,t} \geq \gamma_{j,i}$, then by definition of  $w^{(\ell)}_{i,t}$ in Equation \eqref{eqn_windowsize_agnostic}, we have that: 
    \begin{align*}
    w^{(\ell)}_{i,t} & \leq \sqrt{\frac{3\Delta^2_{j,i}}{2\cdot64}} + \Delta^2_{j,i}  \frac{\log(\frac{2\log_2(T)}{\delta})+4}{64\log(4nT/\delta)} \\ & \leq \frac{\Delta_{j,i}}{8}\sqrt{\frac{3}{2}} + \Delta_{j,i}  \frac{\log(\frac{4nT}{\delta})+4}{64\log(4nT/\delta)} \\
    & \leq \frac{\Delta_{j,i}}{8}\sqrt{\frac{3}{2}} + \Delta_{j,i}  \frac{4\log(\frac{4nT}{\delta})}{64\log(4nT/\delta)} \\
    & = \frac{\Delta_{j,i}}{8}\sqrt{\frac{3}{2}} + \frac{\Delta_{j,i}}{16} \\
    & \leq \frac{\Delta_{j,i}}{4}. 
    \end{align*}
 In the third inequality, we used the simple fact that $4 \leq 3\log(4) \leq 3\log(\frac{4nT}{\delta})$ since $n, T \geq 1$ and $\delta \leq 1$. The final inequality follows from simple algebra, namely that: $\frac{1}{8}\sqrt{\frac{3}{2}} + \frac{1}{16} \leq \frac{1}{4}$.
 
 By symmetry, the above argument is also true for $w^{(\ell)}_{j,t}$ when $\hat{\eta}^{(\ell)}_{j,t} \geq \gamma_{j,i}$.  In other words, once we receive feedback on product $i$ (product $j$) at least $\gamma_{j,i}$ times, its confidence interval becomes smaller than its gap to product $j$ (product $i$). This allows our algorithm to correctly add an edge $(i,j)$ to the product-ordering graph $G^{(\ell)}$, which we will formally prove below.

\begin{lemma}[Properties of Graph $G$ under Event $\eventforc$]
\label{lem_pairwise_edgeadd_agnostic}
Assume that event $\eventforc$ holds. Then, for any two products $i, j$ with $i > j$ and level $\ell \geq \ell^{\star}$:
\begin{enumerate}
    \item In any round $t$, the product-ordering graph $G^{(\ell)}$ does not contain an incorrect edge from $j$ to $i$.
    
    \item  Suppose that in round $t$, $\hat{\eta}^{(\ell)}_{i,t}, \hat \eta^{(\ell)}_{j,t} \geq \gamma_{j,i}$. Then, the product-ordering graph  $G^{(\ell)}$ contains a correct outgoing edge from $i$ to $j$ in round $t'\ge t$. Here, $\gamma_{j,i}$ is defined in Equation \eqref{eqn_gamma_agnostic}. 
\end{enumerate}

\end{lemma}

The proof of Lemma~\ref{lem_pairwise_edgeadd_agnostic} is identical to that of Lemma~\ref{lem_pairwise_edgeadd} and we do not repeat the same steps here.  Lemma \ref{lem_pairwise_edgeadd_agnostic} shows that under the good event $\eventforc$, the product-ordering graphs for levels larger than or equal to $\ell^{\star}$ do not contain erroneous edges. Further, when the amount of feedback for  products $i$ and $j$ is large enough, there is a correct edge between their corresponding nodes in graph $G^{(\ell)}$. As a consequence of this lemma, we infer that with high probability, for any $\ell \geq \ell^{\star}$, the graph $G^{(\ell)}$ is never eliminated by our algorithm as this graph cannot contain any cycles.

This lemma allows us to show the following statement, which says under event $\eventforc$, any mistake in the ranking compared to the optimal ordering must be due to the lack of sufficient feedback.   

\begin{lemma}[Under Event $\eventforc$, Any Mistake is Due to Insufficient Feedback]
	\label{lem_dagprop_agnostic}
	Suppose that the \Call{GraphRankSelect}{} algorithm ranks product $i$ ahead of product $j$ in round $t+1$ for some $i > j$ and level $\ell \geq \ell^{\star}$ Conditional on event $\eventforc$, (i) there must then exist some $k \leq j$ such that there is no edge between $i$ and $k$ in the product-ordering graph  $G^{(\ell)}$ and $\eta^{(\ell)}_{i,t} \leq \eta^{(\ell)}_{k,t}$, and (ii) $\hat\eta^{(\ell)}_{i,t} < \gamma_{k,i} \leq \gamma_{j,i}$, where $\gamma_{k,i}, \gamma_{j,i}$ are as defined in Equation  \eqref{eqn_gamma_agnostic}. 
\end{lemma}

Once again, the proof is rather identical to that of Lemma~\ref{lem_dagprop} and hence, it is omitted.

Having presented these lemmas, we proceed to bounding $\E\left[\sum_{t=1}^T \ind(\event_{r,t}(\pi_t(\p{j}) = i) \cap \ell_t = \ell)~|~ \eventforc \right]$ for any level $\ell \geq \ell^{\star}$ and $i > j$. 

Suppose that $\pi_t(j) = i$ in some round $t$ where $\ell_t = \ell$ and that we receive feedback on product $i$ from a real customer, i.e., $\event_{r,t}(\pi_t(\p{j}) = i) = 1$. Since $i > j$,  there must exist at least one product better than $i$ that is ranked below position $j$. Mathematically, we can write this as:
\begin{equation}
\label{eqn_ordercatch_agnostic}
     \pi_t(j) = i; ~~ i > j \implies \exists r \text{~~with~~} r \leq j <i \text{~~s.t.~~} \pi^{-1}_t(r) > j\,.
\end{equation}

    Since product $i$ is ranked above a better product $r$ in round $t$, we can apply Lemma~\ref{lem_dagprop_agnostic} with $i$ and $r$. Upon application of the lemma, we infer that $\hat{\eta}^{(\ell)}_{i,t-1} < \gamma_{r,i}$. Further, since $\Delta_{r,i} \geq \Delta_{j,i}$, we also have
    \begin{equation}
    \hat{\eta}^{(\ell)}_{i,t-1} < \gamma_{r,i} \leq \gamma_{j,i}\,.
    \label{eqn_eta_agnostic}
    \end{equation}

Thus far, we have  shown that conditional on event $\eventforc$, when  $\event_{r,t}(\pi_t(\p{j}) = i) = 1$ at some arbitrary round $t$ with $\ell_t = \ell$, we  must have $\hat{\eta}^{(\ell)}_{i,t-1} < \gamma_{j,i}$, where $i>j$. Conditioned on event $\eventforc$, consider any arbitrary instantiation of our algorithm, and let $t^{(\ell)}_{i, j}$ denote the first round in which $\eta^{(\ell)}_{i,t} = \gamma_{j,i}$. For any $t > t^{(\ell)}_{i, j}$ such that $\ell_t = \ell \geq \ell^{\star}$, we claim that our algorithm would never select product $i$ at position $j$.  To see why, assume by contradiction that  $\pi_t(j) = i$ for $t > t^{(\ell)}_{i, j}$ and $\ell_t = \ell$. Then, as per Equation~\eqref{eqn_eta_agnostic}, it must be the case that $\hat{\eta}^{(\ell)}_{i,t-1} < \gamma_{j,i}$. However, this would be a contradiction since
\begin{align*}
    \hat{\eta}^{(\ell)}_{i,t-1}  \geq \eta^{(\ell)}_{i,t-1} 
    & \geq \eta^{(\ell)}_{i,t^{(\ell)}_{i, j}} 
     = \gamma_{j,i}.
\end{align*}
Therefore, conditioned on $\eventforc$, $\event_{r,t}(\pi_t(\p{j}) = i) = 0$ for all $t > t^{(\ell)}_{i, j}$. Leveraging this, we have that:
\begin{align*}
   \E\left[\sum_{t=1}^T \ind(\event_{r,t}(\pi_t(\p{j}) = i) \cap ~(\ell_t = \ell))~|~ \eventforc \right] & = \E\left[\sum_{t=1}^{t^{(\ell)}_{i, j}} \ind(\event_{r,t}(\pi_t(\p{j}) = i) ~\cap ~(\ell_t = \ell))~|~ \eventforc \right]\\
    & \leq \E\left[\sum_{t=1}^{t^{(\ell)}_{i, j}}\E\big[\eta^{(\ell)}_{i,t} - \eta^{(\ell)}_{i,t-1} ~|~\eventforc \big] \right]  = \E\left[\eta^{(\ell)}_{i,t^{(\ell)}_{i, j}}~|~\eventforc\right]  = \gamma_{j,i}\,.
\end{align*}

We now bound the total number of times the event $\event_{r,t}(\pi_t(j) = i)$ occurs at all levels $\ell \geq \ell^{\star}$ as follows: 
\begin{align}\label{eq:large_ell}
\begin{split}
   \E\left[\sum_{t=1}^T \ind(\event_{r,t}(\pi_t(\p{j}) = i) ~\cap~ (\ell_t \geq \ell^{\star}))~|~ \eventforc \right] & = \sum_{\ell=\ell^{\star}}^{\level}\E\left[\sum_{t=1}^{T} \ind(\event_{r,t}(\pi_t(\p{j}) = i) ~\cap~ (\ell_t = \ell))~|~ \eventforc \right]\\
    & \leq \sum_{\ell=\ell^{\star}}^{\level} \gamma_{j,i} \leq \level\gamma_{j,i}. \Halmos
    \end{split}
\end{align}

\subsection{Proof of Lemma \ref{lem:tjmax}} 
\label{sec:proof:tjmax}

We prove the lemma by establishing two inductive claims which we will state shortly. Before that, we introduce a series of events which prove helpful in our analysis. Specifically, for $t \in [T]$, we define  
\begin{align}
\eventforc_{t} \triangleq \bigcap_{\substack{i \in [n] , t' \in [t] , \ell \geq \ell^{\star}}}\eventforc^{(\ell)}_{i,t'}
\label{eq:event:t}
\end{align}
In words, $\eventforc_{t}$ 
is  the event that the individual events $\eventforc^{(\ell)}_{i,t'}$ as defined in Equation \eqref{eqn_armwidth_agnostic_2} hold for all levels $\ell \geq \ell^{\star}$, all products $i \in [n]$ and rounds up to $t$.  Note that $\eventforc \subseteq \eventforc_{t}$, for all $t \in [T]$. With this definition, we are ready to state our two inductive claims: \\
---\textbf{Inductive Claim 1.} $ \E[\sum_{t=1}^{\tjmax}\ind(\eventz_{i,t} ~\cap~ (\ell_t < \ell^{\star}))~|~\eventforc] \leq jF\left(8n\gamma_{j,i} +9\gamma_{j,i} + 2T\frac{\delta}{\P(\eventforc)}\right),$ where  $\tjmax$ and $\gamma_{j,i}$ are respectively  defined in Equations \eqref{eq:tjmax:def} and \eqref{eqn_gamma_agnostic}.
\\
---\textbf{Inductive Claim 2.} Conditioned on $\eventforc_{\tjmax}$ (as defined in Equation \eqref{eq:event:t}), at the end of round $\tjmax$, product $i$ has at least $j$ outgoing edges in graph $G^{(\ell^{\star})}$ for any instantiation. 

Recall from our definitions in Section~\ref{sec:FORC} that $\tjmax$ denotes the earliest round in which both the milestone events (Definitions~\ref{defn_firstmilestone},~\ref{defn_secondmilestone}) are achieved. Therefore, conditioning on $\eventforc_{\tjmax}$ allows us to focus on histories where $\hat{r}^{(\ell)}_{i,t} - w^{(\ell)}_{i,t} \leq \mu_i \leq \hat{r}^{(\ell)}_{i,t} + w^{(\ell)}_{i,t}$ is valid for all products $i \in [n]$ and levels $\ell \geq L$ up to the round $\tjmax$ in which both milestones are achieved. Note that as a consequence of the second inductive claim, set $S_j$ is formed by the end of round $\tjmax$, where $S_j \triangleq \{i_1, i_2, \ldots, i_{j}\}$ is the set of the first $j$ products (chronologically) that product $i$ {forms}  outgoing edges to in product-ranking graph $G^{(\ell^{\star})}$.

We prove both claims by induction on $j$ for a given product $i$. The base case follows trivially for $j=0$. Note  that $\tjmaxzero=0$. Further, because $\eventforc \subseteq \eventforc_{\tjmax}$, proving the second inductive claim also proves the second part of the lemma.

\subsubsection*{Proof of 
Inductive Claim 1.} For the first inductive claim, we have that: 
\begin{align}
      \E\Big[\sum_{t=1}^{\tjmax}\ind(\eventz_{i,t} ~\cap~ (\ell_t < \ell^{\star}))~|~\eventforc\Big] & =  \E\Big[\sum_{t=1}^{\tjmaxminus{}}\ind(\eventz_{i,t} ~\cap~ (\ell_t < \ell^{\star}))~|~\eventforc\Big] + \E\Big[\sum_{t=\tjmaxminus{}+1}^{\tjmax}\ind(\eventz_{i,t} ~\cap~ (\ell_t < \ell^{\star}))~|~\eventforc\Big] \nonumber\\
     \notag
      & \overset{(a)}{\leq} \E\Big[\sum_{t=1}^{\tjmaxminus{}}\ind(\eventz_{i,t} ~\cap~ (\ell_t < \ell^{\star}))~|~\eventforc\Big] + \E\Big[\sum_{t=\tjmaxminus{}+1}^{\tj}\ind(\eventz_{i,t} ~\cap~ (\ell_t < \ell^{\star}))~|~\eventforc\Big] \\ \notag
      &\quad + \E\Big[\sum_{t=\tjmaxminus{}+1}^{\tdelta}\ind(\eventz_{i,t} ~\cap ~(\ell_t < \ell^{\star}))~|~\eventforc\Big] 
      \\
      \begin{split}
      &\leq \E\Big[\sum_{t=1}^{\tjmaxminus{}}\ind(\eventz_{i,t} ~\cap ~(\ell_t < \ell^{\star}))~|~\eventforc\Big] + \E\Big[\sum_{t=1}^{\tj}\ind(\eventz_{i,t} ~\cap~ (\ell_t < \ell^{\star}))~|~\eventforc\Big] \\
      &\quad+ \E\Big[\sum_{t=\tjmaxminus{}+1}^{\tdelta}\ind(\eventz_{i,t} \cap ~(\ell_t < \ell^{\star}))~|~\eventforc\Big]\label{eqn_subplays}. 
      \end{split}
\end{align}
 Inequality~(a) comes from the fact that $\tjmax =  \max\{\tdelta, \tj\}$. Having decomposed the regret into these three terms, we focus our efforts on bounding each of these terms separately.
 
 \textbf{Bounding the First Term.} The first term, i.e., $\E\Big[\sum_{t=1}^{\tjmaxminus{}}\ind(\eventz_{i,t} ~\cap~ (\ell_t < \ell^{\star}))~|~\eventforc\Big]$ can be bounded directly by applying the inductive claim, which gives us: 
 \begin{align*}
  \E\Big[\sum_{t=1}^{\tjmaxminus{}}\ind(\eventz_{i,t} ~\cap~ (\ell_t < \ell^{\star}))~|~\eventforc\Big] & \leq  (j-1)F\left(8n\gamma_{j-1,i} +9\gamma_{j-1,i} + 2T\frac{\delta}{\P(\eventforc)}\right)\\
  & \leq  (j-1)F\left(8n\gamma_{j,i} +9\gamma_{j,i} + 2T\frac{\delta}{\P(\eventforc)}\right).
 \end{align*}
  Note that {the last inequality holds because} $\gamma_{j-1,i} \leq \gamma_{j,i}$ according to our definition in Equation \eqref{eqn_gamma_agnostic}.

\textbf{Bounding the Second Term.} We show that the second term, i.e., $\E\Big[\sum_{t=1}^{\tj}\ind(\eventz_{i,t} ~\cap~ (\ell_t < \ell^{\star}))~|~\eventforc\Big]$,  is at most $9F\gamma_{\p{j},i}$. Recall that $\tj$  is defined   as the  smallest round $t$ at which $\hat{\eta}^{(\ell^{\star})}_{i,t} \geq 4\gamma_{\p{j},i}$, and  $\eventz_{i,t}$ is defined as the event that (a) the customer in round $t$ is real, and (b) we receive feedback on product $i$ in round $t$. Therefore, we have
\begin{align}
    \notag \E\Big[\sum_{t=1}^{\tj} \ind(\eventz_{i,t} ~\cap ~(\ell_t < \ell^{\star}))~|~\eventforc\Big]
    & \leq \E\Big[\sum_{\ell=1}^{\ell^{\star}-1}\eta^{(\ell)}_{i,\tj}~|~\eventforc\Big] \\ & \leq \E[2^{\ell^{\star}} \hat{\eta}^{(\ell^{\star})}_{i,\tj}~|~\eventforc] \notag\\
    & \leq 2F  \E[\hat{\eta}^{(\ell^{\star})}_{i,\tj}~|~\eventforc] \notag\\
    & \leq 8F \gamma_{\p{j},i} + 2F\notag\\
    & \leq 9F \gamma_{\p{j},i}\,,\notag 
\end{align}
where the first inequality follows from definition of event $\eventz_{i,t}$ and the fact that $\ell_t< \ell^{\star}$. 
The second inequality crucially relies on the upward cross-learning: it follows because by Equation \eqref{eq:eta_cr},  we have $\hat{\eta}^{(\ell^{\star})}_{i,t} \geq \frac{1}{2^{\ell^{\star}}}\sum_{\ell=1}^{\ell^{\star}-1}\eta^{(\ell)}_{i,t}$. The third expression holds because by definition of $\ell^{\star}$, we have $2^{\ell^{\star}} \leq 2F$.
The fourth inequality follows from definition of $\tj$:
 for any instantiation, it must be the case that:
$\hat{\eta}^{(\ell^{\star})}_{i,\tj-1} < 4\gamma_{\p{j},i}.$ Then, considering the fact that  $\hat{\eta}^{(\ell^{\star})}_{i,t}$ can grow in increments of at most one, we have $\hat{\eta}^{(\ell^{\star})}_{i,\tj} \leq 4\gamma_{\p{j},i}+1$; see  Equation \eqref{eq:eta_cr}. 
We note that  even if in some round $t$, graph $G^{(\ell_t)}$ is eliminated, we still update $\eta_i^{(\ell_t)}$ and $\hat \eta_i^{(\ell_t)}$.  The final inequality comes from the fact that $\gamma_{\p{j},i} \geq 2$ for all $i,j$. \medskip

\textbf{Bounding the Third Term.} Bounding the third term, i.e., $\E\Big[\sum_{t=\tjmaxminus{}+1}^{\tdelta}\ind(\eventz_{i,t} ~\cap~ (\ell_t < \ell^{\star}))~|~\eventforc\Big]$, is the most challenging part of our proof. The difficulty here stems from the fact that graphs $G^{(\ell)}$ corresponding to lower levels could contain incorrect edges involving product $i$ due to fake clicks. As a result of such edges, the algorithm may place this product at a disproportionately high rank whenever $\ell < \ell^{\star}$. This increased visibility implies that we may receive feedback on product $i$ a large number of times (i.e., $\eventz_{i,t}$ is large) at lower levels when compared to level $\ell^{\star}$ that does not contain any incorrect edge. 

To control for the occurrence of this scenario, we have to show that every time $\eventz_{i,t}$ is true, some progress is made at level $\ell^{\star}$ towards adding a correct outgoing edge from $i$, which we know will be transferred to lower levels via downward cross-learning. Our main technique here is a mapping that connects receiving feedback on product $i$ at a lower level to receiving feedback on a product $\sigma \notin S_{j-1}$ at level $\ell^{\star}$.

The third term can be expanded as follows
\begin{align}
\E\Big[\sum_{t=\tjmaxminus{}+1}^{\tdelta}\ind(\eventz_{i,t} ~\cap~ \ell_t < \ell^{\star}) ~|~\eventforc\Big] & = \sum_{t=1}^T \E\left[\ind\big(\eventz_{i,t} ~\cap~ (\ell_t < \ell^{\star}) ~\cap~ (\tjmaxminus{} < t \leq \tdelta) \big) ~|~\eventforc\right]\notag\\ 
& \leq \sum_{t=1}^T\frac{\P(\eventforc_{t-1})}{\P(\eventforc)}\E\left[\ind\big(\eventz_{i,t} ~\cap~ (\ell_t < \ell^{\star}) ~\cap~ (\tjmaxminus{} < t \leq \tdelta) \big) ~|~\eventforc_{t-1}\right]\,,\label{eqn_conditionalchange_final} 
\end{align}
where $\eventforc_{t}$ is defined in Equation \eqref{eq:event:t}.
To see why the inequity holds, consider a generic event $Y$. Then, \[\E [ \ind(Y) | \eventforc] = \frac{\P(Y \cap \eventforc)}{\P(\eventforc)} \leq \frac{\P(Y \cap \eventforc_t)}{\P(\eventforc)}
      = \frac{\P(Y \cap \eventforc_t)}{\frac{\P(\eventforc_t)}{\P(\eventforc_t)}\P(\eventforc)}  =\frac{\P(\eventforc_t)}{\P(\eventforc)}\E [ \ind(Y) | \eventforc_t].\]

Our next lemma allows for a transformation from the event that $\ell_t < \ell^{\star}$ to the event that $\ell_t = \ell^{\star}$, which we subsequently utilize to bound the third term. For any given round $t > \tjmaxminus{}$, define the random variable $\sigma_t$ as the highest ranked product outside of $S_{j-1}$ in $\pi_t$, i.e., $\pi_t^{-1}(\sigma_t) \leq \pi_t^{-1}(k)$ for all $k \notin S_{j-1}$.  Recall that $S_j$ is the set of the first $j$ products (chronologically) that $i$ has outgoing edges to in $G^{(\ell^{\star})}$. Note that by our {second inductive} assumption, by round $\tjmaxminus{}$, set $S_{j-1}$ is already formed. 

\begin{lemma}
\label{lem_agnostic_conditional}
Suppose that $\zeta_t$ denotes the event that $\tjmaxminus{} < t \leq \tdelta{}$.  Assuming that the inductive claims hold up to $j-1$, for the given product $i$ and round $t$, we have that \footnote{Here, event $\ind(\eventz_{\sigma_t,t})=0$ when in some round $t$, product $i$ has fewer than $j-1$ outgoing edges. Nevertheless, by our induction assumption, when event $\zeta_t$ holds, i.e., $\tjmaxminus{} < t \leq \tdelta{}$, set $S_{j-1}$ is already formed and hence, $\sigma_t$ and $\eventz_{\sigma_t,t}$ are well-defined.} 
$$\E\left[\ind(\eventz_{i,t} ~\cap~ (\ell_t < \ell^{\star})~ \cap ~ \zeta_t) ~|~\eventforc_{t-1}\right] \leq  2F \E\left[\ind(\eventz_{\sigma_t,t} ~\cap~ (\ell_t = \ell^{\star})~ \cap~ \zeta_t) ~|~\eventforc_{t-1}\right].$$

\end{lemma}
 In particular, the lemma claims that in a round $t$ with  $\ell_t<\ell^{\star}$, the probability that we receive feedback on product $i$ from a real customer is smaller than $2F$ times the probability that we receive feedback on at least one product ($\sigma_t$) that does not belong to $S_{j-1}$ when $\ell_t= \ell^{\star}$. The proof of the lemma utilizes the fact that $ \P(\ell_t < \ell^{\star})$ is independent of the round $t$ or the history, and is deferred to the end. Applying Lemma~\ref{lem_agnostic_conditional} to the term in Equation \eqref{eqn_conditionalchange_final}, 
 we get
\begin{align}
\sum_{t=1}^T\frac{\P(\eventforc_{t-1})}{\P(\eventforc)} \E\left[\ind(\eventz_{i,t} ~\cap ~(\ell_t < \ell^{\star}) ~\cap~ \zeta_t) ~|~\eventforc_{t-1}\right] 
    & \leq 2F\sum_{t=1}^T \frac{\P(\eventforc_{t-1})}{\P(\eventforc)} \E\left[\ind(\eventz_{\sigma_t,t} ~\cap~ (\ell_t = \ell^{\star}) ~ \cap~ \zeta_t) ~|~\eventforc_{t-1}\right].
    \label{eq:zeta}
\end{align}

Ignoring the $\frac{\P(\eventforc_{t-1})}{\P(\eventforc)}$ term for the moment, the expression inside the summation in the right hand side of the above equation counts the number of rounds within the interval $(\tjmaxminus{}, \tdelta{}]$ in which we receive feedback on some product $\sigma_t \notin S_{j-1}$ when our algorithm selects level $\ell^{\star}$. Recall that $\zeta_t$ denotes the event that $\tjmaxminus{} < t \leq \tdelta{}$. The right hand side of the above equation can be  upper bounded as follows
 \[2F\sum_{t=1}^T \frac{\P(\eventforc_{t-1})}{\P(\eventforc)} \E\left[\ind(\eventz_{\sigma_t,t} ~\cap~ (\ell_t = \ell^{\star}) ~ \cap~ \zeta_t) ~|~\eventforc_{t-1}\right] \leq 2F \sum_{t=1}^T \left(\E\left[\ind(\eventz_{\sigma_t,t} \cap (\ell_t = \ell^{\star}) \cap \zeta_t) ~|~\eventforc\right] + \frac{\P(\eventforc^c)}{\P(\eventforc)}\right)\,. \]
 To see why the inequality holds, consider a generic random variable  $Y\le 1$. Then,
 $$\E[Y | \eventforc_{t-1}] = \E[Y | \eventforc]\P(\eventforc | \eventforc_{t-1}) + \E[Y | \eventforc^c \cap \eventforc_{t-1}]\P(\eventforc^c | \eventforc_{t-1}) \leq \E[Y | \eventforc]\frac{\P(\eventforc)}{\P(\eventforc_{t-1})} + \frac{\P(\eventforc^c)}{\P(\eventforc_{t-1})}$$
 In the above inequality, we used the fact that 
 {$ \eventforc \cap \eventforc_{t-1} = \eventforc$}{, $\E[Y | \eventforc^c \cap \eventforc_{t-1}] \leq 1$, } and $\P(\eventforc_{t-1} \cap \eventforc^c) \leq \P(\eventforc^c)$. 
 
Next, we focus on bounding the term $\sum_{t=1}^T \E\left[\ind(\eventz_{\sigma_t,t} ~\cap~ (\ell_t = \ell^{\star})~ \cap ~ \zeta_t) ~|~\eventforc\right]$ since $\frac{\P(\eventforc^c)}{\P(\eventforc_{t-1})}$ is a small constant, which  by Lemma \ref{lem:parts3}, is at most by $\frac{\delta}{\P(\eventforc)}$. 
We start with the following lemma.

\begin{lemma} \label{lem:sigma_t}
   Assume that the inductive claims hold up to $j-1$ and event $\eventforc$ holds. Consider any round $t \in (\tjmaxminus{}, \tdelta]$, and recall that $\sigma_t$ is the highest ranked product outside of $S_{j-1}$ in $\pi_t$, i.e., $\pi_t^{-1}(\sigma_t) \leq \pi_t^{-1}(k)$ for all $k \notin S_{j-1}$. Then,  $\eta^{(\ell^{\star})}_{\sigma_t,t-1} < 4\gamma_{j,i}$, where $\gamma_{j,i}$ is defined in Equation \eqref{eqn_gamma_agnostic}.
\end{lemma}

Now, we are ready to bound the term $ \sum_{t=1}^T \E\left[\ind(\eventz_{\sigma_t,t} ~\cap~ (\ell_t = \ell^{\star}) ~\cap~ \zeta_t) ~|~\eventforc\right]$. 

\begin{align}
\begin{split}
\sum_{t=1}^T \E\left[\ind(\eventz_{\sigma_t,t} ~\cap~ (\ell_t = \ell^{\star})~ \cap~ \zeta_t) ~|~\eventforc\right] & =  \E\Big[\sum_{t=\tjmaxminus{}+1}^{\tdelta}\ind(\eventz_{\sigma_t,t} ~\cap~ (\ell_t = \ell^{\star}) ) ~|~\eventforc\Big]
\\&\leq \E\big[\sum_{\tjmaxminus{}+1}^{\tdelta}(\eta^{(\ell^{\star})}_{\sigma_t,t} - \eta^{(\ell^{\star})}_{\sigma_t,t-1}) \ind(\eta^{(\ell^{\star})}_{\sigma_t,t} \leq 4\gamma_{j,i})  ~|~\eventforc\big] \\
& \leq  \E\big[\sum_{\tjmaxminus{}+1}^{\tdelta}\sum_{\sigma \notin S_{j-1}}(\eta^{(\ell^{\star})}_{\sigma,t} - \eta^{(\ell^{\star})}_{\sigma,t-1}) \ind(\eta^{(\ell^{\star})}_{\sigma,t} \leq 4\gamma_{j,i}) ~|~\eventforc\big] \\
& \leq \E\big[\sum_{\sigma \notin S_{j-1}}4\gamma_{j,i}~|~ \eventforc \big]\\
& \leq 4(n-j+1)\gamma_{\p{j},i} \leq 4n\gamma_{\p{j},i}\,. 
\end{split}
\label{eq:bound_zeta}
\end{align}
where the first inequality holds because  (a) in any round  $t \in (\tjmaxminus{}, \tdelta]$ where $\ell_t = \ell^{\star}$ and we receive feedback on some product $\sigma_t \notin S_{j-1}$,   $\eta^{(\ell^{\star})}_{\sigma_t}$ is increased by one, and (b) by Lemma \ref{lem:sigma_t}, $\eta^{(\ell^{\star})}_{\sigma_t,t-1} < 4\gamma_{j,i}$, and hence, $\eta^{(\ell^{\star})}_{\sigma_t,t} \leq 4\gamma_{j,i}$. 
{The second inequality holds because we are adding non-negative terms by summing over all feasible candidates for $\sigma_t$.}
The fourth inequality follows from the fact that 
 we have at most $n-j+1$ products (feasible candidates for $\sigma_t$) outside of $S_{j-1}$ and we are counting the number of rounds in which we receive feedback on one of these products. 
 
 Putting Equations \eqref{eq:bound_zeta}, \eqref{eq:zeta}, and \eqref{eqn_conditionalchange_final} together, we can bound the third term, i.e., $\E[\sum_{t=\tjmaxminus{}+1}^{\tdelta}\ind(\eventz_{i,t} ~\cap~ (\ell_t < \ell^{\star}))~|~\eventforc]$, as follows 
\begin{align} 
      \E\Big[\sum_{t=\tjmaxminus{}+1}^{\tdelta}\ind(\eventz_{i,t} ~\cap~ (\ell_t < \ell^{\star})) ~|~\eventforc \Big]  & \leq 8nF\gamma_{\p{j},i} + 2FT\frac{\P(\eventforc^c)}{\P(\eventforc)} \leq 8nF\gamma_{\p{j},i} + 2FT\frac{\delta}{\P(\eventforc)}\,.
      \label{eqn_subplays_thirdterm}
  \end{align}
  where the second inequality follows from 
  {the observation that $\P(\eventforc_{t-1}) \geq \P(\eventforc)$ and}
  the proof of Lemma \ref{lem:parts3} where we show $\P(\eventforc) \ge 1-
  \delta$.
  
  \textbf{Combining the Bounds on the Three Terms.} Combining our upper bounds for the three terms in~\eqref{eqn_subplays}, we get that:
  \begin{align*}
        \E[\sum_{t=1}^{\tjmax}\ind(\eventz_{i,t} ~\cap ~ (\ell_t < \ell^{\star}))~|~\eventforc] & \leq \E[\sum_{t=1}^{\tjmaxminus{}}\ind(\eventz_{i,t} ~\cap~ (\ell_t < \ell^{\star}))~|~ \eventforc] + \E[\sum_{t=1}^{\tj}\ind(\eventz_{i,t} ~ \cap~ (\ell_t < \ell^{\star}))~|~ \eventforc] \\
        &+ \E[\sum_{t=\tjmaxminus{}+1}^{\tdelta}\ind(\eventz_{i,t} ~\cap~ (\ell_t < \ell^{\star}))~|~ \eventforc] \\
        & \leq (j-1)F\left(8n\gamma_{j,i} + 9\gamma_{j,i} + 2T\frac{\delta}{\P(\eventforc)}\right)+  9F\gamma_{j,i} + 8nF\gamma_{j,i} + 2FT\frac{\delta}{\P(\eventforc)} \\
        & \leq 8njF\gamma_{j,i} +9jF\gamma_{j,i} + 2jFT\frac{\delta}{\P(\eventforc)}\,,
      \end{align*}
      
 This completes our proof of the first inductive claim. 

\subsubsection*{Second Inductive Claim.} Here, we show  the second inductive claim, namely that conditioned on $\eventforc_{\tjmax}$, product $i$ has at least $j$ outgoing edges by round $\tjmax$ {in the product-ordering graph $G^{(\ell^{\star})}$}. 
By the inductive claim for $j-1$, we know that $i$ has at least $j-1$ outgoing edges by round $\tjmaxminus{} \leq \tjmax$ conditioned on $\eventforc_{\tjmaxminus{}}$. Since $\eventforc_{\tjmax} \subseteq \eventforc_{\tjmaxminus{}}$, 
it suffices to show that within the interval $(\tjmaxminus{},\tjmax]$, we add at least one more outgoing edge for product $i$ {in $G^{(\ell^{\star})}$}. 

Define $
\sigma \triangleq \arg\max_{k \in \G{\Delta_{j,i}} \setminus S_{j-1}}\left(\eta^{(\ell^{\star})}_{k,\tdelta}\right)$ as the product that (a) does not belong to set $S_{j-1}$, (b) has a reward gap of at least $\frac{\Delta_{j,i}}{2}$ to product $i$, and (c) by round $\tdelta$, receives the maximum amount of feedback under level $\ell^\star$.   We will show that $G^{(\ell^{\star})}$ contains a correct edge from $i$ to $\sigma$ at the end of round $\tjmax$. This follows from the following arguments:
\begin{enumerate}
    \item Since product $\sigma$ has the highest feedback count for level $\ell^{\star}$ among the products in $\G{\Delta_{j,i}} \setminus S_{j-1}$ in round $\tdelta \leq \tjmax{}$, we have that 
    $$\hat{\eta}^{(\ell^{\star})}_{\sigma,\tjmax} \geq \hat{\eta}^{(\ell^{\star})}_{\sigma,\tdelta} \geq \eta^{(\ell^{\star})}_{\sigma,\tdelta} \geq 4\gamma_{j,i}\,.$$
     The first inequality is due to $\tdelta \leq \tjmax{}$ and the monotonicity of the $\hat{\eta}_{i,t}$ values in $t$. The second inequality comes from the fact that for any product $k$, round $t$, and level $\ell$, we have that $\hat{\eta}^{(\ell)}_{k,t} \geq \eta^{(\ell)}_{k,t}$ as per {the definition of $\hat{\eta}^{(\ell)}_{k,t}$ given in  Equation \eqref{eq:eta_cr} in} Algorithm~\ref{alg_cascading_agnostic}. The final inequality is due to our definition of $\sigma$ and {the second milestone} $\tdelta$ from Definition~\ref{defn_secondmilestone}. In particular, $\sigma$ is defined to be the product with the largest value of $\eta^{(\ell^{\star})}_{k,\tdelta}$ in the set $\G{\Delta_{j,i}} \setminus S_{j-1}$. Of course from Definition~\ref{defn_secondmilestone}, we know that: 
\begin{align*}
    \max_{k \in \G{\Delta_{j,i}} \setminus S_{j-1}}\left\{\eta^{(\ell^{\star})}_{k,t}\right\} \geq 4\gamma_{\p{j},i}\,.
\end{align*}
    
        \item Since $\tjmax \geq \tj$, we also have that $\hat{\eta}^{(\ell^{\star})}_{i,\tjmax} \geq  \hat{\eta}^{(\ell^{\star})}_{i,\tj} \geq 4\gamma_{j,i}$. Here, the second inequality is due to our definition of the first milestone event $\tj$, i.e., Definition~\ref{defn_firstmilestone}.

    \item Since $\sigma \in  \G{\Delta_{j,i}}$, we that $\Delta_{\sigma,i} \geq \frac{\Delta_{j,i}}{2}.$
    
    \item This in turn implies that $\gamma_{\sigma,i} \leq 4\gamma_{j,i} \leq \hat{\eta}^{(\ell^{\star})}_{\sigma,\tjmax},\hat{\eta}^{(\ell^{\star})}_{i,\tjmax}$. Applying Lemma~\ref{lem_pairwise_edgeadd_agnostic}, we can infer the existence of an edge from $i$ to $\sigma$ at the end of round $\tjmax$.\footnote{We note that the result in 
    Lemma~\ref{lem_pairwise_edgeadd_agnostic} holds  under event $\eventforc$, not our event of interest  $\eventforc_{\tjmax}$. Nevertheless, we can obtain the same result under event $\eventforc_{\tjmax}$ with a few minor changes. } 
  
\end{enumerate}
This concludes the proof of the second inductive claim that $i$ has at least $j$ outgoing edges by round $\tjmax$. 
Noting that $\eventforc \subseteq \eventforc_{\tjmax}$ completes the proof of Lemma~\ref{lem:tjmax}. \hfill \Halmos

\subsection*{Proof of Lemma~\ref{lem_agnostic_conditional}}
Our goal here is to show that  \begin{align}\E\left[\ind(\eventz_{i,t} ~\cap~ (\ell_t < \ell^{\star}) ~\cap~ \zeta_t) ~|~\eventforc_{t-1}\right] \leq  2F \E\left[\ind(\eventz_{\sigma_t,t} ~\cap~ (\ell_t = \ell^{\star}) ~ \cap~ \zeta_t) ~|~\eventforc_{t-1}\right]\,,\label{eq:l_less_l_star}\end{align}  where $\zeta_t$  denotes the event that $\tjmaxminus{} < t \leq \tdelta$. For any $t\notin (\tjmaxminus{},  \tdelta]$, the bound holds trivially. Thus, we focus on some $t\in (\tjmaxminus{},  \tdelta]$ under which $\ind(\zeta_t)=1$, and remove $\zeta_t$ from all the expressions. In the following, we first derive an upper bound on the left hand side (l.h.s.) and a lower bound on the right hand side (r.h.s.) and show that the lower bound is greater than the upper bound. This will complete the proof. 

\textbf{Upper Bound on the l.h.s. of \eqref{eq:l_less_l_star}.} Fix a round $t\in (\tjmaxminus{}, \tdelta]$ and consider any arbitrary history $\mathcal{H}_{t-1}$ that satisfies the conditions characterized in event $\eventforc_{t-1}$.
 By the law of iterated expectations, we have
\begin{align*}
\E\left[\ind(\eventz_{i,t} ~\cap~ (\ell_t < \ell^{\star})) ~|~\eventforc_{t-1}\right] & = \E\left[\E[\ind(\eventz_{i,t} ~\cap~ (\ell_t < \ell^{\star})) ~|~\mathcal{H}_{t-1} \cap \eventforc_{t-1}]~|~\eventforc_{t-1}\right] \\
&=\E\left[\E[\ind(\eventz_{i,t} ~\cap~ (\ell_t < \ell^{\star})) ~|~\mathcal{H}_{t-1}]~|~\eventforc_{t-1}\right]\,,
\end{align*}
where the last equation holds because we consider a history $\mathcal{H}_{t-1}$ that satisfies the conditions characterized in event $\eventforc_{t-1}$. 
Next, we bound the inner expectation, i.e., $\E[\ind(\eventz_{i,t} \cap~ (\ell_t < \ell^{\star})~ \cap \zeta_t) ~|~\mathcal{H}_{t-1}]$. Recall that $\mathcal{H}_{t-1} = \{(\pi_1, f_1, c_{1}), (\pi_2, f_2, c_2), \ldots, (\pi_{t-1}, f_{t-1}, c_{t-1})\}$.
\begin{align}
    \E\left[\ind(\eventz_{i,t} ~\cap~ (\ell_t < \ell^{\star})~ \cap~ \zeta_t) ~|~\mathcal{H}_{t-1}\right] & = \E\left[\ind(\eventz_{i,t}) ~|~\mathcal{H}_{t-1} \cap (\ell_t < \ell^{\star}) \right]\P(\ell_t < \ell^{\star}) \notag \\ & = \E\left[\ind(\eventz_{i,t}) ~|~\mathcal{H}_{t-1} \cap (\ell_t < \ell^{\star}) \cap \bar{f}_t(\pi_t) \right]\P{(\bar{f}_t(\pi_t)|\mathcal{H}_{t-1})}\P(\ell_t < \ell^{\star}), \notag 
\end{align} 

where  $\bar{f}_t(\pi_t) = 1 - f_t(\pi_t) = 1$ if the user in round $t$ is real, and the first equality  holds because the event that $\ell_t < \ell^{\star}$
is independent of history set $\mathcal{H}_{t-1}$. The second equality holds because (a) event $f_t(\pi_t)=0$ is independent of the realization of $\ell_t$, and (b) when  we have a fake user in round $t$, i.e., $f_t(\pi_t)=1$, 
$\ind(\eventz_{i,t})$ is zero. Recall that  $\eventz_{i,t}$ is the event that (a) the customer in round $t$ is real, and (b) we receive feedback on product $i$ in round $t$. Next, we characterize  $\E\left[\ind(\eventz_{i,t}) ~|~\mathcal{H}_{t-1} \cap (\ell_t < \ell^{\star}) \cap \bar{f}_t(\pi_t) \right]$, which is the  probability that we receive feedback on product $i$ in round $t$ from a real customer when $\ell_t < \ell^{\star}$:
\[\E\left[\ind(\eventz_{i,t}) ~|~\mathcal{H}_{t-1} \cap (\ell_t < \ell^{\star}) \cap \bar{f}_t(\pi_t) \right]=\E\Big[Q(\pi^{-1}_t(i))\prod_{k=1}^{\pi^{-1}_t(i)-1}(1-\mu_{\pi(k)})\Big]\,.\]
Here, the expectation is due to the randomness in the realization of the level $\ell_t$, which in turn affects the realization of $\pi^{-1}_t(i)$, and  $Q(k) = \prod_{\p{r} = 1}^{k-1}(1-q_{\p{r}})$ is the probability that a real customer does not exit before viewing the products in the first $k$ positions. This leads to \def\s{s}
\begin{align}
    \E\left[\ind(\eventz_{i,t} \cap \ell_t < \ell^{\star} \cap \zeta_t) ~|~\mathcal{H}_{t-1}\right] & = \E\Big[Q(\pi^{-1}_t(i))\prod_{k=1}^{\pi^{-1}_t(i)-1}(1-\mu_{\pi(k)})\Big] \P{(\bar{f}_t(\pi_t)|\mathcal{H}_{t-1})}\P(\ell_t < \ell^{\star})   \notag\\
     & \leq \Big(Q(j)\prod_{\s \in S_{j-1}}(1-\mu_{\s})\Big)  \P{(\bar{f}_t(\pi_t)|\mathcal{H}_{t-1})}\P(\ell_t < \ell^{\star}) \notag \\ 
     & \leq 2F \Big(Q(j)\prod_{\s \in S_{j-1}}(1-\mu_{\s})\Big)  \P{(\bar{f}_t(\pi_t)|\mathcal{H}_{t-1})}\P(\ell_t = \ell^{\star}) \,. \label{eqn_s_k}
\end{align}
  The second inequality holds because 
 for all $\s  \in S_{j-1}$, we have $\pi_t^{-1}(\s ) < \pi_t^{-1}(i)$ since these products are ranked above product $i$.  Note that 
 this inequality  is valid even when graph $G^{(\ell_t)}$ is eliminated before round $t$. This is because the ranking $\pi_t$ in such a round would still depend on a graph $G^{(\ell)}$ for $\ell_t < \ell \leq \ell^{\star}$, as a result of which, product $i$ is ranked below all of the products in $S_{j-1}$. Further, we know that graph $G^{(\ell^{\star})}$ is never eliminated before round $t$, conditional on $\eventforc_{t-1}$. The final inequality comes from the observation that $\P(\ell_t < \ell^{\star}) \leq 2F \P(\ell_t = \ell^{\star})$.

\textbf{Lower Bound on the r.h.s. of \eqref{eq:l_less_l_star}.} Here, we present a lower bound on the r.h.s. of \eqref{eq:l_less_l_star}. 
Proceeding similarly as in the previous case, we get
\begin{align}
    \E\left[\ind(\eventz_{\sigma_t,t} ~\cap~ (\ell_t = \ell^{\star}) ~\cap~ \zeta_t) ~|~\mathcal{H}_{t-1}\right] & = \E\left[\ind(\eventz_{\sigma_t,t}) ~|~\mathcal{H}_{t-1} \cap (\ell_t = \ell^{\star}) \cap \bar{f}_t(\pi_t) \right]\P{(\bar{f}_t(\pi_t)|\mathcal{H}_{t-1} )}\P(\ell_t = \ell^{\star}) \notag \\ & = \E\Big[Q(\pi^{-1}_t(\sigma_t))\prod_{k=1}^{\pi^{-1}_t(\sigma_t)-1}(1-\mu_{\pi(k)})\Big] \P{(\bar{f}_t(\pi_t)|\mathcal{H}_{t-1} )}\P(\ell_t = \ell^{\star})\notag\\
     & \geq  \Big(Q(j)\prod_{\s \in S_{j-1}}(1-\mu_{\s})\Big) \P{(\bar{f}_t(\pi_t)|\mathcal{H}_{t-1} )}\P(\ell_t = \ell^{\star})\,,  \label{eqn_h}
\end{align}
 where the last inequality holds because $\pi^{-1}_t(\sigma_t) \leq j$. This is so because  there are at most $j-1$ products ranked above $\sigma_t$ by definition. Recall that $\sigma_t$ is the highest ranked product outside of $S_{j-1}$ in $\pi_t$ and $Q(j)$ is the probability that the user does not exit before position $j$, which is smaller than $Q(\pi^{-1}_t(\sigma_t))$.

 Comparing \eqref{eqn_s_k} and \eqref{eqn_h}, we conclude that $ \E\left[\ind(\eventz_{\sigma_t,t} \cap \ell_t = \ell^{\star} \cap \zeta_t) ~|~\mathcal{H}_{t-1}\right] \geq  \frac{1}{2F} \E\left[\ind(\eventz_{i,t} \cap \ell_t < \ell^{\star} \cap \zeta_t) ~|~\mathcal{H}_{t-1}\right] $, and this completes the proof. \hfill \Halmos

\subsection{Proof of Lemma \ref{lem:sigma_t}}Assume by contradiction that this is not the case and that $\eta^{(\ell^{\star})}_{\sigma_t,t-1} \geq 4\gamma_{j,i}$.  We first argue that $\sigma_t \notin \G{\Delta_{j,i}} \setminus S_{j-1}$. To see why this is the case, note that as per our definition of $\tdelta$, for all $t \leq \tdelta$, it must be the case that for any $k \in \G{\Delta_{j,i}} \setminus S_{j-1}$, the inequality $\eta^{(\ell^{\star})}_{k,t-1} < 4\gamma_{j,i}$ must hold. Recall that $\G{\Delta_{j,i}}$ is the set of products whose rewards are better than $\mu_{i}$ by an additive factor of $\frac{\Delta_{j,i}}{2}$ or more. Since  this set  $S_{j-1}$  cannot include product $\sigma_t$, our only possibility is that $\sigma_t \notin \G{\Delta_{j,i}}$. In simple terms, this implies that 
$$\Delta_{\sigma_t,i} < \frac{\Delta_{j,i}}{2}.$$

Therefore, $\sigma_t \notin [j]$. Next, in round $t$, we know that product $\sigma_t$ is ranked above the better products in $[j] \setminus S_{j-1}$. Given this inversion in ranking, we can apply Lemma \ref{lem_dagprop_agnostic} with product $\sigma_t$ and some arbitrary product from $[j] \setminus S_{j-1}$ to infer the existence of a product $u$ such that there is no edge between $u$ and $\sigma_t$ in graph $G^{(\ell^{\star})}$ at the beginning of round $t$, and
\begin{align} \label{eq:middle_2}\eta^{(\ell^{\star})}_{\sigma_t,t-1} \leq \eta^{(\ell^{\star})}_{u,t-1}.\end{align}

Moreover, since $4\gamma_{j,i} \leq \eta^{(\ell^{\star})}_{\sigma_t,t-1} \leq \hat{\eta}^{(\ell^{\star})}_{\sigma_t,t-1}$ by our earlier assumption, we also have that: 
\begin{align} \label{eq:middle}4\gamma_{j,i} \leq \eta^{(\ell^{\star})}_{\sigma_t,t-1} \leq \eta^{(\ell^{\star})}_{u,t-1} \leq  \hat{\eta}^{(\ell^{\star})}_{u,t-1}.\end{align}
Since $u$ is at least as good as one of the products in $[j] \setminus S_{j-1}$, we can also glean that $\mu_{u} \geq \mu_{j}$. Leveraging this, we can now quantify the gap between products $u$ and $\sigma_t$ as follows:
$$\Delta_{u,\sigma_t,} \geq \Delta_{j,\sigma_t} = \Delta_{j,i} - \Delta_{\sigma_t,i} \geq  \Delta_{j,i} - \frac{\Delta_{j,i}}{2} =  \frac{\Delta_{j,i}}{2}\,,$$
where the second inequality follows from our earlier argument that $\Delta_{\sigma_t,i} < \frac{\Delta_{j,i}}{2}$. 
The above expression, in turn, implies that $\gamma_{u,\sigma_t} \leq 4\gamma_{j,i} \leq  \hat{\eta}^{(\ell^{\star})}_{\sigma_t,t-1},  \hat{\eta}^{(\ell^{\star})}_{u,t-1}$, where
the last inequality follows from Equation \eqref{eq:middle}. 
Applying Lemma~\ref{lem_pairwise_edgeadd_agnostic}, this would indicate the existence of an edge from $\sigma_t$ to $u$ in graph $G^{(\ell^{\star})}$, which is a contradiction. Therefore, tying this back to our original proposition, we conclude that $\eta^{(\ell^{\star})}_{\sigma_t,t-1} < 4\gamma_{j,i}$. \hfill \Halmos

\end{APPENDICES}
\end{document}